\documentclass{new_tlp} 
\usepackage{times}
\usepackage{amsmath}
\usepackage{graphicx}
\usepackage{listings}
\usepackage{url}

\newcommand\bcmdtab{\noindent\bgroup\tabcolsep=0pt%
  \begin{tabular}{@{}p{10pc}@{}p{20pc}@{}}}
\newcommand\ecmdtab{\end{tabular}\egroup}

\jdate{March 2017}
\pubyear{2017}
\pagerange{\pageref{firstpage}--\pageref{lastpage}}
\doi{S1471068401001193}

\newcommand\tagthis{\addtocounter{equation}{1}\tag{\theequation}}
\def\taglabel#1{\tagthis\label{#1}}

\long\def\BOC#1\EOC{\message{(Commented text )}}
\long\def\BOCC#1\EOCC{\message{(Commented text )}}
\long\def\BOCCC#1\EOCCC{\message{(Commented text )}}
\long\def\optional#1{\empty}
\long\def\NB#1{}

\def\ar{\leftarrow}
\def\beq{\begin{equation}}
\def\eeq#1{\label{#1}\end{equation}}
\def\ba{\begin{array}}
\def\ea{\end{array}}
\def\bi{\begin{itemize}}
\def\ei{\end{itemize}}
\def\j#1{\hbox{\it #1\/}}
\def\mi#1{\mathit{#1}}	
\def\mu#1{\mathit{\underline{#1}}}

\def\rar{\rightarrow}
\def\lrar{\leftrightarrow}
\def\proof{\noindent{\bf Proof}.\hspace{3mm}}

\def\false{\hbox{\sc false}}
\def\true{\hbox{\sc true}}
\def\mvis{\!=\!}

\def\sm{\hbox{\rm SM}}
\def\comp{\hbox{\rm Comp}}

\def\bC{{\bf{c}}}

\newtheorem{thm}{Theorem}

\newtheorem{lemma}{Lemma}
\newtheorem{example}{Example}

\def\caused{\hbox{\bf caused}}
\def\iif{\hbox{\bf if}}
\def\iis{\hbox{\bf is}}
\def\derivative{\hbox{\bf derivative}}
\def\of{\hbox{\bf of}}
\def\after{\hbox{\bf after}}

\def\causes{\hbox{\bf causes}}
\def\inertial{\hbox{\bf inertial}}
\def\default{\hbox{\bf default}}
\def\constraint{\hbox{\bf constraint}}

\def\alwayst{\hbox{\bf always\_t}}

\def\nonexecutable{\hbox{\bf nonexecutable}}
\def\exogenous{\hbox{\bf exogenous}}

\begin{document}

\title{Representing Hybrid Automata by Action Language Modulo Theories} 

\author[J. Lee, N. Loney and Y. Meng]{
Joohyung Lee, Nikhil Loney\\ 
School of Computing, Informatics and Decision Systems Engineering \\
Arizona State University, Tempe, AZ, USA \\
\email{\{joolee, nloney\}@asu.edu}
\and
Yunsong Meng\\
Houzz, Inc. \\
Palo Alto, CA, USA \\
\email{Yunsong.Meng@asu.edu}
}

\pagerange{\pageref{firstpage}--\pageref{lastpage}}

%\submitted {} 
%\revised {} 
%\accepted{}

\label{firstpage}

\maketitle

\begin{abstract}
Both hybrid automata and action languages are formalisms for describing the evolution of dynamic systems. This paper establishes a formal relationship between them. We show how to succinctly represent hybrid automata in an action language which in turn is defined as a high-level notation for answer set programming modulo theories (ASPMT) --- an extension of answer set programs to the first-order level similar to the way satisfiability modulo theories (SMT) extends propositional satisfiability (SAT).  We first show how to represent linear hybrid automata with convex invariants by an action language modulo theories. A further translation into SMT allows for computing them using SMT solvers that support arithmetic over reals.  Next, we extend the representation to the general class of non-linear hybrid automata allowing even non-convex invariants. We represent them by an action language modulo ODE (Ordinary Differential Equations), which can be compiled into satisfiability modulo ODE. We present a prototype system {\sc cplus2aspmt} based on these translations, which allows for a succinct representation of hybrid transition systems that can be computed effectively by the state-of-the-art SMT solver ${\tt dReal}$.
\end{abstract}

\begin{keywords}
Answer Set Programming, Action Languages, Hybrid Automata
\end{keywords}

%------------------------------------------------------------------------------
\section{Introduction} \label{sec:intro}
%------------------------------------------------------------------------------

Both hybrid automata \cite{henzinger96thetheory} and action languages \cite{gel98} are formal models for describing the evolution of dynamic systems. The focus of hybrid automata is to model continuous transitions as well as discrete changes, but, unlike action languages,  their discrete components are too simple to represent complex relations among fluents and various properties of actions. On the other hand, transitions described by most action languages are limited to discrete changes only, which hinders action languages from modeling real-time physical systems. One of the exceptions is an enhancement of action language ${\cal C}$+ \cite{lee13answer}, which extends the original, propositional language in the paper by \citeN{giu04} to the first-order level. The main idea there is to extend the propositional ${\cal C}$+ to the first-order level by defining it in terms of Answer Set Programming Modulo Theories (ASPMT) --- a tight integration of answer set programs and satisfiability modulo theories (SMT) to allow SMT-like effective first-order reasoning in ASP. 

This paper establishes a formal relationship between hybrid automata and action language~${\cal C}$+. We first show how to represent {\em linear hybrid automata} with {\em convex invariants} by the first-order ${\cal C}$+.  A further translation into SMT allows for computing them using state-of-the-art SMT solvers that support arithmetic over reals.  
However, many practical domains of hybrid systems involve non-linear polynomials, trigonometric functions, and differential equations that cannot be represented by linear hybrid automata. Although solving the formulas with these functions is undecidable in general, \citeN{gao13satisfiability} presented a novel approach called a ``$\delta$-complete decision procedure'' for computing such SMT formulas, which led to the concept of ``satisfiability modulo ODE.'' \footnote{%
A $\delta$-complete decision procedure for an SMT formula $F$ returns false if $F$ is unsatisfiable, and returns true if its syntactic ``numerical perturbation'' of $F$ by bound $\delta$ is satisfiable, where $\delta>0$ is number provided by the user to bound on numerical errors. The method is practically useful since it is not possible to sample exact values of physical parameters in reality.}
The procedure is implemented in the SMT solver ${\tt dReal}$ \cite{gao13dreal}, which is shown to be useful for formalizing the general class of hybrid automata. 
We embrace the concept into action language ${\cal C}$+ by introducing two new abbreviations of causal laws, one for representing the evolution of continuous variables as specified by ODEs and another for describing invariants that the continuous variables must satisfy when they progress. The extension is rather straightforward thanks to the close relationship between ASPMT and SMT: ASPMT allows for quantified formulas as in SMT, which is essential for expressing non-convex invariants; algorithmic improvements in SMT can be carried over to the ASPMT setting.
We show that the general class of hybrid automata containing non-convex invariants can be expressed in the extended ${\cal C}$+ modulo ODEs.

The extended ${\cal C}$+ allows us to achieve the advantages of both hybrid automata and action languages, where the former provides an effective way to represent continuous changes, and the latter provides an elaboration tolerant way to represent (discrete) transition systems. In other words, the formalism gives us an elaboration tolerant way to represent hybrid transition systems. Unlike hybrid automata, the structured representation of states allows for expressing complex relations between fluents, such as recursive definitions of fluents and indirect effects of actions, and unlike propositional ${\cal C}$+, the transitions described by the extended ${\cal C}$+ are no longer limited to discrete ones only; the advanced modeling capacity of action languages, such as additive fluents, statically defined fluents, and action attributes, can be achieved in the context of hybrid reasoning.

We implemented a prototype system {\sc cplus2aspmt} based on these translations, which allows for a succinct representation of hybrid transition systems in language ${\cal C}$+ that can be compiled into the input language of ${\tt dReal}$.  We show that the system can be used for reasoning about hybrid transition systems, whereas other action language implementations, such as the Causal Calculator \cite{giu04}, {\sc cplus2asp}~\cite{babb13cplus2asp}, and {\sc coala}~\cite{gebser10coala} cannot. 

The paper is organized as follows. In Section~\ref{sec:prelim}, we give a review of hybrid automata to set up the terminologies used for the translations. Section~\ref{sec:linear} presents how to represent the special class of linear hybrid automata with convex invariants by ${\cal C}$+ modulo theory of reals. Section~\ref{sec:new} introduces two new abbreviations of causal laws that can be used for modeling invariant and flow conditions. Section~\ref{sec:non-linear} uses these new constructs to represent the general class of non-linear hybrid automata and shows how to reduce them to the input language of ${\tt dReal}$ leading to the implementation of system {\sc cplus2aspmt}, a variant of the system {\sc cplus2asp}.
%, that uses ${\tt dReal}$ for search. 

The proofs of the theorems and the examples of hybrid automata in the input language of {\sc cplus2aspmt} can be found in the online appendix accompanying the paper at the TPLP archive \cite{lee17representing-online}.

%------------------------------------------------------------------------------
\section{Preliminaries} \label{sec:prelim}
%------------------------------------------------------------------------------

%---------------------------------------------------------------------
\subsection{Review: Hybrid Automata}
%---------------------------------------------------------------------

We review the definition of Hybrid Automata~\cite{henzinger96thetheory,alur00discrete}, formulated in terms of a logical language by representing arithmetic expressions by many-sorted first-order formulas under background theories, such as QF\_NRA (Quantifier-Free Non-linear Real Arithmetic) and QF\_NRA\_ODE (Quantifier-Free Non-linear Real Arithmetic with Ordinary Differential Equations).
By $\mathcal{R}$ we denote the set of all real numbers and by $\mathcal{R}_{\ge 0}$ the set of all non-negative real numbers.
Let $X$ be a set of real variables. An arithmetic expression over $X$ is an atomic formula constructed using functions and predicates from the signature of the background theory and elements from $\mathcal{R}\cup X$. Let $A(X)$ be an arithmetic expression over $X$ and let ${x}$ be a tuple of real numbers whose length is the same as the length of $X$. By $A({x})$, we mean the expression obtained from $A$ by replacing variables in $X$ with the corresponding values in ${x}$. For an arithmetic expression with no variables, we say that {\em $A$ is true} if the expression is evaluated to true in the background theory. 

A {\em Hybrid Automaton} ${\cal H}$ consists of the following components:

\begin{itemize}
\item {\bf Variables:} A finite list of real-valued variables $X = (X_1, \dots , X_n)$. 
The number $n$ is called the {\em dimension} of ${\cal H}$. 
We write $\dot{X}$ for the list $(\dot{X}_1,\dots,\dot{X}_n)$ of dotted variables, representing first derivatives during a continuous change, and 
$X'$ for the set $(X_1', \dots , X_n')$ of primed variables, representing the values at the conclusion of the discrete change.  $X_0\subseteq X$ is the set of initial states.
We use lower case letters to denote the values of these variables.

\item {\bf Control Graph:}  A finite directed graph $\langle V, E\rangle$. The vertices are called {\em control modes}, and the edges are called {\em control switches}.

\item {\bf Initial, Invariant, and Flow Conditions:} Three vertex labeling functions, ${\sf Init}$, ${\sf Inv}$, and ${\sf Flow}$, that assign to each control mode $v\in V$ three first-order formulas:

\begin{itemize}

\item ${\sf Init}_v(X)$ is a first-order formula whose free variables are from $X$. The formula constrains the initial condition.

\item ${\sf Inv}_v(X)$ is a first-order formula whose free variables are from $X$. The formula constrains the value of the continuous part of the state while the mode is $v$.

\item ${\sf Flow}_v(X, \dot{X})$ is a set of first-order formulas whose free variables are from $X\cup \dot{X}$.
The formula constrains the continuous variables and their first derivatives.
\end{itemize}

\item {\bf Events:} A finite set $\Sigma$ of symbols called {\em h-event}s and a function, 
${\sf hevent} : E \rar \Sigma$, that assigns to each edge a unique h-event.

\item {\bf Guard:} For each control switch $e\in E$, ${\sf Guard}_e(X)$ is a first-order formula whose free variables are from $X$.

\item {\bf Reset:} For each control switch $e\in E$, ${\sf Reset}_e(X,X')$ is a first-order formula whose free variables are from $X\cup X'$. 

\NB{May be better to use guard and reset instead of jump} 

\end{itemize}

\newpage
\begin{example}\label{ex:water-tank}
\begin{center}
\includegraphics[scale=0.43]{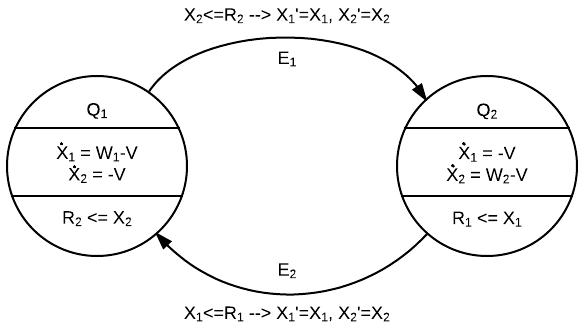}
\end{center}

\smallskip
The figure shows a hybrid automaton for the Water Tank Example 
from the lecture note by~\citeN{lygeros04lecture}, which consists of two variables $X = (X_1 , X_2)$, two h-events $E_1$ and $E_2$, and two control modes $V = \{Q_1 , Q_2 \}$.  For example,  
\begin{itemize}
\item  ${\sf Flow}_{Q_1}(\dot{X_1},\dot{X_2})$ is $\dot{X_1}\mvis {\rm W}\!-\!{\rm V_1}\land
\dot{X_2}\mvis \!-\!{\rm V_2}$. \ \ \ \ 
\item ${\sf Inv}_{Q_1}(X_1,X_2)$ is $X_2\!\ge\! {\rm R_2}$. 
%\item ${\sf Init}_{Q_1}(X_1,X_2)$ is $X_1\!\ge\! R_1\land X_2\!\ge\! {\rm R_2}$. 
\item  ${\sf Guard}_{(Q_1,Q_2)}(X_1,X_2)$ is $X_2\le {\rm R_2}$.\ \ 
\item ${\sf Reset}_{(Q_1,Q_2)}(X_1,X_2,{X_1}',{X_2}')$ is $X_1'\!=\!X_1\land X_2'\!=\!X_2.$
\end{itemize}
\end{example}

\medskip
A {\em labeled transition system} consists of the following components:
\bi
\item {\bf State Space:} A set $Q$ of states and a subset $Q_0\subseteq Q$ of initial states. 

\item {\bf Transition Relations:} A set $A$ of labels. For each label $a\in A$, a binary relation $\rar^a$ on the state space $Q$. Each triple $q\rar^a q'$ is called a {\em transition}.
\ei

The {\em Hybrid Transition System} $T_H$ of a Hybrid Automaton $H$ is the labeled transition system obtained from $H$ as follows. 

\begin{itemize}
\item  The set $Q$ of {\em states} is the set of all $(v,r)$ such that $v\in V$, $r\in\mathcal{R}^n$, and ${\sf Inv}_v(r)$ is true. 

\item $(v, r) \in Q_0$ iff both ${\sf Init}_v(r)$ and ${\sf Inv}_v(r)$ are true.

\item The transitions are labeled by members from $A=\Sigma\cup \mathcal{R}_{\ge 0}$.

\item  $(v, r)\rar^\sigma (v', r')$, where $(v,r), (v',r')\in Q$ and $\sigma$ is an h-event in $\Sigma$, is a {\em transition} if there is an edge $e = (v, v') \in E$ such that:
(1) ${\sf hevent}(e) = \sigma$,
(2) the sentence ${\sf Guard}_e(r)$ is true,  and
(3) the sentence ${\sf Reset}_e(r,r')$ is true.

\item   $(v, r) \rar^{\delta} (v, r')$, where $(v,r), (v,r')\in Q$ and $\delta$ is a nonnegative real, is a {\em transition} if there is a differentiable function 
$f : [0, \delta] \rar \mathcal{R}^n$, 
with the first derivative $\dot{f} : [0, \delta] \rar \mathcal{R}^n$ such that: 
\begin{enumerate}

\item[(1)] $f(0) = r$ and $f(\delta) = r'$,

\item[(2)] for all real numbers $\epsilon\in [0, \delta]$,  ${\sf Inv}_v(f(\epsilon))$ is true and, for all real numbers $\epsilon\in (0, \delta)$, ${\sf Flow}_v(f(\epsilon), {\dot{f}}(\epsilon))$ is true. The function $f$ is called the {\em witness} function for the transition $(v, r) \rar^{\delta} (v, r')$. 
\end{enumerate}

\end{itemize}

%---------------------------------------------------------------------
\subsection{Review: ASPMT and ${\cal C}$+}\label{ssec:review-aspmt-cplus}
%---------------------------------------------------------------------

ASPMT~\cite{bartholomew13functional} is a special case of many-sorted first-order (functional) stable model semantics from the papers by \citeN{ferraris11stable} and by \citeN{bartholomew13functional} by restricting the background signature to be interpreted in the standard way, in the same way SMT restricts first-order logic. 

The syntax of ASPMT is the same as that of SMT. Let $\sigma^{bg}$ be
the (many-sorted) signature of the background theory~$bg$. An
interpretation of $\sigma^{bg}$ is called a {\em background
  interpretation} if it satisfies the background theory. For instance,
in the theory of reals, we assume that $\sigma^{bg}$ contains the set
$\mathcal{R}$ of symbols for all real numbers, the set of arithmetic
functions over real numbers, and the set $\{<, >, \le, \ge\}$ of
binary predicates over real numbers. Background interpretations
interpret these symbols in the standard way.

Let $\sigma$ be a signature that is disjoint from $\sigma^{bg}$.
We say that an interpretation $I$ of $\sigma$ satisfies a sentence $F$
w.r.t. the background theory $bg$, denoted by $I\models_{bg} F$,
if there is a background interpretation $J$ of $\sigma^{bg}$ that has
the same universe as $I$, and $I\cup J$ satisfies $F$.
Interpretation $I$ is a {\em stable model} of $F$ relative to a set of function and predicate constants ${\bf c}$ (w.r.t. the background theory $\sigma^{bg}$) if $I\models_{bg} \sm[F; \bC]$ (we refer the reader to the paper by \citeN{bartholomew13functional} for the definition of the $\sm$ operator).

In the paper by~\citeN{lee13answer}, action language ${\cal C}$+ was reformulated in terms of ASPMT and was shown to be useful for reasoning about hybrid transition systems. 
Appendix~A \cite{lee17representing-online}
reviews this version of ${\cal C}$+. 

%---------------------------------------------------------------------
\section{Representing Linear Hybrid Automata with Convex Invariants by ${\cal C}$+ Modulo Theories} \label{sec:linear}
%---------------------------------------------------------------------

\subsection{Representation}

{\em Linear} hybrid automata \cite{henzinger96thetheory}  are a special case  of hybrid automata where (i) the initial, invariant, flow, guard, and reset conditions are Boolean combinations of linear inequalities, and (ii) the free variables of flow conditions are from $\dot{X}$ only. In this section, we assume that for each ${\sf Inv}_v(X)$ from each control mode $v$, the set of values of $X$ that makes ${\sf Inv}_v(X)$ true forms a convex region. \footnote{A set $X$ is {\em convex} if for any $x_1,x_2\in X$ and any $\theta$ with $0\le \theta \le 1$, we have $\theta x_1 + (1-\theta)x_2\in X$.} For instance, this is the case when ${\sf Inv}_v(X)$ is a {\em conjunction} of linear inequalities. 

We show how a linear hybrid automata $H$ can be turned into an action description $D_H$ in ${\cal C}+$, and extend this representation to non-linear hybrid automata in the next section. We first define the signature of the action description $D_H$ as follows. 

\begin{itemize}
\item  For each real-valued variable $X_i$ in $H$, a simple fluent constant $X_i$ of sort $\mathcal{R}$.
\item  For each control switch $e\in E$ and the corresponding ${\sf hevent}(e)\in \Sigma$, a Boolean-valued action constant ${\sf hevent}(e)$.
\item  An action constant $\j{Dur}$ of sort nonnegative reals. 
\item  A Boolean action constant $\j{Wait}$.
\item  A fluent constant $\j{Mode}$ of sort $V$ (control mode).
\end{itemize}

The ${\cal C}$+ action description $D_H$ consists of the following causal laws. 
We use lower case letter $x_i$ for denoting a real-valued variable. 
Let $X = (X_1,\dots, X_n)$ and $x=(x_1,\dots, x_n)$.
By $X=x$, we denote the conjunction $(X_1=x_1)\land \dots\land (X_n=x_n)$.

\begin{itemize}

\item {\bf Exogenous constants: }
\[
\ba c
  {\exogenous}\ X_i  \ \ \ \ (X_i\in X)\\
  {\exogenous}\ {\sf hevent}(e) \\
  {\exogenous}\ \j{Dur}.
\ea
 \]
Intuitively, these causal laws assert that the values of the fluents can be arbitrary.  The action constant $\j{Dur}$ is to record the duration that each transition takes (discrete transitions are assumed to have duration $0$).

\item {\bf Discrete transitions: }
For each control switch $e=(v_1,v_2) \in E$:
\begin{itemize}

\item  {\bf Guard}: 
\[
    \nonexecutable\ {\sf hevent}(e)\ \iif\ \neg {\rm Guard}_e(X) .
\]
The causal law asserts that an h-event cannot be executed if its guard condition is not satisfied.

\item {\bf Reset}: 
\[
\ba c
   \constraint\ {\sf Reset}_e(x,X)\ \after\ X=x\land {\sf hevent}(e)\mvis\true. 
\ea
\]
The causal law asserts that if an h-event is executed, the discrete transition sets the new value of fluent $X$ as specified by the reset condition. 

\item  {\bf Mode and Duration:} 
\[
\ba l
    \inertial\ \j{Mode}=v                  \hspace{2cm} (v\in V) \\
    \nonexecutable\ {\sf hevent}(e)\ \iif\ \j{Mode}\ne v_1 \\
    {\sf hevent}(e)\ \causes\ \j{Mode}=v_2\\
    {\sf hevent}(e)\ \causes\ \j{Dur}\mvis 0.
\ea
\]
The first causal law asserts the commonsense law of inertia on the control mode: the mode does not change when no action affects it.
The second causal law asserts an additional constraint for an h-event to be executable (when the state is in the corresponding mode). The third and fourth causal laws set the new control mode and the duration when the h-event occurs.  
\end{itemize}

\item {\bf Continuous Transitions:} 

\begin{itemize}

\item  {\bf Wait}: 
\[
\ba l
     \default\ \j{Wait}\mvis\true \\
     {\sf hevent}(e)\ \causes\ \j{Wait}\mvis\false  .
\ea
\]
$\j{Wait}$ is an auxiliary action constant that is true when no h-event is executed, in which case a continuous transition should occur. 

\item 
{\bf Flow}: For each control mode $v\in V$ and for each $X_i\in X$,
\beq 
\ba l
 \constraint\ {\sf Flow}_v( (X-x) / \delta)\\
 \hspace{2.5cm} \after\  X=x \land \j{Mode}=v \land\j{Dur}=\delta\land \j{Wait}\mvis\true  \ \ \ (\delta >0) \\
\constraint\ X = x\ 
  \after\  X=x \land \j{Mode}=v \land \j{Dur}=0 \land \j{Wait}\mvis\true .
\ea
\eeq{r2} 
These causal laws assert that when no h-event is executed (i.e., $\j{Wait}$ is true), the next values of the continuous variables are determined by the flow condition.  

\item  {\bf Invariant}: For each control mode $v\in V$,
\beq
     \constraint\ \j{Mode}\mvis v \rar {\sf Inv}_v(X).
\eeq{inv-linear}
The causal law asserts that in each state, the invariant condition for the control mode should be true.

\end{itemize}

\end{itemize}

It is easy to see from the assumption on the flow condition of linear hybrid automata that the witness function exists and is unique ($f(\epsilon) = x+\frac{x'-x}{\delta}\epsilon$); obviously it is linear. 

Note that \eqref{inv-linear} checks the invariant condition in each state only, not during the transition between the states. This does not affect the correctness because of the assumption that the invariant condition is convex and the flow condition is linear, from which it follows that 
\beq
   \forall \epsilon\in[0,\delta]({\sf Inv}_v(f(0))\ \wedge\ {\sf Inv}_v(f(\delta))\to {\sf Inv}_v(f(\epsilon)))
\eeq{invariant-state}
is true, where $f$ is the witness function. 

%

%\begin{example}
Figure~\ref{fig:water-cplus} shows the translation of the Hybrid Automaton in Example~\ref{ex:water-tank} into ${\cal C}$+.

%%%%%%%%%%%%%%%%%%%%%%%%%%%%%%%%%%%%%%%%%%%
%
%      mark
    \begin{figure}[b]
    {%\footnotesize
    \hrule
    \begin{tabbing}
    $q\in\{Q_1, Q_2\}$;\ \  $t$, $x_1$, $x_2$ are variables of sort $\mathcal{R}_{\ge 0}$.  ${\rm W_1}, {\rm W_2}, {\rm V}$ are fixed real numbers\\ \\
    
    Simple fluent constants:   \hskip 3.5cm \= Sort: \\
    ~~~~$X_1$, $X_2$                \> ~~~~$\mathcal{R}_{\ge 0}$ \\
    ~~~~$\j{Mode}$                          \> ~~~~$\{Q_1,Q_2\}$ \\
    Action constants:                     \> Sort:\\
    ~~~~$E_1$, $E_2$, \j{Wait}                 \> ~~~~Boolean\\
    ~~~~$\j{Dur}$                           \> ~~~~$\mathcal{R}_{\ge 0}$ \\ \\
    
    \% Exogenous constants: \\
    $\exogenous\ X_1, X_2, E_1, E_2, \j{Dur}$ \\  \\
    
    \% Guard:  \\
    $\nonexecutable\ E_1\ \iif\ \neg (X_2\le R_2)$  \> 
    $\nonexecutable\ E_2\ \iif\ \neg (X_1\le R_1)$ \\  \\
    
    \% Reset: \\
    $\constraint\ (X_1,X_2)=(x_1,x_2)\ \after\ (X_1,X_2)=(x_1,x_2)\land E_1\mvis\true$ \\
    $\constraint\ (X_1,X_2)=(x_1,x_2)\ \after\ (X_1,X_2)=(x_1,x_2)\land E_2\mvis\true$ \\ \\
    
    \% Mode: \\
    $\nonexecutable\ E_1\ \iif\ \neg(\j{Mode}=Q_1)$ \> 
    $\nonexecutable\ E_2\ \iif\ \neg(\j{Mode}=Q_2)$ \\ 
    
    $E_1\ \causes\ \j{Mode}=Q_2$\  \hspace{3.5cm}
    $E_2\ \causes\ \j{Mode}=Q_1$\  \\ 
    
    $\inertial\ \j{Mode}=q    \ \ \ \ (q\in \{Q_1,Q_2\}$ \\ \\
    
    \% Duration: \\
    $E_1\ \causes\ \j{Dur}\mvis 0$  \> 
    $E_2\ \causes\ \j{Dur}\mvis 0$  \\ \\
    
    \% Wait: \\
    $\default\ \j{Wait}=\true$ \\
    $E_1\ \causes\ \j{Wait}=\false$ \> 
    $E_1\ \causes\ \j{Wait}=\false$ \\  \\
    
    \% Flow: \\ 
    $\constraint\ ((X_1\!\!-\!\!x_1)/t, (X_2\!\!-\!\!x_2)/t) = ({\rm W_1}\!-\!{\rm V}, -{\rm V})$  \\
    $ \hspace{2cm}  \after\ (X_1,X_2)=(x_1,x_2)\land \j{Mode}=Q_1\land\j{Dur}=t\land t>0\land \j{Wait}=\true $\\ 
    $\constraint\ ((X_1\!\!-\!\!x_1)/t, (X_2\!\!-\!\!x_2)/t) = (-{\rm V}, {\rm W_2}-V)$ \\
    $ \hspace{2cm}  \after\ (X_1,X_2)=(x_1,x_2)\land \j{Mode}=Q_2\land\j{Dur}=t\land t>0\land \j{Wait}=\true $\\ 
    
    $\constraint\ (X_1, X_2) = (x_1,x_2)\ \after\ (X_1,X_2)=(x_1, x_2)\land \j{Mode}=q\land\j{Dur}=0\land \j{Wait}=\true$  \ \ $(q\in \{Q_1,Q_2\})$ \\  \\
    
    \% Invariant \\
    $\constraint\ \j{Mode}\mvis Q_1 \rar X_2\ge {\rm R_2}$ \\
    $\constraint\ \j{Mode}\mvis Q_2 \rar X_1\ge {\rm R_1}$ 
    \end{tabbing}
    \hrule
    }
    \caption{${\cal C}$+ Representation of Hybrid Automaton of Water Tank}
    %\vspace{-5mm}
    \label{fig:water-cplus}
    \end{figure}

The following theorem asserts the correctness of the translation.  By a path we mean a sequence of transitions. \footnote{For simplicity of the comparison, as with action descriptions, the theorem does not require that  the initial state of a path in the labeled transition system satisfy the initial condition. The condition can be easily added.}

\begin{thm}\label{thm:ha2cplus}
There is a 1:1 correspondence between the paths of the transition system of a hybrid automaton~$H$ and the paths of the transition system of the action description $D_H$.
\end{thm}

The proof is immediate from the following two lemmas. 
First, we state that every path in the labeled transition system of $T_H$ is a path in the transition system described by $D_H$.

\begin{lemma}\label{lem:ha2cplus}\optional{lem:ha2cplus-linear}
For any path 
\[
  p=(v_{0}, {r}_{0}) \xrightarrow{\sigma_0}
    (v_{1}, {r}_{1}) \xrightarrow{\sigma_1}\dots
    \xrightarrow{\sigma_{m-1}}
    (v_{m}, {r}_{m})
\]
in the labeled transition system of $H$,   let 
\[
  p' = \langle s_0,a_0,s_1,a_1,\dots, a_{m-1}, s_m\rangle, 
\]
where each $s_i$ is an interpretation of fluent constants and each $a_i$ is an interpretation of action constants such that, for $i=0,\dots m\!-\!1$,
\begin{itemize}
\item $s_0 \models_{bg} (\j{Mode}, {X})
           =(v_{0}, {r}_{0})$;
           
\item $s_{i+1}\models_{bg} (\j{Mode}, {X})=  (v_{i+1}, {r}_{i+1})$; 
     \item  if $\sigma_i={\sf hevent}( v_{i}, v_{i+1} )$, 
       then $(\j{Dur})^{a_i}\mvis 0$, $(\j{Wait})^{a_i}=\false$, and, for all $e\in E$,        $({\sf hevent}(e))^{a_i}=\true$ iff $e=(v_{i}, v_{i+1})$;

     \item   if $\sigma_i\in\mathcal{R}_{\ge 0}$, then
        $(\j{Dur})^{a_i}=\sigma_i$, $(\j{Wait})^{a_i}=\true$, and, for all $e\in E$, we have $({\sf hevent}(e))^{a_i}=\false$.
\end{itemize}

Then, $p'$ is a path in the transition system $D_H$.
\end{lemma}

Next, we show that every path in the transition system of $D_H$ is a path in the labeled transition system of $H$.

\begin{lemma}\label{lem:cplus2ha}\optional{lem:cplus2ha-linear}
For any path 
 \[
  q = \langle s_0,a_0,s_1,a_1,\dots, a_{m-1}, s_m\rangle
\]
in the transition system of $D_H$, 
let 
\[
  q'=(v_{0}, {r}_{0}) \xrightarrow{\sigma_0}
    (v_{1}, {r}_{1}) \xrightarrow{\sigma_1}\dots
    \xrightarrow{\sigma_{m-1}}
    (v_{m}, {r}_{m}), 
\]
where 
\begin{itemize}
\item  
$v_{i}\in V$ and ${r}_{i}\in \mathcal{R}^n$ ($i=0,\dots,m$) are such that $s_i \models_{bg} (\j{Mode}, {X})
           =(v_{i}, r_{i})$;

\item  $\sigma_i$ ($i=0,\dots, m\!-\!1$) is 
\begin{itemize}
\item  ${\sf hevent}(v_{i}, v_{i+1})$ if  
    $({\sf hevent}(v_{i}, v_{i+1}))^{a_i}=\true$;
\item  $(Dur)^{a_i}$ otherwise.
\end{itemize}
\end{itemize}

Then, $q'$ is a path in the transition system of $T_H$.
\end{lemma}

%-----------------------------------------------------------------------
\subsection{Representing Non-Linear Hybrid Automata using Witness Function}\label{ssec:non-linear-witness}
%---------------------------------------------------------------------

Note that formula \eqref{invariant-state} is not necessarily true in general even when ${\sf Inv}_v(X)$ is a Boolean combination of linear (in)equalities (e.g., a disjunction over them may yield a non-convex invariant). 

Let us assume ${\sf Flow}_v(X,\dot{X})$ is the conjunction of formulas of the form $\dot{X}_i = g_i(X)$ for each $X_i$, where $g_i(X)$ is a Lipschitz continuous function whose variables are from $X$ only.\footnote{%
A function $f:\mathcal{R}^n\rar \mathcal{R}^n$ is called {\em Lipschitz continuous} if there exists $\lambda>0$ such that for all ${x,x'}\in \mathcal{R}^n$,
\[
   |f(x)-f(x')|\le \lambda|{x}-x'|.
\] 
}
In this case, it is known that the witness function $f$ exists and is unique.
This is a common assumption imposed on hybrid automata,

Even when the flow condition is non-linear, as long as we already know the unique witness function satisfies \eqref{invariant-state}, the invariant checking can still be done at each state only.
In this case, the representation in the previous section works with a minor modification. 
We modify the {\bf Flow} representation as 
%\begin{itemize}
%\item {\bf Flow}: For each $v\in V$ and $X_i\in X$,
%\begin{align*}
%      \constraint\ X_i = x_i + \int_0^\delta \dot{X}_i\ dt\ \after\ X=x \land \j{Mode}=v\land \j{Dur}=\delta \land \j{Wait}=\true.
%\end{align*} 
%\end{itemize}
%
\begin{itemize}
\item {\bf Flow}: For each $v\in V$ and $X_i\in X$,
\[
\ba l
      \constraint\ X_i = f_i(\delta)\ \after\ X=x \land \j{Mode}=v\land \j{Dur}=\delta \land \j{Wait}=\true\\
\ea
\] 
where $f_i : [0, \delta] \rar \mathcal{R}^n$ is the witness function for $X_i$ such that
(i) $f_i(0) = x_i$ and 
(ii) for all reals $\epsilon \in [0, \delta]$,  ${\sf Flow}_v(f(\epsilon),\dot{f(\epsilon)})$ is true, where $f = (f_1,\dots,f_n)$.
\end{itemize}

\begin{example}
Consider a hybrid automaton for the two bouncing balls with different elasticity. 
\begin{center}
  \includegraphics[height=3.4cm]{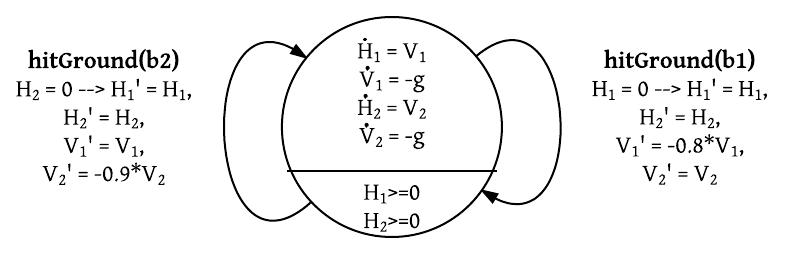}
\end{center}

The {\bf Flow} condition for Ball $b1$ is represented as 
\[
\ba l
   \constraint\ V_1 = v+(-g)\cdot\delta\ \after\ V_1=v\land  \j{Dur}=\delta\land \j{Wait}=\true \\ 
   \constraint\ H_1= h+v\cdot\delta - (0.5)\cdot g\cdot \delta\cdot \delta\ \after\ H_1=h\land\j{Dur}=\delta \land\j{Wait}=\true.
\ea
\]
The invariant ($H_1\ge 0, H_2\ge 0$) is trivial and satisfies equation~\eqref{invariant-state}. So, it is sufficient to check the invariant using \eqref{inv-linear} at each state only.
\end{example}

However, this method does not ensure that a (non-convex) invariant holds during continuous transitions. For example, consider the problem of a car navigating through the pillars as in Figure~\ref{fig:robot}, where the circles represent pillars that the car has to avoid collision with.
Checking the invariants at each discrete time point is not sufficient; it could generate an infeasible plan, such as (b), where the initial position $(0,0)$ and the next position $(13,0)$ satisfy the invariant $(x-9)^2 + y^2 > 9$, but some positions between them, such as $(8, 0)$, do not.  This is related to the challenge in integrating high-level task planning and low-level motion planning, where plans generated by task planners may often fail in motion planners.

The next section introduces new constructs in ${\cal C}$+ to address this issue. 

%\begin{multicols}{2}
\begin{figure}[h!]  
  \includegraphics[height=4.2cm]{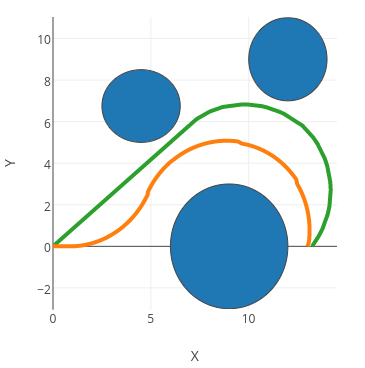} 
%  \end{figure}
 % \begin{figure}[h!]
  \includegraphics[height=4.2cm]{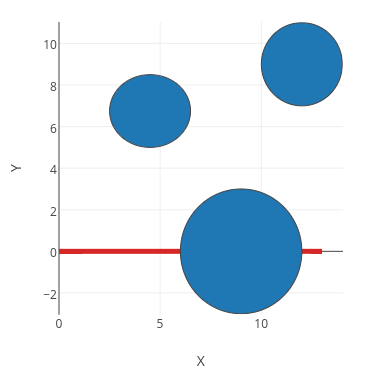}
  \caption{\hspace{0.5cm}(a) feasible plan   \hspace{1.7cm} (b) infeasible plan}
  \label{fig:robot}
  \end{figure}
%\end{multicols}
 
%------------------------------------------------------------------------------
\section{New Abbreviations of Causal Laws for Expressing Continuous Evolutions via ODEs} \label{sec:new}
%------------------------------------------------------------------------------

In this section we introduce two new abbreviations of causal laws to express the continuous evolutions governed by ODEs. 

We assume the set $\sigma^\mi{fl}$ of fluent constants contains a set $\sigma^\mi{diff}$ of real valued fluent constants $X=(X_1,\dots, X_n)$ called {\sl differentiable} fluent constants, and an inertial fluent constant $\j{Mode}$, which ranges over a finite set of control modes. Intuitively, the values of differentiable fluent constants are governed by some ODEs controlled by each value of $\j{Mode}$. We also assume that $\j{Dur}$ is an exogenous action constant of sort ${\cal R}_{\ge 0}$.

Below are the two new abbreviations related to ODEs. 
First, a {\em rate declaration} is an expression of the form:
\beq
   \derivative\ \of\ X_i\ \iis\ F_i(X)\ {\bf if}\ \j{Mode}=v
\eeq{rate}
where $X_i$ is a differentiable fluent constant, $v$ is a control mode, 
%for each differentiable fluent constant $X_i\in \sigma^\mi{diff}$ and for each value $v$ of $\j{Mode}$, where 
and $F_i(X)$ is a fluent formula over $\sigma^{bg}\cup\sigma^\mi{diff}$.  
We assume that an action description has a unique rate declaration \eqref{rate} for each pair of $X_i$ and $v$.
So, by $d/dt[X_i](v)$ we denote the formula $F_i(X)$ in~\eqref{rate}.
The set of all rate declarations \eqref{rate} for each value $v$ of $\j{Mode}$ introduces the following causal law: 
\begin{align*} % \label{integral-cplus}
 \constraint\ & (X_1,\dots, X_n)=(x_1+y_1, \dots, x_n+y_n)\
\after\ (X_1,\dots, X_n)=(x_1, \dots, x_n) \\
&	  \land\ (y_1, \dots, y_n) = \int_0^\delta  (d/dt[X_1](v), \dots, d/dt[X_n](v)) dt \ \\
&	\land\ \j{Mode}=v\ \land\ \j{Dur}=\delta\ \land\  \j{Wait}\mvis\true  \taglabel{integral-cplus}
\end{align*}
where $x_1, \dots, x_n$ and $y_1, \dots, y_n$ are real variables. 

Second, an {\em invariant law} is an expression of the form
\beq
{\bf \alwayst}\ F(X)\ {\bf if}\  \j{Mode}=v 
\eeq{invariant}
where $F(X)$ is a fluent formula of signature $\sigma^\mi{diff}\cup\sigma^{bg}$. %and $G$ is a fluent formula.

We expand each invariant law \eqref{invariant}  into 
\begin{align*}\taglabel{invariant-extended}
 &   \constraint\ \forall t \forall x \big( (0\le t\le \delta)\land \\
 &  \hspace{2cm} \big(x = \big((x_1,\dots, x_n)+\int_0^t (d/dt[X_1](v), \dots, d/dt[X_n](v)) dt\big) \rar F(x)\big) \big)\\
 & \hspace{2cm}\after\ (X_1, \dots, X_n)=(x_1,\dots, x_n) \land \j{Mode}=v\land \j{Dur}=\delta\land\j{Wait}=\true.
\end{align*}

Notice that the causal law uses the universal quantification to express that all values of $X$ during the continuous transition satisfy the formula $F(X)$. 

%------------------------------------------------------------------------------
\section{Encoding Hybrid Transition Systems in ${\cal C}$+ Modulo ODE} \label{sec:non-linear}
%------------------------------------------------------------------------------

\subsection{Representation} 

In this section, we represent the general class of hybrid automata, allowing non-linear hybrid automata with non-convex invariants, in the language of ${\cal C}$+ modulo ODE using the new abbreviations introduced in the previous section.
As before, we assume derivatives are Lipschitz continuous in order to ensure that the solutions to the ODEs are unique.

The translation consists of the same causal laws as those in Section \ref{sec:linear} except for those that account for continuous transitions. Each variable in hybrid automata is identified with a differentiable fluent constant. The representations of the flow and the invariant condition are modified as follows. 
\begin{itemize}

\item 
{\bf Flow}:  
We assume that flow conditions are  written as a set of $\dot{X_i} = F_i(X)$ for each $X_i$ in $\sigma^\mi{diff}$ where $F_i(X)$ is a formula whose free variables are from $X$ only, and assume there is only one such formula for each $X_i$ in each mode. 
For each $v\in V$ and each $X_i\in X$, $D_H$ includes a rate declaration \[
    \derivative\ \of\ X_i\ \iis\ F_i(X)\ \iif\ \j{Mode}=v
\]
which describes the flow of each differentiable fluent constant $X_i$ for the value of $\j{Mode}$.

\item  {\bf Invariant}: For each $v\in V$, $D_H$ includes an invariant law 
\[
\ba l 
     \constraint\ \j{Mode}\mvis v\rar {\sf Inv}_v(X)\\
   {\bf \alwayst}\ {\sf Inv}_v(X)\ \iif\ \j{Mode}=v
\ea
\]

The new {\bf always\_t}  law ensures the invariant is true even during the continuous transition. 
\end{itemize}

The above representation expresses that operative ODEs and invariants are completely determined by the current value of $\j{Mode}$. In turn, one can set the value of the mode by possibly complex conditions over fluents and actions. 

Theorem \ref{thm:ha2cplus} and Lemmas \ref{lem:ha2cplus}, \ref{lem:cplus2ha} remain true even when $H$ is a non-linear hybrid automaton allowing non-convex invariants if we use this version of $D_H$ instead of the previous one. 

\subsection{Turning in the Input Language of ${\tt dReal}$} 

Since the new causal laws are abbreviations of basic causal laws, the translation by \citeN{lee13answer} from a ${\cal C}$+ description into ASPMT  
and a further translation into SMT apply to the extension as well. On the other hand, system ${\tt dReal}$ \cite{gao13dreal}  has a non-standard ODE extension to SMT-LIB2 standard, which succinctly represents integral and universal quantification over time variables (using {\tt integral} and {\tt forall\_t} constructs). 
In its language, ${\tt t}$-variables (variables ending with ${\tt \_t}$) have a special meaning.  $c\_i\_{\tt t}$ is a $t$-variable between timepoint $i$ and $i+1$ that progresses in accordance with ODE specified by some flow condition and is universally quantified to assert that their values during each transition satisfy the invariant condition for that transition (c.f. \eqref{invariant-extended}).

To account for encoding the SMT formula $F$ obtained by the translation into  the input language of ${\tt dReal}$, by $dr(F)$ we denote the set of formulas obtained from $F$ by 
\begin{itemize}
\item  replacing every occurrence of $0\!:\!c$ in $F$ with $c{\tt \_0}$ if $c\in\sigma^\mi{diff}$;
\item  replacing every occurrence of $i\!:\!c$ in $F$ with $c\_(i-1){\tt \_t}$ if $c\in\sigma^\mi{diff}$ and $i>0$;
\item replacing every occurrence of $i\!:\!c$ in $F$ with $c\_i$ if $c\in\sigma$ and $c\notin\sigma^\mi{diff}$
\end{itemize}
for every $i\in\{0,\dots, m-1\}$. 

The translations of the causal laws other than \eqref{rate} and \eqref{invariant} into ASPMT and then into SMT follows the same one in the paper by~\citeN{lee13answer} except that we use $dr(F)$ in place of~$F$. Below we explain how the new causal laws are encoded in the language of ${\tt dReal}$. 

Let $\theta_v$ be the list $(d/dt[X_1](v), \dots, d/dt[X_n](v))$ for all differentiable fluent constants $X_1,\dots, X_n$ in $\sigma^\mi{diff}$. 
The set of rate declaration laws \eqref{rate} describes a unique complete set of ODEs $\theta_v$ for each value $v$ of $\j{Mode}$ and can be expressed in the language of ${\tt dReal}$ as 
\[
   (\texttt{define-ode } \texttt{flow\_}v\ 
       ((\texttt{= d/dt}[X_1]\ F_1),\dots, (\texttt{= d/dt}[X_n]\ F_n))).
\]

In the language of ${\tt dReal}$, the integral  construct 
\begin{lstlisting}
   (integral (0. $\delta$ [$X_1^0$, $\dots$, $X_n^0$] flow_$v$))
\end{lstlisting}
where $X_1^0$, $\dots$, $X_n^0$ are initial values of $X_1, \dots, X_n$, represents the list of values 
\[ 
   (X_1^0,\dots, X_n^0) +  \int_0^\delta (d/dt[X_1](v), \dots, d/dt[X_n](v))\ dt.
\]

Using the integral construct, causal law \eqref{integral-cplus} is turned into the input language of ${\tt dReal}$ as 
\bi\addtolength{\itemsep}{-0mm}
\item if $i=0$,

\begin{lstlisting}[breaklines]
(assert (=> (and ((= mode_0 $v$) (= wait_0 true)))
            (= [$X_1$_0_t, $\dots$, $X_n$_0_t] 
               (integral (0. dur_0 [$X_1$_0_0, $\dots$ , $X_n$_0_0] flow_$v$))))
\end{lstlisting}

\item if $i>1$,
\begin{lstlisting}[breaklines]
(assert (=> (and ((= mode_$i$_$v$) (= wait_$i$ true)))
     (= [$X_1$_$i$_t, $\dots$, $X_n$_$i$_t] 
        (integral (0. dur_$i$ [$X_1$_$(i\!-\!1)$_t, $\dots$ , $X_n$_$(i\!-\!1)$_t] flow_$v$))))
\end{lstlisting}
\ei

The causal law \eqref{invariant-extended}, which stands for invariant law \eqref{invariant}, can be succinctly represented in the language of ${\tt dReal}$ using {\tt forall\_t} construct as 
\begin{lstlisting}
     (assert (forall_t $v$ [0 dur_$i$] $dr(i:F)$)).
\end{lstlisting}

\subsection{Implementation and Example} 

\begin{figure}[h!]
  \includegraphics[height=3.5cm]{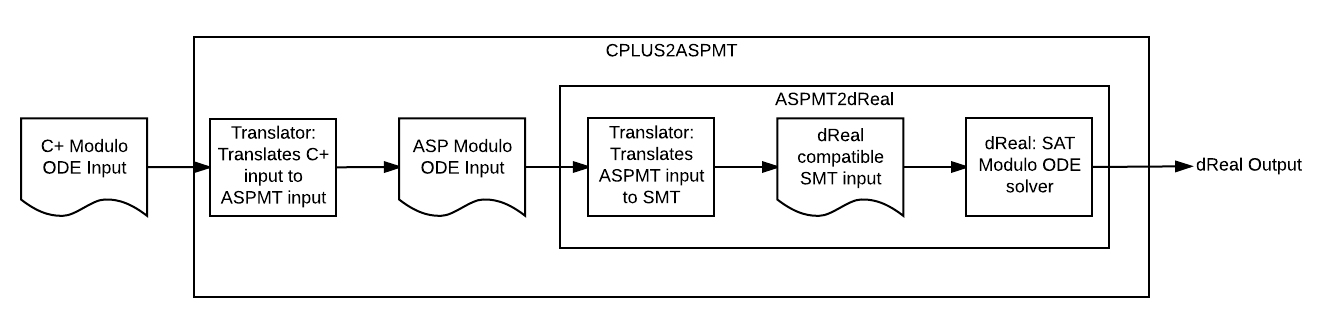}
  \caption{Architecture of system {\sc cplus2ASPMT}}
  \label{fig:cplus2aspmt-arch}
\end{figure}

We implemented a prototype system {\sc cplus2aspmt}, which allows us for representing hybrid transition systems in the action language $\cal{C}$+. 
The system supports an extension of $\cal{C}$+ by adding constructs for ODE support, then translating into an equivalent ASPMT program and finally translating it into the input language of ${\tt dReal}$. The architecture of the system is shown in Figure~\ref{fig:cplus2aspmt-arch}.
The system {\sc cplus2aspmt} is available at \url{http://reasoning.eas.asu.edu/cplus2aspmt}. 

\begin{example}

%\begin{center}
%  \includegraphics[height=7cm]{car-turn-HA.png}
%\end{center}
\begin{minipage}{.4\textwidth}
Let us revisit the car example introduced earlier. 
%The car moves at a constant speed of 1 unit. 
The car is initially at the origin where $x=0$ and $y=0$ and $\theta=0$. Additionally, there are pillars defined by the equations $(x-9)^2+y^2\le 9$, $(x-5)^2+(y-7)^2\le 4$, $(x-12)^2+(y-9)^2\le 4$. The goal is to find a plan such that the car ends up at $x=13$ and $y=0$ without hitting the pillars. The dynamics of the car is as described by~\citeN{corke11robotics}.
\end{minipage}
\begin{minipage}{.55\textwidth}
\begin{center}
  \includegraphics[height=7cm]{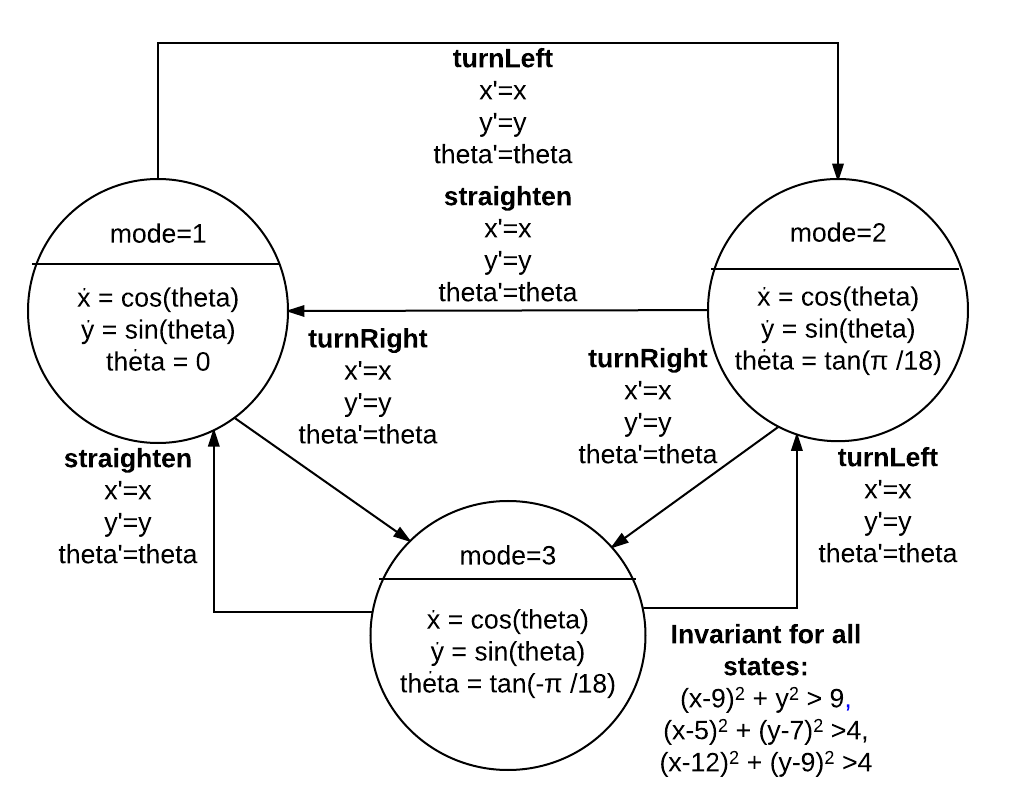}
\end{center}
\end{minipage}

We show some part of the hybrid automaton representation in the input language of {\sc cplus2aspmt}.\footnote{The complete formalization is given in Appendix~C \cite{lee17representing-online}.} First, fluent constants and action constants are declared as follows: 
\begin{lstlisting}[breaklines]
  :- constants
  x            :: differentiableFluent(real[0..40]);
  y            :: differentiableFluent(real[-50..50]);
  theta        :: differentiableFluent(real[-50..50]);
  straighten, turnLeft, turnRight   :: exogenousAction.
\end{lstlisting}
%  mode         :: inertialFluent(real[1..3]);
%  wait         :: action;
%  dur          :: exogenousAction(real[0..30]).
(In the ODE support mode, {\tt mode}, {\tt wait}, and {\tt duration} are implicitly declared by the system.)

The derivative of the differentiable fluent constants for mode=2 (movingLeft) is declared as follows:
\begin{lstlisting}[breaklines]
  derivative of x is cos(theta) if mode=2.
  derivative of y is sin(theta) if mode=2.    
  derivative of theta is tan(pi/18) if mode=2.
\end{lstlisting}

The invariants for avoiding the collision with the bottom pillar are represented as follows: 

\begin{lstlisting}[breaklines]
  constraint x=X & y=Y ->> ((X-9)*(X-9) + Y*Y > 9).
  always_t (x=X & y=Y ->> ((X-9)*(X-9) + Y*Y > 9)) if mode=V.
\end{lstlisting} 

The precondition and effects of {\tt turnLeft} action are represented as follows:
\begin{lstlisting}[breaklines]
  nonexecutable turnLeft if mode=2.
  turnLeft causes mode=2.
  turnLeft causes dur=0
\end{lstlisting} 
\end{example}

Figure~\ref{fig:car-turn-5-3} (a) illustrates the trajectory returned by the system when we instruct it to find a plan of length $5$ to reach the goal position. For the path of length $3$, the system returned the trajectory in Figure~\ref{fig:car-turn-5-3}~(b). 
The system could not find a plan of length $1$ because of the $\alwayst$ proposition asserting the invariant during the continuous transition. If we remove the proposition, the system returns the 
physically unrealizable plan in Figure~\ref{fig:robot}~(b).
% Here the car moves straight from origin to $(13,0)$ in a single step.

\begin{figure}[h!]
  \includegraphics[height=1.8cm, width=13cm]{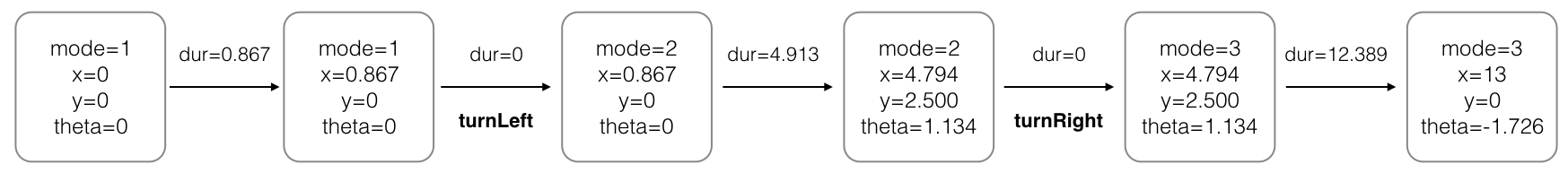}
  \includegraphics[height=2cm]{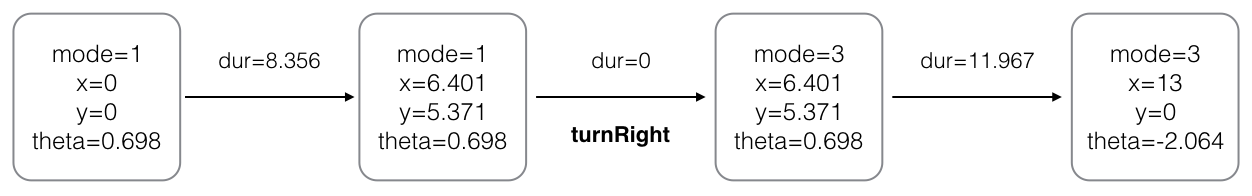}
  \caption{Output of Car Example \hspace{1cm} (a) top: maxstep=5\hspace{1cm} (b) bottom: maxstep=3}
  \label{fig:car-turn-5-3}
\end{figure}

%\vspace{-8mm}
%------------------------------------------------------------------------------
\section{Related Work}
%------------------------------------------------------------------------------

Due to space restriction, we list only some of the related work. 
PDDL+ \cite{fox06modelling} is a planning description language to model mixed discrete and continuous changes. The semantics is defined by mapping primitives of PDDL+ to hybrid automata. Most PDDL+ planners assume that the continuous change is linear, while a recent paper by \citeN{bryce15smt}, closely related to our work, presents an SMT encoding of a PDDL+ description that is able to perform reasoning about non-linear hybrid automata. However, no dedicated translator from PDDL+ to SMT is provided. The fact that both PDDL+ and ${\cal C}$+ can be turned into SMT may tell us how the two high-level languages are related to each other, which we leave for future work.  In the paper by~\citeN{bryce15smt}, the encoding was in the language of ${\tt dReach}$ with the emphasis on extending ${\tt dReach}$ with planning-specific heuristics to find a valid and possibly optimized mode path. The heuristic search has not been considered in the work of {\sc cplus2aspmt}, which makes the system less scalable (see Appendix~D \cite{lee17representing-online} for some experimental result).

SMT solvers have been actively used in formal verification of hybrid systems (e.g., the papers by \citeN{cimatti12smt}; and by \citeN{alur11formal}), but mostly focused on linear differential equations. ${\tt dReal}$ is an exception.

Instead of SMT solvers, constraint ASP solvers may also be used for  hybrid automata reasoning. \citeN{balduccini16pddl+} shows PDDL+ primitives can be encoded in the language of constraint ASP solvers, and compared its performance with other PDDL+ computing approaches including ${\tt dReal}$. On the other hand, unlike our work, the encoding checks continuous invariants at discretized timepoints and no proof of the soundness of the translation is given. Constraint ASP solvers do not support $\delta$-satisfiability checking. Thus, the general method of invariant checking during continuous transitions as in ${\tt dReal}$ is not yet available there. 

Action language ${\cal H}$~\cite{chin04,chintabathina12anew} is another action language that can model hybrid transitions, but its semantics does not describe the hybrid transition systems of the same kind as hybrid automata.  
Instead of using SMT solvers, an implementation of $\cal{H}$ is by a translation into the language $\cal{AC}$ \cite{mellarkod08integrating}, which extends ASP with constraints. 
Language $\cal{H}$ does not provide support for continuous evolution via ODEs and invariant checking during the continuous transition. 

ASPMT is also related to HEX programs, which are an extension of answer set programs with external computation sources.  HEX programs with numerical external computation have been used for hybrid reasoning in games and robotics \cite{calimeri16angry,erdem16systematic}.

%\vspace{-3mm}
%------------------------------------------------------------------------------
\section{Conclusion}
%------------------------------------------------------------------------------

We represented hybrid automata in action language modulo theories. As our action language is based on ASPMT, which in turn is founded on the basis of ASP and SMT, it enjoys the development in SMT solving techniques as well as the expressivity of ASP language.  We presented an action language modulo ODE, which lifts the concept of SMT modulo ODE to the action language level.

One strong assumption we imposed is that an action description has to specify {\sl complete} ODEs. This is because existing SMT solving techniques are not yet mature enough to handle composition of partial ODEs.  In the paper by \citeN{gao13dreal}, such extension is left for the future work using new commands {\tt pintegral} and {\tt connect}. We expect that it is possible to extend the action language to express partial ODEs in accordance with this extension.

In our representation of hybrid automata in action language ${\cal C}$+, we use only a fragment of the action language, which does not use other features, such as additive fluents, statically determined fluents, action attributes,  defeasible causal laws. One may write a more elaboration tolerant high-level action description for hybrid domains using these features. 
 
SMT solvers are becoming a key enabling technology in formal verification in hybrid systems. Nonetheless, modeling in the low-level language of SMT is non-trivial. We expect the high-level action languages may facilitate encoding efforts.

\medskip\noindent
{\bf Acknowledgements:} 
We are grateful to  Sicun Gao and Soonho Kong for their help on running ${\tt dReal}$ and ${\tt dReach}$ systems, to Daniel Bryce for his help with benchmark problems, and the anonymous referees for their useful comments. This work was partially supported by the National Science Foundation under Grants IIS-1319794 and IIS-1526301.

\bibliographystyle{acmtrans}
%\bibliography{bib,bib2}

\BOCC
\setcounter{page}{1}
\title{Appendix: Representing Hybrid Automata by Action Language Modulo Theories} 

\begin{center}
{\large\textnormal{Online appendix for the paper}}   \\
\medskip
{\Large {\sl Representing Hybrid Automata \\ by Action Language Modulo Theories}
\\
\medskip
{\large\textnormal{published in Theory and Practice of Logic Programming}}
}

\medskip
Joohyung Lee, Nikhil Loney\\ 
{\sl School of Computing, Informatics and Decision Systems Engineering \\
Arizona State University, Tempe, AZ, USA}

\medskip
Yunsong Meng\\
{\sl Houzz, Inc. \\
Palo Alto, CA, USA }

%\author[Lee, Loney \& Meng]{Joohyung Lee, Nikhil Loney, and Yunsong Meng}
\end{center}

\thispagestyle{empty}
\EOCC
\begin{appendix}

\BOCCC
%------------------------------------------------------------------------------
\section{Review: SMT and ASPMT} 
%------------------------------------------------------------------------------

Formally, an SMT instance is a formula in many-sorted first-order
logic, where some designated function and predicate constants are
constrained by some fixed background interpretation. SMT is the
problem of determining whether such a formula has a model that expands
the background interpretation~\cite{barrett09satisfiability}. Some of the background theories relevant to this paper are as follows:

%{\cblu 
\begin{itemize}
\item QF\_LRA (Quantifier free linear real arithmetic): It includes the following binary function constants that represent arithmetic functions: $+$, $-$, $\times$ and $/$ and the following binary predicates that represent comparison operators: $\ge$, $\le$, $<$ and $>$. In essence, Boolean combinations of inequations between linear polynomials over real variables. 
\item QF\_NRA (Quantifier free non-linear real arithmetic):  It includes the following binary function constants that represent arithmetic functions: $+$, $-$, $*$ and $/$ and the following binary predicates that represent comparison operators: $\ge$, $\le$, $<$ and $>$. It is similar to QF\_LRA with the difference that it is additionally capable of handling non-linear polynomials, trigonometric functions($sin$, $cos$, $tan$) and other non-linear theories.
\item QF\_NRA\_ODE (Quantifier free non-linear real arithmetic for ODEs):  This is a special background theory for dealing with ODEs. In addition to QF\_NRA, QF\_NRA\_ODE adds support for functions like integral and $\forall^t$.
\end{itemize}

{\cblu [[ for all t ]]}

The syntax of ASPMT is the same as that of SMT. Let $\sigma^{bg}$ be
the (many-sorted) signature of the background theory~$bg$. An
interpretation of $\sigma^{bg}$ is called a {\em background
  interpretation} if it satisfies the background theory. For instance,
in the theory of reals, we assume that $\sigma^{bg}$ contains the set
$\mathcal{R}$ of symbols for all real numbers, the set of arithmetic
functions over real numbers, and the set $\{<, >, \le, \ge\}$ of
binary predicates over real numbers. Background interpretations
interpret these symbols in the standard way.
%}

Let $\sigma$ be a signature that is disjoint from $\sigma^{bg}$.
We say that an interpretation $I$ of $\sigma$ satisfies $F$
w.r.t. the background theory $bg$, denoted by $I\models_{bg} F$,
if there is a background interpretation $J$ of $\sigma^{bg}$ that has
the same universe as $I$, and $I\cup J$ satisfies $F$.
For any ASPMT sentence $F$ with background theory
$\sigma^{bg}$, interpretation $I$ is a stable model of $F$ relative
to~${\bf c}$ (w.r.t. background theory $\sigma^{bg}$) if
$I\models_{bg} \sm[F; \bC]$ (we refer to \cite{bartholomew13functional} for the definition of the $\sm$ operator).

%We will often write $G\ar F$, in a rule form as in logic programs,  to denote the universal closure of $F\rar G$. 
\EOCCC

%------------------------------------------------------------------------------
\section{Review: ${\cal C}$+} \label{sec:cplus}
%------------------------------------------------------------------------------

%---------------------------------------------------------------------
\subsection{Syntax of $\cal C$+}\label{ssec:cplus-syntax} 
%---------------------------------------------------------------------

${\cal C}$+ was originally defined as a propositional language \cite{giu04}. 
In this section we review its reformulation in terms of ASPMT \cite{lee13answer}.

%formulate in terms of ASPMT. We refer the reader to \cite{bartholomew13functional} for the definition of ASPMT. %\footnote{This is not novel; ${\cal C}$+ was formulated in terms of ASPMT in \cite{lee13action}.}

We consider a many-sorted first-order signature $\sigma$ that is partitioned into three sub-signatures: the set $\sigma^\mi{fl}$ of object constants called {\em fluent constants}, the set $\sigma^\mi{act}$ of object constants called {\em action constants}, and the background signature~$\sigma^{bg}$. The signature $\sigma^\mi{fl}$ is further partitioned into the set $\sigma^{sim}$ of {\em simple} fluent constants and the set $\sigma^{sd}$ of {\em statically determined} fluent constants.

A {\em fluent formula} is a formula of signature~$\sigma^\mi{fl}\cup\sigma^{bg}$. An {\em action formula} is a formula of~$\sigma^\mi{act}\cup\sigma^{bg}$ that contains at least one action constant and no fluent constants.

A {\em static law} is an expression of the form
\beq
  \caused\ F\ \iif\ G
\eeq{static}
where $F$ and $G$ are fluent formulas.  

An {\sl action dynamic law} is an expression of the form~(\ref{static}) in which $F$ is an action formula and $G$ is a formula.

A {\sl fluent dynamic law} is an expression of the form
\beq
 \caused\ F\ \iif\ G\ \after\ H
\eeq{dynamic}
where~$F$ and~$G$ are fluent formulas and $H$ is a formula, provided that~$F$ does not contain statically determined constants. 

%Static laws can be used to talk about causal dependencies between fluents in the same state; action dynamic laws can be used to express causal dependencies between concurrently executed actions; fluent dynamic laws can be used to describe direct effects of actions. 

A {\sl causal law} is a static law, an action dynamic law, or a fluent dynamic law.
An {\sl action description} is a finite set of causal laws.

The formula~$F$ in causal laws~(\ref{static}) and~(\ref{dynamic}) is called the {\sl head}.  

We call an action description {\em definite} if the head $F$ of every
causal law \eqref{static} and \eqref{dynamic} is an atomic formula
that is $(\sigma^{fl}\cup\sigma^{act})$-plain. 
%Throughout this paper we
%consider definite action descriptions only.
\footnote{
For any function constant $f$, we say that a first-order formula is
{\em $f$-plain} if each atomic formula in it
% does not contain $f$, or
% is of the form $f({\bf t}) = t_1$ where ${\bf t}$ is a list of
% terms not containing $f$, and $t_1$ is a term not containing~$f$.
\begin{itemize}\addtolength{\itemsep}{-0.5mm}
\item  does not contain $f$, or
\item  is of the form $f({\bf t}) = t_1$ where ${\bf t}$ is a list of
  terms not containing $f$, and $t_1$ is a term not containing~$f$.
\end{itemize}
For any list $\bC$ of predicate and function constants, we say that $F$
is $\bC$-plain if $F$ is $f$-plain for each function constant $f$ in
$\bC$.
}

%---------------------------------------------------------------------
\subsection{Semantics  of $\cal C$+}\label{ssec:cplus-semantics}
%---------------------------------------------------------------------

For a signature $\sigma$ and a nonnegative integer $i$, expression
\hbox{$i:\sigma$} is the signature consisting of the pairs $i:c$ such
that \hbox{$c\in \sigma$}, and the value sort of $i:c$ is the same as
the value sort of $c$. Similarly, if $s$ is an interpretation of
$\sigma$, expression $i:s$ is an interpretation of $i:\sigma$ such
that $c^s = (i:c)^{i:s}$.

For any action description $D$ of signature
\hbox{$\sigma^\mi{fl}\cup\sigma^\mi{act}\cup\sigma^{bg}$} and any nonnegative
integer $m$, the ASPMT program $D_m$ is defined as follows. The
signature of $D_m$ is 
$0\!:\!\sigma^\mi{fl}\cup\dots\cup
m\!:\!\sigma^\mi{fl}\cup0\!:\!\sigma^\mi{act}\cup\dots\cup
(m\!-\!1)\!:\!\sigma^\mi{act}\cup\sigma^{bg}$.
By $i:F$ we denote the result of inserting $i:$ in front of every
occurrence of every fluent and action constant in a formula~$F$.

ASPMT program $D_m$ is the conjunction of
\[
  i\!:\!G \rar i\!:\!F
\]
for every static law \eqref{static} in $D$ and every
$i\in\{0,\dots,m\}$, and for every action dynamic law \eqref{static}
in $D$ and every \hbox{$i\in\{0,\dots,m\!-\!1\}$};
\[
  (i\!+\!1)\!:\! G \land i\!:\!H \rar (i\!+\!1)\!:\!F
\]
for every fluent dynamic law \eqref{dynamic} in $D$ and every
$i\in\{0,\dots, m-1\}$.

The transition system represented by an action description~$D$
consists of states (vertices) and transitions (edges).
A {\em state} is an interpretation $s$ of $\sigma^\mi{fl}$ such that
\hbox{$0\!:\!s\models_{bg} \sm[D_0;\ 0\!\!:\!\!\sigma^{sd}]$}.
%States are the vertices of the transition system represented by $D$.
A {\em transition} is a triple $\langle s,e,s'\rangle$, where $s$ and
$s'$ are interpretations of $\sigma^\mi{fl}$ and $e$ is an interpretation
of $\sigma^\mi{act}$, such that
\[
\ba l
  (0\!:\! s)\cup (0\!:\!e) \cup (1\!:\!s')\models_{bg} \sm[D_1;\ (0\!:\!\sigma^{sd})\cup (0\!:\!\sigma^\mi{act})\cup
  (1\!:\!\sigma^\mi{fl})]\ . 
\ea
\]

The definition of the transition system above implicitly relies on the
following property of transitions:

\begin{thm}\label{thm:state}\optional{thm:state}\cite[Theorem~3]{lee13answer}
For every transition $\langle s,e,s'\rangle$, $s$ and $s'$ are states.
\end{thm}

The following theorem states the correspondence between the stable
models of $D_m$ and the paths in the transition system represented by
$D$:

\begin{thm}\label{thm:transition}\optional{thm:transition}\cite[Theorem~4]{lee13answer}
\[
\ba l
(0\!:\!s_0) \cup (0\!:\!e_0) \cup (1\!:\!s_1) \cup (1\!:\!e_1) \cup
        \cdots \cup (m\!:\!s_m) \\
\hspace{0.2cm}\models_{bg}
\sm[D_m;\ (0\!:\!\sigma^{sd})\cup
         (0\!:\!\sigma^\mi{act})\cup
         (1\!:\!\sigma^\mi{fl})\cup
         (1\!:\!\sigma^\mi{act})  
 \cup\dots\cup
        (m\!-\!1\!:\!\sigma^\mi{act})\cup
        (m\!:\!\sigma^\mi{fl})]
\ea
\]
iff each triple $\langle s_i,e_i,s_{i+1}\rangle$ $(0\le i<m)$ is a
transition.
\end{thm}

%[[ Simple extension that allows infinite domains ]]

It is known that when $D$ is definite, ASPMT program $D_m$ that is
obtained from action description $D$ is always tight. Functional completion~\cite{bartholomew13functional} on ASPMT can be applied to turn $D_m$ into an SMT instance.

\subsection{Some Useful Abbreviations of ${\cal C}$+ Causal Laws}

This section explains the abbreviations of ${\cal C}$+ causal laws used in the paper.
%\noindent{\bf 1.}
%A static law or an action dynamic law of the form
%$${\bf caused}\ F\ {\bf if}\ \top$$
%can be written as
%$${\bf caused}\ F.$$
%
%\medskip\noindent{\bf 2.}
%A fluent dynamic law of the form
%$${\bf caused}\ F\ {\bf if}\ \top\ {\bf after}\ H$$
%can be written as
%$${\bf caused}\ F\ {\bf after}\ H.$$

\medskip\noindent{\bf 1.}
A static law of the form
$${\bf caused}\ \bot\ {\bf if}\ \neg F$$
can be written as
$${\bf constraint}\ F.$$

\medskip\noindent{\bf 2.}
A fluent dynamic law of the form
$${\bf caused}\ \bot\ {\bf if}\ \neg F\ {\bf after}\ G$$
can be written as
$${\bf constraint}\ F\ {\bf after}\ G.$$

%\medskip\noindent{\bf 5.}
%An expression of the form
%$$
%{\bf rigid}\ c
%$$
%where~$c$ is a fluent constant stands for the set of causal laws
%$${\bf constraint}\ c\mvis v\ {\bf after}\ c\mvis v$$
%for all $v\in\j{Dom}(c)$.

%\medskip\noindent{\bf 6.}
%A fluent dynamic law of the form
%$${\bf caused}\ \bot\ {\bf after}\ \neg F$$
%can be written as
%$${\bf always}\ F.$$

\medskip\noindent{\bf 3.}
A fluent dynamic law of the form
$${\bf caused}\ \bot\ {\bf after}\ F\wedge G$$
where $F$ is an action formula can be written as
\beq
{\bf nonexecutable}\ F\ {\bf if}\ G.
\eeq{a8}

\medskip\noindent{\bf 4.}
An expression of the form
\beq
F\ {\bf causes}\ G\ {\bf if}\ H
\eeq{a9}
where $F$ is an action formula stands for the fluent dynamic law
$${\bf caused}\ G\ {\bf after}\ F\wedge H$$
if $G$ is a fluent formula,\footnote{It is clear that the expression in the
previous line is a fluent dynamic law only when~$G$ does not contain
statically determined fluent constants.  Similar remarks
can be made in connection with many of the abbreviations introduced below.}
and for the action dynamic law
$${\bf caused}\ G\ {\bf if}\ F\wedge H$$
if $G$ is an action formula.

\medskip\noindent{\bf 5.}
An expression of the form
\beq
{\bf default}\ F\ {\bf if}\ G
\eeq{a12}
stands for the causal law
$${\bf caused}\ \{F\}^{\rm ch}\ {\bf if}\ G.$$

\medskip\noindent{\bf 6.}
An expression  of the form
$$%\beq
{\bf default}\ F\ {\bf if}\ G\ {\bf after}\ H
$$%\eeq{a11}
stands for the fluent dynamic law \footnote{%
$\{F\}^{\rm ch}$ stands for choice formula $F\lor\neg F$.}
$${\bf caused}\ \{F\}^{\rm ch}\ {\bf if}\ G\ {\bf after}\ H.$$

\medskip\noindent{\bf 7.}
An expression of the form
\beq
{\bf exogenous}\ c\ {\bf if}\ G
\eeq{a13}
where~$c$ is a constant stands for the set of causal laws
$${\bf default}\ c\mvis v\ {\bf if}\ G$$
for all $v\in\j{Dom}(c)$.

%\medskip\noindent{\bf 12.}
%An expression of the form
%\beq
%F\ {\bf may\ cause}\ G\ {\bf if}\ H
%\eeq{a14}
%where $F$ is an action formula stands for the fluent dynamic law
%$${\bf default}\ G\ {\bf after}\ F\wedge H$$
%if~$G$ is a fluent formula, and for the action dynamic law
%$${\bf default}\ G\ {\bf if}\ F\wedge H$$
%if~$G$ is an action formula.

\medskip\noindent{\bf 8.}
An expression of the form
\beq
{\bf inertial}\ c\ {\bf if}\ G
\eeq{a15}
where~$c$ is a fluent constant stands for the set of fluent dynamic laws
$${\bf default}\ c\mvis v\ {\bf after}\ c\mvis v\wedge G$$
for all $v\in\j{Dom(c)}$.

\medskip\noindent{\bf 9.}
In the abbreviations of causal laws above, "${\bf if}\ G$" and "${\bf if}\ H$" can be omitted if $G$ and $H$ are~$\top$. 
%If any of the abbreviations (\ref{a9})--(\ref{a15}) ends with
%$${\bf if}\ \top$$
%then this part of the expression can be dropped.

%\include{cplus-ode-proofs}

%------------------------------------------------------------------------------
\section{Proofs} \label{sec:proofs}
%------------------------------------------------------------------------------

%------------------------------------------------------------------------------
\subsection{Proof of $\textbf{Theorem \ref{thm:ha2cplus}}$} \label{ssec:proof-ha2cplus}
%------------------------------------------------------------------------------

We assume the case for linear hybrid automata with convex invariants. The proof of the general case of non-linear hybrid automata with non-convex invariants are mostly similar except for the difference in ${\sf Flow}$ and ${\sf Inv}$ conditions. 

\medskip
\noindent{\bf Theorem~\ref{thm:ha2cplus} \optional{thm:ha2cplus}}\\ 
{
There is a 1:1 correspondence between the paths of the transition system of a Hybrid automata $H$ and the paths of the transition system of the $\mathcal{C+}$ action description $D_H$.
}
\medskip

%\begin{proof} 
The proof is immediate from Lemma~\ref{lem:ha2cplus} and Lemma~\ref{lem:cplus2ha}, which are proven below.
%\end{proof}

\subsubsection{Proof of Lemma \ref{lem:ha2cplus}} \label{sssec:lemma-ha2cplus}

\BOCCC
\begin{lemma}\label{lemma:convex}
For any linear function 
$f:[0,\sigma]\to\mathcal{R}^n$ and a conjunction of linear inequalities $P(X)$ where $X$ ranges over $\mathcal{R}^n$, we have
\[
   \forall \epsilon\in{(0,\sigma)}(P(f(0))\land P(f(\sigma))\to P(\epsilon)).
\]
\end{lemma}

\begin{proof} 
Since $P(X)$ is a conjunction of linear inequalities, we know from \cite[Sec 2.2.4]{boyd04convex}) the values of $X$ that satisfies $P(X)$ must form a convex region\footnote{A convex region is a set of points such that, given any two points $A$, $B$ in that set, the line $AB$ joining them lies entirely within that set. If $P$ is a co nvex set and $x_1 ... x_k$ are any points in it, then
$x=\sum_{i=1}^{k} \lambda_i x_{i}$ is also in $P$, where each $\lambda_i > 0$ and $\sum_{i=1}^{k} \lambda_i=1$.
} in  $\mathcal{R}^n$. Since $f(t)$ is a linear function, from the fact that $P(f(0))$ and $P(f(\sigma))$ are true, it follows that for any $\epsilon \in (0, \sigma)$, $P(f(\epsilon))$ is true.
\end{proof}
\EOCCC

\begin{lemma}\label{lemma:linear} 
Let $H$ be a  linear hybrid automaton with convex invariants, and let 
\[
   (v,  r) \xrightarrow{\sigma} (v,  r')
\] 
be a transition in $T_H$ such that $\sigma\in \mathcal{R}_{> 0}$. Function~$f(t)\mvis r+ t\!\times\!( r'\!-\! r)\!/ \! \sigma$ is a linear differentiable function from $[0, \sigma]$ to $\mathcal{R}^n$, with the first derivative 
$\dot{f} : [0, \sigma] \rar \mathcal{R}^n$ such that
(i) $f(0) \mvis  r$ and $f(\sigma) \mvis  r'$ and 
(ii) for all reals $\epsilon \in (0, \sigma)$, both ${\sf Inv}_v(f(\epsilon))$
and ${\sf Flow}_v({\dot{f}}(\epsilon))$ are true.
\end{lemma}

\begin{proof}
We check that $f$ satisfies the above conditions:
\begin{itemize}
\item  It is clear that $f(t)$ is differentiable over $t\in[0, \sigma]$, $f(0) \mvis  r$ and $f(\sigma) \mvis  r'$.

\item  Since $(v,  r)$ and $(v,  r')$ are states of $T_H$, it follows that ${\sf Inv}_{v}(f(0))$ and ${\sf Inv}_{v}(f(\sigma))$ are true. 
Since the values of $X$ that makes ${\sf Inv}_v(X)$ form a convex region in ${\cal R}^n$ and $f(t)$ is a linear function, it follows that for $\epsilon \in (0, \sigma)$, ${\sf Inv}_{v}(f(\epsilon))$ is true.

\item Since $(v,  r) \xrightarrow{\sigma}
   (v,  r')$ is a transition in $T_H$, it follows that there is a function $g$ such that (i) $g$ is differentiable in $[0,\sigma]$, (ii) for any $\epsilon \in (0, \sigma)$, ${\sf Flow}_{v}({\dot{g}}(\epsilon))$ is true, (iii) $g(0)= r$ and $g(\sigma)= r'$. Since $g$ is continuous on $[0,\sigma]$ (differentiability implies continuity) and differentiable on $(0, \sigma)$, by the mean value theorem\footnote{http://en.wikipedia.org/wiki/Mean\_value\_theorem}, there is a point $c\in (0, \sigma)$ such that ${\dot{g}}(c)=( r'\!-\! r)\!/\! \sigma$. Consequently, ${\sf Flow}_{v}(( r'\!-\! r)\!/\! \sigma)$ is true. As a result, we get ${\sf Flow}_v({\dot{f}}(\epsilon))$ is true for all $\epsilon \in (0, \sigma)$.
\end{itemize}
\end{proof}

In the following two lemmas, $s_i, a_i, s_{i+1}$ are defined as in Lemma~\ref{lem:ha2cplus}.

\begin{lemma}\label{lemma:state-linear}
For each $i\geq 0$, $s_i$ is a state in the transition system of $D_H$.
\end{lemma}

\begin{proof}
Since $D_H$ does not contain statically determined fluent constants and every simple fluent constant is declared exogenous, it is sufficient to prove 
\[
   0\!:\!s_i\models_{bg} \sm[(D_H)_0;\ \emptyset],
\]
while $\sm[(D_H)_0;\ \emptyset]$ is equivalent to the conjunction of 
\beq
   0\!:\!\j{Mode}=v\rar 0\!:\!{\sf Inv}_v(X)
\eeq{prop:ha-c-linear-0}
for each $v\in V$. Since $p$ is a path, for each $i\ge 0$, $(v_{i},r_{i})$ is a state in $T_H$. By the definition of a hybrid transition system, ${\sf Inv}_{v_{i}}(r_{i})$ is true. Since $s_i \models_{bg} (\j{Mode},X)=(v_i,r_i)$, we have 
$0\!:\!s_i\models_{bg}\eqref{prop:ha-c-linear-0}$.
\end{proof}

\begin{lemma}\label{lemma:edge-linear}
For each $i\ge 0$, $\langle s_i,a_i,s_{i+1}\rangle$ is a transition in the transition system of $D_H$.
\end{lemma}

\begin{proof}
By definition, we are to show that
\beq 
   0\!:\!s_i \cup 0\!:\!a_i \cup 1\!:\! s_{i+1}\models_{bg} 
    \sm[(D_H)_1; 0\!:\!\sigma^{act}\cup 1\!:\!\sigma^{fl}].
\eeq{prop:ha-c-1-2}
We check that $(D_H)_1$ is tight, so that \eqref{prop:ha-c-1-2} is equivalent to
\[
   0\!:\!s_i \cup 0\!:\!a_i \cup 1\!:\! s_{i+1}\models_{bg} 
    \comp[(D_H)_1; 0\!:\!\sigma^{act}\cup 1\!:\!\sigma^{fl}],
\] 
where the completion $\comp[(D_H)_1; 0\!:\!\sigma^{act}\cup 1\!:\!\sigma^{fl}]$ is equivalent to the conjunction of the following formulas:
%
%From \cite{ferraris11stable} it is clear that for any tight formula $F$, $\sm[F]$ is equivalent to the completion $Comp[F]$. Additionally, it is also clear that $\sm[F]\ \land\ Choice({\bf c})$ is equivalent to $\sm[F]$ where $Choice({\bf c})$ represents the choice rules for all non-intensional constants ${\bf c}$.
%\j{Mode}
%$Comp[(D_H)_1;0\!:\!\sigma^{act}\cup 1:\sigma^{fl}]$ is equivalent to the conjunction of following formulas:

\begin{itemize}
\item Formula $\j{FLOW}$, which is the conjunction of
\beq
    {\sf Flow}_v((1\!:\!X-0\!:\!X)/t)\ar 0\!:\!\j{Mode}=v\ \land\ 0\!:\!\j{Dur}=t\ \land\  0\!:\!\j{Wait}=\true\ \land\ t>0
\eeq{prop:ha-c-linear-1-4}
and
\beq
    1\!:\!X=0\!:\!X\ \ar  0\!:\!\j{Mode}=v\ \land\ 0\!:\!\j{Dur}=0\ \land\  0\!:\!\j{Wait}=\true
\eeq{prop:ha-c-linear-1-4a}
for each $v\in V$.

\item Formula $\j{INV}$, which is the conjunction of
\beq
    k:{\sf Inv}_v(X)\ \ar\ k:\j{Mode}=v
\eeq{prop:ha-c-linear-1-3}
for  each $k\in\{0,1\}$ and each $v\in V$.

\item Formula $\j{WAIT}$, which is the conjunction of
\[
\ba l
    0\!:\!\j{Wait}=\false\ \lrar\ \bigvee_{e\in E} 0\!:\!{\sf hevent}(e)\mvis\true.
%    0\!:\!\j{Wait}=\true\ \lrar\ \neg\neg 0\!:\!\j{Wait}=true
\ea
\]%for each $v\in V$.

\item Formula $\j{GUARD}$, which is the conjunction of
\beq
 \bot\ \ar\ 0\!:\!{\sf hevent}(e)\mvis\true\ \land\ 0\!:\!\neg {\sf Guard}_e(X)
\eeq{prop:ha-c-linear-1-5}
for each edge $e\in E$.

\item Formula $\j{RESET}$, which is the conjunction of
\[
\ba l
   {\sf Reset}_e(0\!:\!X, 1\!:\!X)\ \ar\  0\!:\!{\sf hevent}(e)\mvis\true
\ea
\]
for each edge $e=(v_1,v_2)\in E$.

\item Formula $\j{MODE}$, which is the conjunction of
\[
\ba l
    \bot\ar 0\!:\!{\sf hevent}(e)\mvis\true\ \land\ \neg(0\!:\!\j{Mode}=v_1) 
\ea
\] 
for each $e=(v_1,v_2)\in E$;
\[
\ba l
    1:\j{Mode}=v\ \lrar\ \bigvee_{\{v'\mid (v',v)\in E\}} 0\!:\!{\sf hevent}(v',v)\mvis\true 
                \lor\ 0\!:\!\j{Mode}=v
                %                (1:\j{Mode}=v_2\ \land\ 0\!:\!\j{Mode}=v_2)
%    \bot\ar 0\!:\!{\sf hevent}(e)\mvis\true\ \land\ \neg(0\!:\!\j{Mode}=v_1) \\
%    1:\j{Mode}=v_2\ \lrar\ (0\!:\!\j{Mode}=v_1 \land 0\!:\!{\sf hevent}(e))\ \lor\ (1:\j{Mode}=v_2\ \land\ 0\!:\!\j{Mode}=v_2)
\ea
\]
for each $v\in V$.
%for each edge $e=(v_1,v_2)\in E$.

\item Formula $\j{DURATION}$, which is the conjunction of
\[
\ba l
    0\!:\!\j{Dur}=0\ \ar\ \bigvee_{e\in E}0\!:\!{\sf hevent}(e)
\ea
\]
for each edge $e\in E$.
\end{itemize}

We will show that $0\!:\!s_i\cup 0\!:\!a_i\cup 1\!:\! s_{i+1}$ satisfies each of the formulas above. 
First, we check $\j{INV}$.
\begin{itemize}
\item $\j{INV}$: From the fact that $(v_{i},  r_{i})$ and $(v_{i+1},  r_{i+1})$ are states in $T_H$, by the definition of a hybrid transition system, ${\sf Inv}_{v_{i}}( r_{i})$ and ${\sf Inv}_{v_{i+1}}( r_{i+1})$ are true. Note that
$s_i \models_{bg} (\j{Mode},  X) =  (v_{i},  r_{i})$ and 
$s_{i+1}\models_{bg} (\j{Mode},  X) =  (v_{i+1},  r_{i+1})$. As a result,
\[
\ba {l}  
  0\!:\!s_{i}\models_{bg} (0\!:\!\j{Mode}=v\ \rar\ 0\!:\!{\sf Inv}_{v_{i}}(X))\\
  1\!:\!s_{i+1}\models_{bg} (1\!:\!\j{Mode}=v\ \rar\ 1\!:\!{\sf Inv}_{v_{i+1}}(X)). \ea
\]
Hence $0\!:\!s_i \cup 0\!:\!a_i \cup 1\!:\! s_{i+1}\models_{bg} \j{INV}$.
\end{itemize}

%$\j{FLOW}$, $\j{GUARD}$, $\j{RESET}$, $\j{MODE}$, $\j{DURATION}$, $\j{WAIT}$. 
%From the fact that $(v_{(i)} x_{(i)}), (v_{(i+1)}, x_{(i+1)})$ are the states of $T_H$, 
Next, we check the remaining formulas. 
From the definition of $T_H$, there are two cases for the value of $\sigma_i$.

\smallskip\noindent{\sl Case 1:}  
$\sigma_i={\sf hevent}(e)$ where $e=(v_{i},v_{i+1})$. It follows from the construction of $p'$ that $(\j{Dur})^{a_i}=0$, $({\sf hevent}(e))^{a_i}=\true$, $({\sf hevent}(e'))^{a_i}=\false$ for all $e'\ne e$ and $(\j{Wait})^{a_i}=\false$. 
%Since $\j{FLOW}$ is trivially satisfied, it is sufficient to consider the remaining formulas. 
%only $\j{GUARD},\j{RESET},\j{MODE},\j{DURATION},\j{WAIT}$,$\j{INV}$.

From the fact that
\[
  (v_{i}, r_{i})\xrightarrow{\sigma_i}(v_{i+1}, r_{i+1})
\]
is a transition in $T_H$  and that $\sigma_i= {\sf hevent}(e)$, it follows from the definition of a hybrid transition system that ${\sf Guard}_e(r_{i})$ and ${\sf Reset}_e(r_{i}, r_{i+1})$ are true.
\begin{itemize}
\item $\j{FLOW}$: Since $0:a_i\models 0:\j{Wait}=\false$, trivially, 
$0\!:\!s_i\ \cup\ 0\!:\!a_i\ \cup\ 1\!:\!s_{i+1} \models_{bg} \j{FLOW}$.

\item $\j{WAIT}$: Since $({\sf hevent}(e))^{a_i}=\true$, and $(\j{Wait})^{a_i}=\false$, it follows that $0\!:\!s_i\ \cup\ 0\!:\!a_i\ \cup\ 1\!:\!s_{i+1} \models_{bg} \j{WAIT} $.

\item $\j{GUARD}$:  From $s_i\models_{bg} X=r_{i}$, it follows that 
$0\!:\!s_i\models_{bg} 0\!:\!{\sf Guard}_e(X)$. 
Since $({\sf hevent}(e))^{a_i}=\true$, 
it follows that $0\!:\!s_i\ \cup\ 0\!:\!a_i\ \cup\ 1\!:\!s_{i+1} \models_{bg} \j{GUARD}$.

\item $\j{RESET}$: From $s_i\models_{bg} (\j{Mode},X) = (v_{i}, r_{i})$ and $s_{i+1}\models_{bg} (\j{Mode},X) = (v_{i+1}, r_{i+1})$, it follows that 
$0\!:\!s_i\ \cup\ 1\!:\!s_{i+1}\models_{bg} {\sf Reset}_e(0\!:\!X,1\!:\!X)$.
%,0\!:\!s_i \models_{bg} 0\!:\!\j{Mode}=v_{(i)}$.
%$0\!:\!s_i \models_{bg} 0\!:\!X=x_{(i)}$ and 
%$1\!:\!s_{i+1} \models_{bg} 1\!:\!X=x_{(i+1)}$. 
Since  $({\sf hevent}(e))^{a_i}=\true$, it follows that 
$0\!:\!s_i\ \cup\ 0\!:\!a_i\ \cup\ 1\!:\!s_{i+1} \models_{bg} \j{RESET} $.

\item $\j{MODE}$: Note that $s_i\models_{bg} (\j{Mode},X) = (v_{i}, r_{i})$ and $s_{i+1}\models_{bg} (\j{Mode},X) = (v_{i+1}, r_{i+1})$. It is immediate that $0\!:\!s_i \models_{bg} 0\!:\!\j{Mode}=v_{i}$ and $1\!:\!s_{i+1} \models_{bg} 1\!:\!\j{Mode}=v_{i+1}$. Since  $({\sf hevent}(e))^{a_i}=\true$, it follows that $0\!:\!s_i\ \cup\ 0\!:\!a_i\ \cup\ 1\!:\!s_{i+1} \models_{bg} \j{MODE}$.

\item $\j{DURATION}$: Since $(\j{Dur})^{a_i}=0$ and $({\sf hevent}(e))^{a_i}=\true$, it follows that $0\!:\!s_i\ \cup\ 0\!:\!a_i\ \cup\ 1\!:\!s_{i+1} \models_{bg} \j{DURATION} $.

\end{itemize}

\smallskip\noindent{\sl Case 2:} $\sigma_i\in \mathcal{R}_{\ge 0}$. By the construction of $p'$, $(\j{Dur})^{a_i}\mvis \sigma_i$, $(\j{Wait})^{a_i} \mvis \true$ and $({\sf hevent}(e))^{a_i} \mvis \false$ for every $e=(v,v')\in E$. 
It is easy to check that $\j{WAIT}$, $\j{GUARD}$, $\j{RESET}$, $\j{MODE}$, $\j{DURATION}$ are trivially satisfied by $0\!:\!s_i\cup 0\!:\!a_i\cup 1\!:\! s_{i+1}$. So, it is sufficient to consider only $\j{FLOW}$.

From the fact that \[(v_{i},  r_{i}) \xrightarrow{\sigma_i} (v_{i+1},  r_{i+1})\] is a transition of $T_H$ and that $\sigma_i\in \mathcal{R}_{\ge 0}$,
it follows from the definition of a hybrid transition system that
\begin{itemize}
\item[(a)] $v_{i} \mvis  v_{i+1}$, and

\item[(b)] there is a differentiable function $f:[0, \sigma_i]\rar \mathcal{R}^n$, with the first derivative $\dot{f} : [0, \sigma_i] \rar \mathcal{R}^n$ such
that: (1) $f(0) \mvis  r_{i}$ and $f(\sigma_i) \mvis  r_{i+1}$ and (2) for all reals $\epsilon \in (0, \sigma_i)$, both ${\sf Inv}_{v_{i}}(f(\epsilon))$
and ${\sf Flow}_{v_{i}}({\dot{f}}(\epsilon))$ are true.
\end{itemize}

\begin{itemize}
\item $\j{FLOW}$: 
\begin{itemize}
\item If $\sigma_i=0$, then $(\j{Dur})^{a_i}=0$. From (b), $ r_{i}= r_{i+1}=f(0)$. As a result $X^{s_i}=X^{s_{i+1}}$ and it follows that $0\!:\!s_i \cup 0\!:\!a_i \cup 1\!:\! s_{i+1}\models_{bg} (\ref{prop:ha-c-linear-1-4a})$.
\item If $\sigma_i>0$, then $(\j{Dur})^{a_i}\!>\! 0$. By Lemma~\ref{lemma:linear}, $f(t)=r_{i}+t*(r_{i+1}-r_i)/\sigma_i$ is a differentiable function that satisfies all the conditions in (b). As a result, ${\sf Flow}_{v_{i}}((r_{i+1}-r_{i})/\sigma_i)$ is true and thus $0\!:\!s_i \cup 0\!:\!a_i \cup 1\!:\! s_{i+1}\models_{bg} {\sf Flow}_{v_{i}}((1\!:\! r-0\!:\! r)/\j{Dur})$.
It follows that $0\!:\!s_i \cup 0\!:\!e_i \cup 1\!:\! s_{i+1}\models_{bg} (\ref{prop:ha-c-linear-1-4})$.
\end{itemize}
\end{itemize}
\end{proof}

\noindent{\bf Lemma~\ref{lem:ha2cplus} \optional{lem:ha2cplus}}\\ 
{\sl
$p'$ is a path in the transition system $D_H$.
}
\medskip

\begin{proof}
By Lemma~\ref{lemma:state-linear}, each $s_i$ is a state of $D_H$. By Lemma~\ref{lemma:edge-linear}, each
$\langle s_i,a_i,s_{i+1}\rangle$ is a transition of $D_H$. So $p'$ is a path in the transition system of $D_H$.
\end{proof}

%----------------------------------------------------------------------------------------------------
\subsubsection{Proof of Lemma \ref{lem:cplus2ha}} \label{sssec:lemma-cplus2ha}
%----------------------------------------------------------------------------------------------------

In the following two lemmas, $v_{i}, r_{i}$ are defined as in Lemma~\ref{lem:cplus2ha}. 

\begin{lemma}\label{lemma:state-linear-1}
%\begin{itemize}
%\item[(a)] 
For each $i\ge 0$, $(v_{i},  r_{i})$ is a state in $T_H$. 
%\item[(b)] $(v_{(0)},  r_{(0)})$ is an initial state in $T_H$.
%\end{itemize}
\end{lemma}

\begin{proof}
%\noindent (a) 
By definition, we are to show that ${\sf Inv}_{v_{i}}( r_{i})$ is true. Since each $s_i$ is a state in the transition system of $D_H$, by definition,
\beq 0\!:\!s_i\models_{bg} \sm[(D_H)_0;\emptyset].  \eeq{prop:ha-c-linear-1-11} Note that $\sm[(D_H)_0;\emptyset]$ is equivalent to the conjunction of the formula:
\beq 0\!:\!{\sf Inv}_v(X) \ar 0\!:\!\j{Mode}=v \eeq{prop:ha-c-linear-1-1} for each  $v\in V$.
Since $(\j{Mode})^{s_i}=v_{i}$, it follows that $s_i\models_{bg} {\sf Inv}_{v_{i}}(X)$. Since $X^{s_i}= r_{i}$, it follows that $ {\sf Inv}_{v_{i}}( r_{i})$ is true.
\end{proof}

\BOCCC
\medskip\noindent (b) We are to show that ${\sf Init}_{v_{(0}}( x_{(0})$ is true. 
Since we express ${\sf Init}_{v_{(0}}( x_{(0})$ as is in $D_H$. It is trivially entailed.
\EOCCC
%This is clear from the fact that $s_0\models_{bg} \j{INIT}$ and $( x)^{s_0}= x_{(0}$.

\begin{lemma}\label{lemma:edge-linear-1}
For each $i\ge 0$, $(v_{i},r_{i})\xrightarrow{\sigma_i}(v_{i+1},r_{i+1})$ is a transition in $T_H$.
\end{lemma}

\begin{proof}
From the fact that $(s_i,a_i,s_{i+1})$ is a transition of $D_H$, by definition we know that
\beq 0\!:\!s_i \cup 0\!:\!a_i \cup 1\!:\! s_{i+1}\models_{bg} \sm[(D_H)_1;\ 0\!:\!\sigma^{act}\cup 1\!:\!\sigma^{fl}].\eeq{prop:c-ha-linear-1-2}

%We check that $(D_H)_1$ is tight. By Lemma~\ref{lem:sm-comp}, (\ref{prop:ha-c-1-2}) is equivalent to
%\beq 0\!:\!s_i \cup 0\!:\!a_i \cup 1\!:\! s_{i+1}\models_{bg} \comp[(D_H)_1;0\!:\!\sigma^{act}\cup 1\!:\!\sigma^{sim}]\eeq{prop:ha-c-1-8}
%where $\comp[(D_H)_1;0\!:\!\sigma^{act}\cup 1\!:\!\sigma^{sim}]$ is equivalent to the conjunction of the following formulas:

Since $(D_H)_1$ is tight, $\sm[(D_H)_1; 0\!:\!\sigma^{act}\cup 1\!:\!\sigma^{fl}]$
is equivalent to
$\comp[(D_H)_1;0\!:\!\sigma^{act}\cup 1\!:\!\sigma^{fl}]$, which is equivalent to the conjunction of  $\j{FLOW}$, $\j{INV}$,  $\j{WAIT}$, $\j{GUARD}$, $\j{RESET}$, $\j{MODE}$, $\j{DURATION}$ (See the proof of Lemma~\ref{lemma:edge-linear} for the definitions of these formulas).

Consider two cases:

\smallskip\noindent{\sl Case 1:}  There exists an edge $e=(v,v')$ such that $({\sf hevent}(e))^{a_i} = \true$. Since $\j{Mode}^{s_i} = v_{i}$ and $\j{Mode}^{s_{i+1}} = v_{i+1}$, it follows that $(v,v')$ must be $(v_{i},v_{i+1})$. Since $({\sf hevent}(e))^{a_i}=\true$, it follows from the definition that $\sigma_i$ is ${\sf hevent}(e)$.

\begin{itemize}
\item Since $0\!:\!s_i \cup 0\!:\!a_i \cup 1\!:\! s_{i+1}\models_{bg} \j{GUARD}$ and $X^{s_i} = r_{i}$, it is immediate that ${\sf Guard}_e(r_{i})$ is true.
\item Since $0\!:\!s_i \cup 0\!:\!a_i \cup 1\!:\! s_{i+1}\models_{bg} \j{RESET}$, $X^{s_{i+1}} = r_{i+1}$ and $X^{s_i} = r_{i}$, it is immediate that ${\sf Reset}_e(r_{i},r_{i+1})$ is true.
\item By Lemma~\ref{lemma:state-linear-1},
$(v_{i},  r_{i})$ and $(v_{i+1},  r_{i+1})$ are states.
\end{itemize}
Consequently, we conclude that $(v_{i},r_{i})\xrightarrow{\sigma_i}(v_{i+1},r_{i+1})$ is a transition in $T_H$.

\smallskip\noindent{\sl Case 2:} $({\sf hevent}(e))^{a_i} = \false$ for all edges $e=(v,v')\in E$. By construction, $\sigma_i=(\j{Dur})^{a_i}$ where $(\j{Dur})^{a_i}\in \mathcal{R}_{\ge 0}$. By Lemma~\ref{lemma:state-linear-1},
$(v_{i},  r_{i})$ and $(v_{i+1},  r_{i+1})$ are states of $T_H$.
Since $0\!:\!s_i \cup 0\!:\!a_i \cup 1\!:\! s_{i+1}\models\j{MODE}$, it follows that $\j{Mode}^{s_i}=\j{Mode}^{s_{i+1}}$. As a result, $v_{i}=v_{i+1}$. We are to show that there is a
differentiable function $f : [0, \sigma_i] \rar \mathcal{R}^n$, with the first derivative $\dot{f} : [0, \sigma_i] \rar \mathcal{R}^n$ such
that: (i) $f(0) =  r_{i}$ and $f(\sigma_i) =  r_{i+1}$ and 
(ii) for all reals $\epsilon \in (0, \sigma_i)$, both ${\sf Inv}_{v_{i}}(f(\epsilon))$
and ${\sf Flow}_{v_{i}}({\dot{f}}(\epsilon))$ are true. We now check these conditions for two cases.
\begin{enumerate}
\item $\sigma_i=0$: Since 
$0\!:\! s_i\cup 0\!:\! a_i\cup 1\!:\! s_{i+1}\models_{bg}(\ref{prop:ha-c-linear-1-4a})$, it is clear that $r_{i+1}=r_{i}$. This satisfies condition (i) since $f(\sigma_i)=f(0)=r_{i+1}=r_{i}$. Condition (ii) is trivially satisfied since there is no $\epsilon\in (0,0)$.

\item $\sigma_i> 0$:
Define $f(t)= r_{i}+t*( r_{i+1}- r_{i})/\sigma_i$. We check that $f$ satisfies the above conditions:
\begin{itemize}
\item $f(t)$ is differentiable over $[0, \sigma_i]$.
\item It is clear that $f(0) =  r_{i}$ and $f(\sigma_i) =  r_{i+1}$.

\item We check that for any $\epsilon \in (0, \sigma)$, ${\sf Inv}_{v}(f(\epsilon))$ is true. From $0\!:\! s_i\cup 1\!:\! s_{i+1}\models_{bg}(\ref{prop:ha-c-linear-1-3})$, it follows that ${\sf Inv}_{v_{i}}(f(0))$ and ${\sf Inv}_{v_{i}}(f(\sigma_i))$ are true. 
Since the values of $X$ that makes ${\sf Inv}_{v_i}(X)$ form a convex region in ${\cal R}^n$ and $f(t)$ is a linear function, it follows that for $\epsilon \in (0, \sigma)$, ${\sf Inv}_{v_i}(f(\epsilon))$ is true.
% By Lemma~\ref{lemma:convex} it is clear that for $\epsilon \in (0, \sigma)$, ${\sf Inv}_{v_{i}}(f(\epsilon))$ is true.
\item We check that for any $\epsilon \in (0, \sigma)$, ${\sf Flow}_{v_{i}}({\dot{f}}(\epsilon))$ is true. From (\ref{prop:ha-c-linear-1-4}), it follows that ${\sf Flow}_{v_{i}}((f(\sigma_i)-f(0))/\sigma_i)$ is true. Since $f(t)$ is a linear function, it follows that for any $\epsilon \in (0, \sigma_i)$, ${\dot{f}}(\epsilon)=(f(\sigma_i)-f(0))/\sigma_i$. As a result, ${\sf Flow}_{v_{i}}({\dot{f}}(\epsilon))$ is true
\end{itemize}
\end{enumerate}
Consequently, we conclude that $(v_{i},r_{i})\xrightarrow{\sigma_i}(v_{i+1},r_{i+1})$ is a transition in $T_H$.
\end{proof}

\bigskip
\noindent{\bf Lemma~\ref{lem:cplus2ha} \optional{lem:cplus2ha}}\\ 
{\sl
$q'$ is a path in the transition system of $T_H$.
}
\medskip

\begin{proof}
By Lemma~\ref{lemma:state-linear-1}, each $(v_{i},  r_{i})$ is a state in $T_H$. By Lemma~\ref{lemma:edge-linear-1}, each
$(v_{i},  r_{i}) \xrightarrow{\sigma_i} (v_{i+1},  r_{i+1})$ is a transition in $T_H$. So $q'$ is a path in $T_H$. 
%If $s_0\models_{bg} {\sf Init}_{v_{0}}$ then by Lemma~\ref{lemma:state-linear-1} (b), $(v_0,  x_{(0})$ is a initial state in $T_H$.
\end{proof}

\section{Examples} 

\subsection{Water Tank Example}
\begin{center}
  \includegraphics[height=4cm]{water-tank-HA.png}
\end{center}

This example describes a water tank example with 2 tanks $X_1$ and $X_2$. Here $R_1$ and $R_2$ are constants that describe the lower bounds of the level of water in the respective tanks. $W_1$ and $W_2$ are constants that define the rate at which water is being added to the respective tanks and $V$ is the constant rate at which water is draining from the tanks. We assume that water is added only one tank at a time. 
%apply the restriction that you can add water to only 1 tank at a time.\\

Assuming $W_1=W_2=7.5$, $V=5$, $R_1=R_2=0$ and initially the level of water in the respective tanks are $X_1=0, X_2=8$, then the goal is to find a way to add water to each of the tanks with the passage of time. 

\subsubsection{Hybrid Automata Components} 

\begin{itemize}
\item Variables:
\begin{itemize}
\item $X_1,X_1',\dot{X_1}$
\item $X_2,X_2',\dot{X_2}$
\end{itemize}
\item States:
\begin{itemize}
\item $Q_1$ (mode=1)
\item $Q_2$ (mode=2)
\end{itemize}
\item Directed Graph: The graph is given above
\item Invariants: 
\begin{itemize} 
\item ${\sf Inv}_{Q_1}(X): X_2\geq {\rm R_2}$
\item ${\sf Inv}_{Q_2}(X): X_1\geq {\rm R_1}$
\end{itemize}
\item Flow: 
\begin{itemize} 
\item ${\sf Flow}_{Q_1}(X): \dot{X_1}={\rm W_1}-{\rm V}\ \land\ \dot{X_2}=-{\rm V}$.
\item ${\sf Flow}_{Q2}(X): \dot{X_1}=-{\rm V}\ \land\ \dot{X_2}={\rm W_2}-{\rm V}$.
\end{itemize}
\item Guard and Reset:
\begin{itemize} 
\item ${\sf Guard}_{(Q_1, Q_2)}(X): X_2\leq {\rm R_2}$.
\item ${\sf Guard}_{(Q_2, Q_1)}(X): X_1\leq {\rm R_1}$.
\item ${\sf Reset}_{(Q_1, Q_2)}(X,X'): X_1'=X_1\ \land\ X_2'=X_2$.
\item ${\sf Reset}_{(Q_2, Q_1)}(X,X'): X_1'=X_1\ \land\ X_2'=X_2$.
\end{itemize}
\end{itemize}

\subsubsection{In the Input Language of {\sc cplus2aspmt}}

\begin{lstlisting}[breaklines]
% File: water.cp 

:- constants
x1,x2    :: simpleFluent(real[0..30]);
mode     :: inertialFluent(real[1..2]);
e1,e2    :: exogenousAction;
wait     :: action;
duration :: exogenousAction(real[0..10]).

:- variables
X11,X21,X10,X20,T,X.

exogenous x1.
exogenous x2.

% Guard
nonexecutable e1 if -(x2<=r2).
nonexecutable e2 if -(x1<=r1).

% Reset
constraint (x1=X10 & x2=X20) after x1=X10 & x2=X20 & e1.
constraint (x1=X10 & x2=X20) after x1=X10 & x2=X20 & e2.

% Mode
nonexecutable e1 if -(mode=1).
nonexecutable e2 if -(mode=2).
e1 causes mode=2.
e2 causes mode=1.

% Duration
e1 causes duration=0.
e2 causes duration=0.

% Wait
default wait.
e1 causes ~wait.
e2 causes ~wait.

% Flow
constraint (x1=X11 & x2=X21 ->> (((X11-X10)//T)=w1-v & ((X21-X20)//T)=-v)) 
    after mode=1 & x1=X10 & x2=X20 & duration=T & wait & T>0.

constraint (x1=X10 & x2=X20) 
    after mode=1 & x1=X10 & x2=X20 & duration=0 & wait.

constraint (x1=X11 & x2=X21 ->> (((X11-X10)//T)=-v & ((X21-X20)//T)=w2-v)) 
    after mode=2 & x1=X10 & x2=X20 & duration=T & wait & T>0.

constraint (x1=X10 & x2=X20) 
    after mode=2 & x1=X10 & x2=X20 & duration=0 & wait.

% Invariant

constraint (mode=1 ->> (x2=X ->> X>=r2)).
constraint (mode=2 ->> (x1=X ->> X>=r1)).

:- query
label :: test;
maxstep :: 6;
0:mode=1;
0:x1 = 0;
0:x2 = 8;
2:mode=2;
4:mode=1;
6:mode=2.


\end{lstlisting}

\BOCCC
\subsubsection{{\sc aspmt} Input Generated by {\sc cplus2aspmt}}

\begin{lstlisting}[breaklines]

% Sort Declarations ---------------------------------------------------
:- sorts
step;astep;qSort;pstep;zero.


% Variable Declarations -----------------------------------------------
:- variables
ST::pstep;
AS::astep.



% Object Declarations -----------------------------------------------
:- objects
0 :: zero;
0..maxstep :: step;
1..maxstep :: pstep;
0..maxstep-1 :: astep;
query :: qSort.


% Constant Declarations -----------------------------------------------
:- constants
x1(step)::real[0..30];
x2(step)::real[0..30];
mode(step)::real[1..2];
e1(astep)::boolean;
e2(astep)::boolean;
wait(astep)::boolean;
duration(astep)::real[0..10];
qlabel(qSort) :: boolean.


{x1(0)=GVAR_real}.
{x2(0)=GVAR_real}.
{mode(ST)=GVAR_real} <- mode(ST-1)=GVAR_real.
{mode(0)=GVAR_real}.
{e1(ST-1)=GVAR_boolean}.
{e2(ST-1)=GVAR_boolean}.
{duration(ST-1)=GVAR_real}.

x1(0)=GVAR_real <- x1(0)=GVAR_real.
x1(ST)=GVAR_real <- x1(ST)=GVAR_real.

x2(0)=GVAR_real <- x2(0)=GVAR_real.
x2(ST)=GVAR_real <- x2(ST)=GVAR_real.

% Guard

false <- e1(ST-1)=true & not (GVAR_real<=0) & x2(ST-1)=GVAR_real.
false <- e2(ST-1)=true & not (GVAR_real<=0) & x1(ST-1)=GVAR_real.

% Reset

false <- not (x1(ST)=X10 & x2(ST)=X20) & x1(ST-1)=X10 & x2(ST-1)=X20 & e1(ST-1)=true.
false <- not (x1(ST)=X10 & x2(ST)=X20) & x1(ST-1)=X10 & x2(ST-1)=X20 & e2(ST-1)=true.

%Mode

false <- e1(ST-1)=true & not (mode(ST-1)=1).
false <- e2(ST-1)=true & not (mode(ST-1)=2).
mode(ST)=2 <- e1(ST-1)=true.
mode(ST)=1 <- e2(ST-1)=true.

%Duration

duration(ST-1)=0 <- e1(ST-1)=true.
duration(ST-1)=0 <- e2(ST-1)=true.

%Wait

wait(ST-1)=true <- wait(ST-1)=true.
wait(ST-1)=false <- e1(ST-1)=true.
wait(ST-1)=false <- e2(ST-1)=true.

%Flow

false <- not (x1(ST)=X11 & x2(ST)=X21 -> (((X11-X10)/T)==7.5-5 & ((X21-X20)/T)==-5)) & mode(ST-1)=1 & x1(ST-1)=X10 & x2(ST-1)=X20 & duration(ST-1)=T & wait(ST-1)=true & T>0.

false <- not (x1(ST)=X10 & x2(ST)=X20) & mode(ST-1)=1 & x1(ST-1)=X10 & x2(ST-1)=X20 & duration(ST-1)=0 & wait(ST-1)=true.

false <- not (x1(ST)=X11 & x2(ST)=X21 -> (((X11-X10)/T)==-5 & ((X21-X20)/T)==7.5-5)) & mode(ST-1)=2 & x1(ST-1)=X10 & x2(ST-1)=X20 & duration(ST-1)=T & wait(ST-1)=true & T>0.

false <- not (x1(ST)=X10 & x2(ST)=X20) & mode(ST-1)=2 & x1(ST-1)=X10 & x2(ST-1)=X20 & duration(ST-1)=0 & wait(ST-1)=true.

%Invariant

false <- not (mode(0)=1 -> (x2(0)=X -> X>=0)).
false <- not (mode(ST)=1 -> (x2(ST)=X -> X>=0)).

false <- not (mode(0)=2 -> (x1(0)=X -> X>=0)).
false <- not (mode(ST)=2 -> (x1(ST)=X -> X>=0)).


% Query: init
% Maxstep: ..6
<- not mode(0)=1 & qlabel(init).
<- not x1(0)=0 & qlabel(init).
<- not x2(0)=8 & qlabel(init).
<- not mode(2)=2 & qlabel(init) & maxstep == 0.
<- not mode(2)=2 & qlabel(init).
<- not mode(4)=1 & qlabel(init) & maxstep == 0.
<- not mode(4)=1 & qlabel(init).
<- not mode(6)=2 & qlabel(init) & maxstep == 0.
<- not mode(6)=2 & qlabel(init).


% -------------------------------------------------------------------------
% epilogue ----------------------------------------------------------------
% -------------------------------------------------------------------------

qlabel(query).

% -------------------------------------------------------------------------

#hide qlabel/1.
\end{lstlisting}

\subsubsection{${\tt dReal}$ Input Generated by {\sc cplus2aspmt}}

\begin{lstlisting}[breaklines]
(set-logic QF_NRA_ODE)
(declare-const true_a Bool)
(declare-const false_a Bool)
(declare-const duration_0_ Real)
(declare-const duration_1_ Real)
(declare-const duration_2_ Real)
(declare-const duration_3_ Real)
(declare-const duration_4_ Real)
(declare-const duration_5_ Real)
(declare-const e1_0_ Bool)
(declare-const e1_1_ Bool)
(declare-const e1_2_ Bool)
(declare-const e1_3_ Bool)
(declare-const e1_4_ Bool)
(declare-const e1_5_ Bool)
(declare-const e2_0_ Bool)
(declare-const e2_1_ Bool)
(declare-const e2_2_ Bool)
(declare-const e2_3_ Bool)
(declare-const e2_4_ Bool)
(declare-const e2_5_ Bool)
(declare-const mode_0_ Real)
(declare-const mode_1_ Real)
(declare-const mode_2_ Real)
(declare-const mode_3_ Real)
(declare-const mode_4_ Real)
(declare-const mode_5_ Real)
(declare-const mode_6_ Real)
(declare-const qlabel_init_ Bool)
(declare-const wait_0_ Bool)
(declare-const wait_1_ Bool)
(declare-const wait_2_ Bool)
(declare-const wait_3_ Bool)
(declare-const wait_4_ Bool)
(declare-const wait_5_ Bool)
(declare-const x1_0_ Real)
(declare-const x1_1_ Real)
(declare-const x1_2_ Real)
(declare-const x1_3_ Real)
(declare-const x1_4_ Real)
(declare-const x1_5_ Real)
(declare-const x1_6_ Real)
(declare-const x2_0_ Real)
(declare-const x2_1_ Real)
(declare-const x2_2_ Real)
(declare-const x2_3_ Real)
(declare-const x2_4_ Real)
(declare-const x2_5_ Real)
(declare-const x2_6_ Real)
(assert true_a)
(assert (not false_a))
(assert (>= duration_0_ 0))
(assert (<= duration_0_ 10))
(assert (>= duration_1_ 0))
(assert (<= duration_1_ 10))
(assert (>= duration_2_ 0))
(assert (<= duration_2_ 10))
(assert (>= duration_3_ 0))
(assert (<= duration_3_ 10))
(assert (>= duration_4_ 0))
(assert (<= duration_4_ 10))
(assert (>= duration_5_ 0))
(assert (<= duration_5_ 10))
(assert (>= mode_0_ 1))
(assert (<= mode_0_ 2))
(assert (>= mode_1_ 1))
(assert (<= mode_1_ 2))
(assert (>= mode_2_ 1))
(assert (<= mode_2_ 2))
(assert (>= mode_3_ 1))
(assert (<= mode_3_ 2))
(assert (>= mode_4_ 1))
(assert (<= mode_4_ 2))
(assert (>= mode_5_ 1))
(assert (<= mode_5_ 2))
(assert (>= mode_6_ 1))
(assert (<= mode_6_ 2))
(assert (>= x1_0_ 0))
(assert (<= x1_0_ 30))
(assert (>= x1_1_ 0))
(assert (<= x1_1_ 30))
(assert (>= x1_2_ 0))
(assert (<= x1_2_ 30))
(assert (>= x1_3_ 0))
(assert (<= x1_3_ 30))
(assert (>= x1_4_ 0))
(assert (<= x1_4_ 30))
(assert (>= x1_5_ 0))
(assert (<= x1_5_ 30))
(assert (>= x1_6_ 0))
(assert (<= x1_6_ 30))
(assert (>= x2_0_ 0))
(assert (<= x2_0_ 30))
(assert (>= x2_1_ 0))
(assert (<= x2_1_ 30))
(assert (>= x2_2_ 0))
(assert (<= x2_2_ 30))
(assert (>= x2_3_ 0))
(assert (<= x2_3_ 30))
(assert (>= x2_4_ 0))
(assert (<= x2_4_ 30))
(assert (>= x2_5_ 0))
(assert (<= x2_5_ 30))
(assert (>= x2_6_ 0))
(assert (<= x2_6_ 30))
(assert (or (or (and (= wait_5_ true) (= wait_5_ true)) (and (= e2_5_ true) (= wait_5_ false))) (and (= e1_5_ true) (= wait_5_ false))))
(assert (=> (= wait_5_ true) (= wait_5_ true)))
(assert (=> (= e2_5_ true) (= wait_5_ false)))
(assert (=> (= e1_5_ true) (= wait_5_ false)))
(assert (or (or (and (= wait_4_ true) (= wait_4_ true)) (and (= e2_4_ true) (= wait_4_ false))) (and (= e1_4_ true) (= wait_4_ false))))
(assert (=> (= wait_4_ true) (= wait_4_ true)))
(assert (=> (= e2_4_ true) (= wait_4_ false)))
(assert (=> (= e1_4_ true) (= wait_4_ false)))
(assert (or (or (and (= wait_3_ true) (= wait_3_ true)) (and (= e2_3_ true) (= wait_3_ false))) (and (= e1_3_ true) (= wait_3_ false))))
(assert (=> (= wait_3_ true) (= wait_3_ true)))
(assert (=> (= e2_3_ true) (= wait_3_ false)))
(assert (=> (= e1_3_ true) (= wait_3_ false)))
(assert (or (or (and (= wait_2_ true) (= wait_2_ true)) (and (= e2_2_ true) (= wait_2_ false))) (and (= e1_2_ true) (= wait_2_ false))))
(assert (=> (= wait_2_ true) (= wait_2_ true)))
(assert (=> (= e2_2_ true) (= wait_2_ false)))
(assert (=> (= e1_2_ true) (= wait_2_ false)))
(assert (or (or (and (= wait_1_ true) (= wait_1_ true)) (and (= e2_1_ true) (= wait_1_ false))) (and (= e1_1_ true) (= wait_1_ false))))
(assert (=> (= wait_1_ true) (= wait_1_ true)))
(assert (=> (= e2_1_ true) (= wait_1_ false)))
(assert (=> (= e1_1_ true) (= wait_1_ false)))
(assert (or (or (and (= wait_0_ true) (= wait_0_ true)) (and (= e2_0_ true) (= wait_0_ false))) (and (= e1_0_ true) (= wait_0_ false))))
(assert (=> (= wait_0_ true) (= wait_0_ true)))
(assert (=> (= e2_0_ true) (= wait_0_ false)))
(assert (=> (= e1_0_ true) (= wait_0_ false)))
(assert (= qlabel_init_ true))
(assert (= qlabel_init_ true))
(assert (or (or (= mode_6_ mode_5_) (and (= e2_5_ true) (= mode_6_ 1))) (and (= e1_5_ true) (= mode_6_ 2))))
(assert (=> (= e2_5_ true) (= mode_6_ 1)))
(assert (=> (= e1_5_ true) (= mode_6_ 2)))
(assert (or (or (= mode_5_ mode_4_) (and (= e2_4_ true) (= mode_5_ 1))) (and (= e1_4_ true) (= mode_5_ 2))))
(assert (=> (= e2_4_ true) (= mode_5_ 1)))
(assert (=> (= e1_4_ true) (= mode_5_ 2)))
(assert (or (or (= mode_4_ mode_3_) (and (= e2_3_ true) (= mode_4_ 1))) (and (= e1_3_ true) (= mode_4_ 2))))
(assert (=> (= e2_3_ true) (= mode_4_ 1)))
(assert (=> (= e1_3_ true) (= mode_4_ 2)))
(assert (or (or (= mode_3_ mode_2_) (and (= e2_2_ true) (= mode_3_ 1))) (and (= e1_2_ true) (= mode_3_ 2))))
(assert (=> (= e2_2_ true) (= mode_3_ 1)))
(assert (=> (= e1_2_ true) (= mode_3_ 2)))
(assert (or (or (= mode_2_ mode_1_) (and (= e2_1_ true) (= mode_2_ 1))) (and (= e1_1_ true) (= mode_2_ 2))))
(assert (=> (= e2_1_ true) (= mode_2_ 1)))
(assert (=> (= e1_1_ true) (= mode_2_ 2)))
(assert (or (or (= mode_1_ mode_0_) (and (= e2_0_ true) (= mode_1_ 1))) (and (= e1_0_ true) (= mode_1_ 2))))
(assert (=> (= e2_0_ true) (= mode_1_ 1)))
(assert (=> (= e1_0_ true) (= mode_1_ 2)))
(assert (=> (= e2_5_ true) (= duration_5_ 0)))
(assert (=> (= e1_5_ true) (= duration_5_ 0)))
(assert (=> (= e2_4_ true) (= duration_4_ 0)))
(assert (=> (= e1_4_ true) (= duration_4_ 0)))
(assert (=> (= e2_3_ true) (= duration_3_ 0)))
(assert (=> (= e1_3_ true) (= duration_3_ 0)))
(assert (=> (= e2_2_ true) (= duration_2_ 0)))
(assert (=> (= e1_2_ true) (= duration_2_ 0)))
(assert (=> (= e2_1_ true) (= duration_1_ 0)))
(assert (=> (= e1_1_ true) (= duration_1_ 0)))
(assert true_a)
(assert (=> (= e2_0_ true) (= duration_0_ 0)))
(assert (=> (= e1_0_ true) (= duration_0_ 0)))
(assert (not (and (= e1_0_ true) (not (<= x2_0_ 0)))))
(assert (not (and (= e1_1_ true) (not (<= x2_1_ 0)))))
(assert (not (and (= e1_2_ true) (not (<= x2_2_ 0)))))
(assert (not (and (= e1_3_ true) (not (<= x2_3_ 0)))))
(assert (not (and (= e1_4_ true) (not (<= x2_4_ 0)))))
(assert (not (and (= e1_5_ true) (not (<= x2_5_ 0)))))
(assert (not (and (= e2_0_ true) (not (<= x1_0_ 0)))))
(assert (not (and (= e2_1_ true) (not (<= x1_1_ 0)))))
(assert (not (and (= e2_2_ true) (not (<= x1_2_ 0)))))
(assert (not (and (= e2_3_ true) (not (<= x1_3_ 0)))))
(assert (not (and (= e2_4_ true) (not (<= x1_4_ 0)))))
(assert (not (and (= e2_5_ true) (not (<= x1_5_ 0)))))
(assert (not (and (not (= x1_1_ x1_0_)) (= e1_0_ true))))
(assert (not (and (not (= x1_2_ x1_1_)) (= e1_1_ true))))
(assert (not (and (not (= x1_3_ x1_2_)) (= e1_2_ true))))
(assert (not (and (not (= x1_4_ x1_3_)) (= e1_3_ true))))
(assert (not (and (not (= x1_5_ x1_4_)) (= e1_4_ true))))
(assert (not (and (not (= x1_6_ x1_5_)) (= e1_5_ true))))
(assert (not (and (not (= x2_1_ x2_0_)) (= e1_0_ true))))
(assert (not (and (not (= x2_2_ x2_1_)) (= e1_1_ true))))
(assert (not (and (not (= x2_3_ x2_2_)) (= e1_2_ true))))
(assert (not (and (not (= x2_4_ x2_3_)) (= e1_3_ true))))
(assert (not (and (not (= x2_5_ x2_4_)) (= e1_4_ true))))
(assert (not (and (not (= x2_6_ x2_5_)) (= e1_5_ true))))
(assert (not (and (not (= x1_1_ x1_0_)) (= e2_0_ true))))
(assert (not (and (not (= x1_2_ x1_1_)) (= e2_1_ true))))
(assert (not (and (not (= x1_3_ x1_2_)) (= e2_2_ true))))
(assert (not (and (not (= x1_4_ x1_3_)) (= e2_3_ true))))
(assert (not (and (not (= x1_5_ x1_4_)) (= e2_4_ true))))
(assert (not (and (not (= x1_6_ x1_5_)) (= e2_5_ true))))
(assert (not (and (not (= x2_1_ x2_0_)) (= e2_0_ true))))
(assert (not (and (not (= x2_2_ x2_1_)) (= e2_1_ true))))
(assert (not (and (not (= x2_3_ x2_2_)) (= e2_2_ true))))
(assert (not (and (not (= x2_4_ x2_3_)) (= e2_3_ true))))
(assert (not (and (not (= x2_5_ x2_4_)) (= e2_4_ true))))
(assert (not (and (not (= x2_6_ x2_5_)) (= e2_5_ true))))
(assert (not (and (= e1_0_ true) (not (= mode_0_ 1)))))
(assert (not (and (= e1_1_ true) (not (= mode_1_ 1)))))
(assert (not (and (= e1_2_ true) (not (= mode_2_ 1)))))
(assert (not (and (= e1_3_ true) (not (= mode_3_ 1)))))
(assert (not (and (= e1_4_ true) (not (= mode_4_ 1)))))
(assert (not (and (= e1_5_ true) (not (= mode_5_ 1)))))
(assert (not (and (= e2_0_ true) (not (= mode_0_ 2)))))
(assert (not (and (= e2_1_ true) (not (= mode_1_ 2)))))
(assert (not (and (= e2_2_ true) (not (= mode_2_ 2)))))
(assert (not (and (= e2_3_ true) (not (= mode_3_ 2)))))
(assert (not (and (= e2_4_ true) (not (= mode_4_ 2)))))
(assert (not (and (= e2_5_ true) (not (= mode_5_ 2)))))
(assert (not (and (and (and (not (= (/ (- x1_1_ x1_0_) duration_0_) (- (/ 75 10) 5))) (> duration_0_ 0)) (= wait_0_ true)) (= mode_0_ 1))))
(assert (not (and (and (and (not (= (/ (- x1_2_ x1_1_) duration_1_) (- (/ 75 10) 5))) (> duration_1_ 0)) (= wait_1_ true)) (= mode_1_ 1))))
(assert (not (and (and (and (not (= (/ (- x1_3_ x1_2_) duration_2_) (- (/ 75 10) 5))) (> duration_2_ 0)) (= wait_2_ true)) (= mode_2_ 1))))
(assert (not (and (and (and (not (= (/ (- x1_4_ x1_3_) duration_3_) (- (/ 75 10) 5))) (> duration_3_ 0)) (= wait_3_ true)) (= mode_3_ 1))))
(assert (not (and (and (and (not (= (/ (- x1_5_ x1_4_) duration_4_) (- (/ 75 10) 5))) (> duration_4_ 0)) (= wait_4_ true)) (= mode_4_ 1))))
(assert (not (and (and (and (not (= (/ (- x1_6_ x1_5_) duration_5_) (- (/ 75 10) 5))) (> duration_5_ 0)) (= wait_5_ true)) (= mode_5_ 1))))
(assert (not (and (and (and (not (= (/ (- x2_1_ x2_0_) duration_0_) -5)) (> duration_0_ 0)) (= wait_0_ true)) (= mode_0_ 1))))
(assert (not (and (and (and (not (= (/ (- x2_2_ x2_1_) duration_1_) -5)) (> duration_1_ 0)) (= wait_1_ true)) (= mode_1_ 1))))
(assert (not (and (and (and (not (= (/ (- x2_3_ x2_2_) duration_2_) -5)) (> duration_2_ 0)) (= wait_2_ true)) (= mode_2_ 1))))
(assert (not (and (and (and (not (= (/ (- x2_4_ x2_3_) duration_3_) -5)) (> duration_3_ 0)) (= wait_3_ true)) (= mode_3_ 1))))
(assert (not (and (and (and (not (= (/ (- x2_5_ x2_4_) duration_4_) -5)) (> duration_4_ 0)) (= wait_4_ true)) (= mode_4_ 1))))
(assert (not (and (and (and (not (= (/ (- x2_6_ x2_5_) duration_5_) -5)) (> duration_5_ 0)) (= wait_5_ true)) (= mode_5_ 1))))
(assert (not (and (and (and (not (= x1_1_ x1_0_)) (= wait_0_ true)) (= mode_0_ 1)) (= duration_0_ 0))))
(assert (not (and (and (and (not (= x1_2_ x1_1_)) (= wait_1_ true)) (= mode_1_ 1)) (= duration_1_ 0))))
(assert (not (and (and (and (not (= x1_3_ x1_2_)) (= wait_2_ true)) (= mode_2_ 1)) (= duration_2_ 0))))
(assert (not (and (and (and (not (= x1_4_ x1_3_)) (= wait_3_ true)) (= mode_3_ 1)) (= duration_3_ 0))))
(assert (not (and (and (and (not (= x1_5_ x1_4_)) (= wait_4_ true)) (= mode_4_ 1)) (= duration_4_ 0))))
(assert (not (and (and (and (not (= x1_6_ x1_5_)) (= wait_5_ true)) (= mode_5_ 1)) (= duration_5_ 0))))
(assert (not (and (and (and (not (= x2_1_ x2_0_)) (= wait_0_ true)) (= mode_0_ 1)) (= duration_0_ 0))))
(assert (not (and (and (and (not (= x2_2_ x2_1_)) (= wait_1_ true)) (= mode_1_ 1)) (= duration_1_ 0))))
(assert (not (and (and (and (not (= x2_3_ x2_2_)) (= wait_2_ true)) (= mode_2_ 1)) (= duration_2_ 0))))
(assert (not (and (and (and (not (= x2_4_ x2_3_)) (= wait_3_ true)) (= mode_3_ 1)) (= duration_3_ 0))))
(assert (not (and (and (and (not (= x2_5_ x2_4_)) (= wait_4_ true)) (= mode_4_ 1)) (= duration_4_ 0))))
(assert (not (and (and (and (not (= x2_6_ x2_5_)) (= wait_5_ true)) (= mode_5_ 1)) (= duration_5_ 0))))
(assert (not (and (and (and (not (= (/ (- x1_1_ x1_0_) duration_0_) -5)) (> duration_0_ 0)) (= wait_0_ true)) (= mode_0_ 2))))
(assert (not (and (and (and (not (= (/ (- x1_2_ x1_1_) duration_1_) -5)) (> duration_1_ 0)) (= wait_1_ true)) (= mode_1_ 2))))
(assert (not (and (and (and (not (= (/ (- x1_3_ x1_2_) duration_2_) -5)) (> duration_2_ 0)) (= wait_2_ true)) (= mode_2_ 2))))
(assert (not (and (and (and (not (= (/ (- x1_4_ x1_3_) duration_3_) -5)) (> duration_3_ 0)) (= wait_3_ true)) (= mode_3_ 2))))
(assert (not (and (and (and (not (= (/ (- x1_5_ x1_4_) duration_4_) -5)) (> duration_4_ 0)) (= wait_4_ true)) (= mode_4_ 2))))
(assert (not (and (and (and (not (= (/ (- x1_6_ x1_5_) duration_5_) -5)) (> duration_5_ 0)) (= wait_5_ true)) (= mode_5_ 2))))
(assert (not (and (and (and (not (= (/ (- x2_1_ x2_0_) duration_0_) (- (/ 75 10) 5))) (> duration_0_ 0)) (= wait_0_ true)) (= mode_0_ 2))))
(assert (not (and (and (and (not (= (/ (- x2_2_ x2_1_) duration_1_) (- (/ 75 10) 5))) (> duration_1_ 0)) (= wait_1_ true)) (= mode_1_ 2))))
(assert (not (and (and (and (not (= (/ (- x2_3_ x2_2_) duration_2_) (- (/ 75 10) 5))) (> duration_2_ 0)) (= wait_2_ true)) (= mode_2_ 2))))
(assert (not (and (and (and (not (= (/ (- x2_4_ x2_3_) duration_3_) (- (/ 75 10) 5))) (> duration_3_ 0)) (= wait_3_ true)) (= mode_3_ 2))))
(assert (not (and (and (and (not (= (/ (- x2_5_ x2_4_) duration_4_) (- (/ 75 10) 5))) (> duration_4_ 0)) (= wait_4_ true)) (= mode_4_ 2))))
(assert (not (and (and (and (not (= (/ (- x2_6_ x2_5_) duration_5_) (- (/ 75 10) 5))) (> duration_5_ 0)) (= wait_5_ true)) (= mode_5_ 2))))
(assert (not (and (and (and (not (= x1_1_ x1_0_)) (= wait_0_ true)) (= mode_0_ 2)) (= duration_0_ 0))))
(assert (not (and (and (and (not (= x1_2_ x1_1_)) (= wait_1_ true)) (= mode_1_ 2)) (= duration_1_ 0))))
(assert (not (and (and (and (not (= x1_3_ x1_2_)) (= wait_2_ true)) (= mode_2_ 2)) (= duration_2_ 0))))
(assert (not (and (and (and (not (= x1_4_ x1_3_)) (= wait_3_ true)) (= mode_3_ 2)) (= duration_3_ 0))))
(assert (not (and (and (and (not (= x1_5_ x1_4_)) (= wait_4_ true)) (= mode_4_ 2)) (= duration_4_ 0))))
(assert (not (and (and (and (not (= x1_6_ x1_5_)) (= wait_5_ true)) (= mode_5_ 2)) (= duration_5_ 0))))
(assert (not (and (and (and (not (= x2_1_ x2_0_)) (= wait_0_ true)) (= mode_0_ 2)) (= duration_0_ 0))))
(assert (not (and (and (and (not (= x2_2_ x2_1_)) (= wait_1_ true)) (= mode_1_ 2)) (= duration_1_ 0))))
(assert (not (and (and (and (not (= x2_3_ x2_2_)) (= wait_2_ true)) (= mode_2_ 2)) (= duration_2_ 0))))
(assert (not (and (and (and (not (= x2_4_ x2_3_)) (= wait_3_ true)) (= mode_3_ 2)) (= duration_3_ 0))))
(assert (not (and (and (and (not (= x2_5_ x2_4_)) (= wait_4_ true)) (= mode_4_ 2)) (= duration_4_ 0))))
(assert (not (and (and (and (not (= x2_6_ x2_5_)) (= wait_5_ true)) (= mode_5_ 2)) (= duration_5_ 0))))
(assert (not (and (= mode_0_ 1) (not (>= x2_0_ 0)))))
(assert (not (and (= mode_6_ 1) (not (>= x2_6_ 0)))))
(assert (not (and (= mode_5_ 1) (not (>= x2_5_ 0)))))
(assert (not (and (= mode_4_ 1) (not (>= x2_4_ 0)))))
(assert (not (and (= mode_3_ 1) (not (>= x2_3_ 0)))))
(assert (not (and (= mode_2_ 1) (not (>= x2_2_ 0)))))
(assert (not (and (= mode_1_ 1) (not (>= x2_1_ 0)))))
(assert (not (and (= mode_0_ 2) (not (>= x1_0_ 0)))))
(assert (not (and (= mode_6_ 2) (not (>= x1_6_ 0)))))
(assert (not (and (= mode_5_ 2) (not (>= x1_5_ 0)))))
(assert (not (and (= mode_4_ 2) (not (>= x1_4_ 0)))))
(assert (not (and (= mode_3_ 2) (not (>= x1_3_ 0)))))
(assert (not (and (= mode_2_ 2) (not (>= x1_2_ 0)))))
(assert (not (and (= mode_1_ 2) (not (>= x1_1_ 0)))))
(assert (not (not (= mode_0_ 1))))
(assert (not (not (= x1_0_ 0))))
(assert (not (not (= x2_0_ 8))))
(assert (not (not (= mode_2_ 2))))
(assert (not (not (= mode_4_ 1))))
(assert (not (not (= mode_6_ 2))))

(check-sat)
(exit)
\end{lstlisting}
\EOCCC

\subsubsection{Output}

\begin{center}
  \includegraphics[height=7cm]{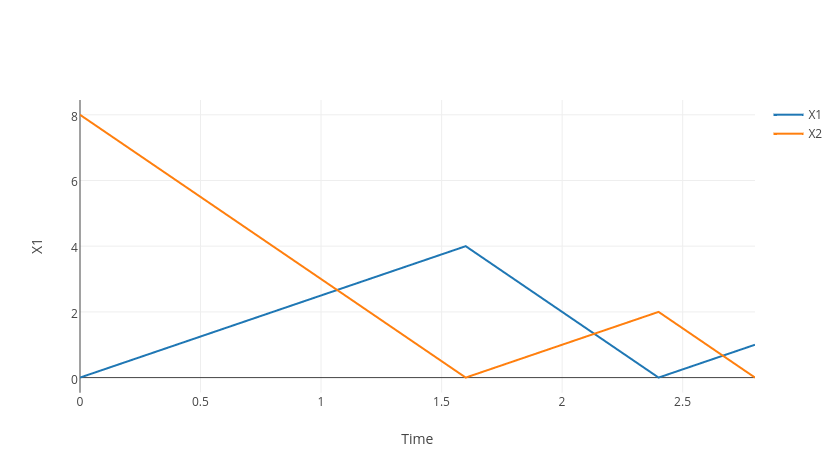}
\end{center}
\begin{lstlisting}[breaklines]
Command: cplus2aspmt water.cp -c maxstep=6 -c query=test -c w1=7.5 -c w2=7.5 -c v=5 -c r1=0 =c r2=0

Solution:
duration_0_ : [ ENTIRE ] = [1.6, 1.6]
duration_1_ : [ ENTIRE ] = [0, 0]
duration_2_ : [ ENTIRE ] = [0.7999999999999998, 0.8000000000000003]
duration_3_ : [ ENTIRE ] = [0, 0]
duration_4_ : [ ENTIRE ] = [0.3999999999999999, 0.4000000000000002]
duration_5_ : [ ENTIRE ] = [0, 0]
mode_0_ : [ ENTIRE ] = [1, 1]
mode_1_ : [ ENTIRE ] = [1, 1]
mode_2_ : [ ENTIRE ] = [2, 2]
mode_3_ : [ ENTIRE ] = [2, 2]
mode_4_ : [ ENTIRE ] = [1, 1]
mode_5_ : [ ENTIRE ] = [1, 1]
mode_6_ : [ ENTIRE ] = [2, 2]
x1_0_ : [ ENTIRE ] = [0, 0]
x1_1_ : [ ENTIRE ] = [4, 4.000000000000001]
x1_2_ : [ ENTIRE ] = [4, 4.000000000000001]
x1_3_ : [ ENTIRE ] = [0, 0]
x1_4_ : [ ENTIRE ] = [0, 0]
x1_5_ : [ ENTIRE ] = [0.9999999999999998, 1.000000000000001]
x1_6_ : [ ENTIRE ] = [0.9999999999999998, 1.000000000000001]
x2_0_ : [ ENTIRE ] = [8, 8]
x2_1_ : [ ENTIRE ] = [0, 0]
x2_2_ : [ ENTIRE ] = [0, 0]
x2_3_ : [ ENTIRE ] = [2, 2.000000000000001]
x2_4_ : [ ENTIRE ] = [2, 2.000000000000001]
x2_5_ : [ ENTIRE ] = [0, 0]
x2_6_ : [ ENTIRE ] = [0, 0]
true_a : Bool = true
false_a : Bool = false
e1_0_ : Bool = false
e1_1_ : Bool = true
e1_2_ : Bool = false
e1_3_ : Bool = false
e1_4_ : Bool = false
e1_5_ : Bool = true
e2_0_ : Bool = false
e2_1_ : Bool = false
e2_2_ : Bool = false
e2_3_ : Bool = true
e2_4_ : Bool = false
e2_5_ : Bool = false
qlabel_init_ : Bool = true
wait_0_ : Bool = true
wait_1_ : Bool = false
wait_2_ : Bool = true
wait_3_ : Bool = false
wait_4_ : Bool = true
wait_5_ : Bool = false
delta-sat with delta = 0.00100000000000000
Total time in milliseconds: 4328
\end{lstlisting}

\BOCCC
\subsection{2 Ball Example} \label{ssec:2ball}
\begin{center}
  \includegraphics[height=3cm]{2ball-ha.png}
\end{center}
This example involves 2 balls with the same elasticity falling to the ground. Ball 1 starts at 2 units from the ground while ball 2 starts at 3 units from the ground. Our aim is to formalize this problem and observe the transition system of these 2 balls with time.

\subsubsection{Hybrid Automata Components} 

\begin{itemize}
\item Variables:
\begin{itemize}
\item \(H_1,H_1',\dot{H_1}\)
\item \(H_2,H_2',\dot{H_2}\)
\item \(V_1,V_1',\dot{V_1}\)
\item \(V_2,V_2',\dot{V_2}\)
\end{itemize}
\item States:
\begin{itemize}
\item \(s_0\)  (Corresponds to $mode=1$)
\end{itemize}
\item Hevents:
\begin{itemize}
\item \(hitGround\_b1\)
\item \(hitGround\_b2\)
\end{itemize}
\item Directed Graph: The graph is given above
\item Invariants: 
\begin{itemize} 
\item \({\sf Inv}_{s_0}(H_1,H_2): H_1\geq 0 \wedge H_2\geq 0\)
\end{itemize}
\item Flow: 
\begin{itemize} 
\item \({\sf Flow}_{s_0}(H_1,H_2,V_1,V_2): \dot{H_1}=V_1\ \wedge\ \dot{H_2}=V_2\ \wedge\ \dot{V_1}=-g\ \wedge \ \dot{V_2}=-g\).
\end{itemize}
\item Jump:
\begin{itemize} 
\item ${\sf Guard}_{(s_0,s_0)}(H_1,H_2,V_1,V_2): H_1=0\ \wedge\ H_2=0$
\item ${\sf Reset}_{(s_0,s_0)}(H_1,H_2,V_1,V_2): H_1'=0\ \wedge\ H_2'=H_2\ \wedge V_1'=-0.8\times V_1\ \wedge\ V_2'=V_2$
\item ${\sf Reset}_{(s_0,s_0)}(H_1,H_2,V_1,V_2): H_1'=H_1\ \wedge\ H_2'=0\ \wedge V_1'=V_1\ \wedge\ V_2'=-0.9\times V_2$
\end{itemize}
\end{itemize}

\subsubsection{C+ Input}

\begin{lstlisting}[breaklines]
:- sorts
ball.

:- objects
b1,b2::ball.

:- constants
mode :: inertialFluent(integer);
height(ball) :: simpleFluent(real[0..50]);
velocity(ball) :: simpleFluent(real[-30..30]);
hitGround(ball) :: exogenousAction;
wait :: action;
duration :: exogenousAction(real[0..100]).

:- variables
E :: events;
B :: ball;
B1 :: ball;
V,V_1,V0,D,H,H_1,H0,T,R.

%System Rules
%All actions have 0 duration
caused duration=0 if hitGround(B).

default wait.
caused ~wait if hitGround(B).
exogenous height(B).
exogenous velocity(B).

% Flow Translation
constraint height(B)=H_1 after velocity(B)=V0 & height(B) = H0 &
H_1=H0+V0*T+0.5*-g*T*T & mode=1 & duration = T & wait.
constraint velocity(B)=V_1 after velocity(B) = V0 & V_1=V0 + -g*T & mode=1 & duration = T & wait.

% Guard Translation
nonexecutable hitGround(B) if -(height(B)=0).

% Hevent Restrictions
nonexecutable hitGround(B) if -(mode=1).

% Invariant Translation
constraint (mode=1 ->> -(height(B)=H ->> H>=0)).

% Reset Translation
constraint (height(b1)=H_10 & velocity(b1)=V_10 & height(b2)=0 & velocity(b2)=-0.9*V_20) after height(b2)=H_20 & velocity(b2)=V_20 & height(b1)=H_10 & velocity(b1)=V_10 & hitGround(b2).
constraint (height(b2)=H_20 & velocity(b2)=V_20 & height(b1)=0 & velocity(b1)=-0.8*V_10) after height(b2)=H_20 & velocity(b2)=V_20 & height(b1)=H_10 & velocity(b1)=V_10 & hitGround(b2).

% Planning
:- query
label :: init;
maxstep :: 7;
0:mode=1;
0:height(b1)=2;
0:height(b2)=3;
0:velocity(b1)=0;
0:velocity(b2)=0.
\end{lstlisting}

\subsubsection{$C+$ Modulo ODE input}
\begin{lstlisting}[breaklines]
:- sorts
ball.

:- objects
b1,b2::ball.

:- constants
height(ball) :: continuousFluent(0..50);
velocity(ball) :: continuousFluent(-30..30);
hitGround(ball):: exogenousAction.

:- variables
B :: ball;
B1 :: ball;
V,D,H,H_10,H_20,V_10,V_20.

caused duration = 0 if hitGround(B).

default wait.
caused ~wait if hitGround(B).

%Rates
derivative of height(B) is velocity(B) if mode=1.
derivative of velocity(B) is -g if mode=1.

% Guard Translation
nonexecutable hitGround(B) if -(height(B)=0).

% Hevent Restrictions
nonexecutable hitGround(B) if -(mode=1).

% Invariant Translation
constraint (mode=1 ->> (height(B)=H ->> H>=0)).
always_t (height(B)=H ->> H>=0) if mode=1.

% Reset Translation
constraint (height(b1)=H_10 & velocity(b1)=V_10 & height(b2)=0 & velocity(b2)=-0.9*V_20) after height(b2)=H_20 & velocity(b2)=V_20 & height(b1)=H_10 & velocity(b1)=V_10 & hitGround(b2).
constraint (height(b2)=H_20 & velocity(b2)=V_20 & height(b1)=0 & velocity(b1)=-0.8*V_10) after height(b2)=H_20 & velocity(b2)=V_20 & height(b1)=H_10 & velocity(b1)=V_10 & hitGround(b2).

% Planning
:- query
label :: init;
maxstep :: 7;
0:mode=1;
0:height(b1)=2;
0:height(b2)=3;
0:velocity(b1)=0;
0:velocity(b2)=0.
\end{lstlisting}

\subsubsection{Extended ASPMT description}
\begin{lstlisting}[breaklines]
:- sorts
events,ball;astep;step.

:- objects
0..maxstep :: step;
0..maxstep-1 :: astep;
b1,b2::ball.

:- constants
step:mode :: integer[0..5];
0:height(ball) :: real[0..50];
astep:height_t(ball) :: real[0..50];
0:velocity(ball) :: real[0..50];
astep:velocity_t(ball) :: real[-30..30];
astep:hitGround(ball)::boolean;
astep:wait :: boolean;
astep:duration :: real[0..100].

:- variables
E :: events;
B :: ball;
B1 :: ball.

{AS+1:mode=X} <- AS:mode=X.
{AS:height_t(B)=X}.
{AS:velocity_t(B)=X}.
{AS:duration=X}.
{AS:hitGround(B)=X}.

%System Rules
%All actions have 0 duration
AS:duration=0 <- AS:hitGround(B).

{wait(AS)=true}.
AS:wait=false <- AS:hitGround(B).

%Rates
d/dt[height(B)](1)=velocity(B).
d/dt[velocity(B)](1)=-g.

% Flow Translation
[0:height_t(b1), 0:velocity_t(b1), 0:height_t(b2), 0:velocity_t(b2)]=int(0,T,[H_10, V_10, H_20, V_20],1) <- 0:velocity(b1)=V_10 & 0:height(b1) = H_10 & 0:velocity(b2)=V_20 & 0:height(b2) = H_20 & mode=1 & 0:duration = T & 0:wait.

% Guard Translation
<- 0:hitGround(B) & -(0:height(B)=0).
<- AS:hitGround(B) & -(AS-1:height_t(B)=0).

% Hevent Restrictions
<- AS:hitGround(B) & -(AS:mode=1).

% Invariant Translation
<-  (0:mode=1 ->> (0:height(B)=H ->> H>=0)).
<-  (AS:mode=1 ->> (AS-1:height_t(B)=H ->> H>=0).
forall_t 1 [0 AS:duration] (AS-1:height_t(B)=H -> H>=0).
% Invariant Translation


% Reset Translation
<- not (0:height_t(b1)=H_10 & 0:velocity_t(b1)=V_10 & 0:height_t(b2)=0 & 0:velocity_t(b2)=-0.9*V_20) & 0:height(b2)=H_20 & 0:velocity(b2)=V_20 & 0:height(b1)=H_10 & 0:velocity(b1)=V_10 & 0:hitGround(b2).
<- not (AS:height_t(b1)=H_10 & AS:velocity_t(b1)=V_10 & AS:height_t(b2)=0 & AS:velocity_t(b2)=-0.9*V_20) & AS-1:height(b2)=H_20 & AS-1:velocity(b2)=V_20 & AS-1:height(b1)=H_10 & AS-1:velocity_t(b1)=V_10 & AS:hitGround(b2).
<- not (0:height_t(b1)=H_10 & 0:velocity_t(b1)=V_10 & 0:height_t(b2)=0 & 0:velocity_t(b2)=-0.9*V_20) & 0:height(b2)=H_20 & 0:velocity(b2)=V_20 & 0:height(b1)=H_10 & 0:velocity(b1)=V_10 & 0:hitGround(b2).
<- not (AS:height_t(b1)=0 & AS:velocity_t(b1)=-0.8*V_10 & AS:height_t(b2)=H_20 & AS:velocity_t(b2)=V_20) & AS-1:height(b2)=H_20 & AS-1:velocity(b2)=V_20 & AS-1:height(b1)=H_10 & AS-1:velocity_t(b1)=V_10 & AS:hitGround(b1).

% Planning
<- not 0:mode=1;
<- not 0:height(b1)=2;
<- not 0:height(b2)=3;
<- not 0:velocity(b1)=0;
<- not 0:velocity(b2)=0.

\end{lstlisting}

\subsubsection{SMT2 Encoding in dReal}

\begin{lstlisting}[breaklines]
(set-logic QF_NRA_ODE)
(declare-const true_a Bool)
(declare-const false_a Bool)
(declare-const height Real)
(declare-const height_t Real)
(declare-const velocity Real)
(declare-const velocity_t Real)
(declare-const velocity_b2 Real)
(declare-const velocity_b1 Real)
(declare-const height_b1 Real)
(declare-const height_b2 Real)
(declare-const duration_0 Real)
(declare-const duration_1 Real)
(declare-const duration_2 Real)
(declare-const duration_3 Real)
(declare-const height_b1_0 Real)
(declare-const height_b1_0_t Real)
(declare-const height_b1_1_t Real)
(declare-const height_b1_2_t Real)
(declare-const height_b1_3_t Real)
(declare-const height_b2_0 Real)
(declare-const height_b2_0_t Real)
(declare-const height_b2_1_t Real)
(declare-const height_b2_2_t Real)
(declare-const height_b2_3_t Real)
(declare-const hitGround_b1__0_ Bool)
(declare-const hitGround_b1__1_ Bool)
(declare-const hitGround_b1__2_ Bool)
(declare-const hitGround_b1__3_ Bool)
(declare-const hitGround_b2__0_ Bool)
(declare-const hitGround_b2__1_ Bool)
(declare-const hitGround_b2__2_ Bool)
(declare-const hitGround_b2__3_ Bool)
(declare-const mode_0_ Real)
(declare-const mode_1_ Real)
(declare-const mode_2_ Real)
(declare-const mode_3_ Real)
(declare-const mode_4_ Real)
(declare-const qlabel_init_ Bool)
(declare-const velocity_b1_0 Real)
(declare-const velocity_b1_0_t Real)
(declare-const velocity_b1_1_t Real)
(declare-const velocity_b1_2_t Real)
(declare-const velocity_b1_3_t Real)
(declare-const velocity_b2_0 Real)
(declare-const velocity_b2_0_t Real)
(declare-const velocity_b2_1_t Real)
(declare-const velocity_b2_2_t Real)
(declare-const velocity_b2_3_t Real)
(declare-const wait_0_ Bool)
(declare-const wait_1_ Bool)
(declare-const wait_2_ Bool)
(declare-const wait_3_ Bool)
(define-ode flow_1 ((= d/dt[height_b1] velocity_b1)(= d/dt[velocity_b1] -10)(= d/dt[height_b2] velocity_b2)(= d/dt[velocity_b2] -10)))
(assert true_a)
(assert (not false_a))
(assert (>= duration_0 0))
(assert (<= duration_0 50))
(assert (>= duration_1 0))
(assert (<= duration_1 50))
(assert (>= duration_2 0))
(assert (<= duration_2 50))
(assert (>= duration_3 0))
(assert (<= duration_3 50))
(assert (>= height_b1_0 -50))
(assert (<= height_b1_0 50))
(assert (>= height_b1_0_t -50))
(assert (<= height_b1_0_t 50))
(assert (>= height_b1_1_t -50))
(assert (<= height_b1_1_t 50))
(assert (>= height_b1_2_t -50))
(assert (<= height_b1_2_t 50))
(assert (>= height_b1_3_t -50))
(assert (<= height_b1_3_t 50))
(assert (>= height_b2_0 -50))
(assert (<= height_b2_0 50))
(assert (>= height_b2_0_t -50))
(assert (<= height_b2_0_t 50))
(assert (>= height_b2_1_t -50))
(assert (<= height_b2_1_t 50))
(assert (>= height_b2_2_t -50))
(assert (<= height_b2_2_t 50))
(assert (>= height_b2_3_t -50))
(assert (<= height_b2_3_t 50))
(assert (>= mode_0_ 0))
(assert (<= mode_0_ 3))
(assert (>= mode_1_ 0))
(assert (<= mode_1_ 3))
(assert (>= mode_2_ 0))
(assert (<= mode_2_ 3))
(assert (>= mode_3_ 0))
(assert (<= mode_3_ 3))
(assert (>= mode_4_ 0))
(assert (<= mode_4_ 3))
(assert (>= velocity_b1_0 -30))
(assert (<= velocity_b1_0 30))
(assert (>= velocity_b1_0_t -30))
(assert (<= velocity_b1_0_t 30))
(assert (>= velocity_b1_1_t -30))
(assert (<= velocity_b1_1_t 30))
(assert (>= velocity_b1_2_t -30))
(assert (<= velocity_b1_2_t 30))
(assert (>= velocity_b1_3_t -30))
(assert (<= velocity_b1_3_t 30))
(assert (>= velocity_b2_0 -30))
(assert (<= velocity_b2_0 30))
(assert (>= velocity_b2_0_t -30))
(assert (<= velocity_b2_0_t 30))
(assert (>= velocity_b2_1_t -30))
(assert (<= velocity_b2_1_t 30))
(assert (>= velocity_b2_2_t -30))
(assert (<= velocity_b2_2_t 30))
(assert (>= velocity_b2_3_t -30))
(assert (<= velocity_b2_3_t 30))
(assert (or (or (and (= wait_3_ true) (= wait_3_ true)) (and (= hitGround_b2__3_ true) (= wait_3_ false))) (and (= hitGround_b1__3_ true) (= wait_3_ false))))
(assert (=> (= wait_3_ true) (= wait_3_ true)))
(assert (=> (= hitGround_b2__3_ true) (= wait_3_ false)))
(assert (=> (= hitGround_b1__3_ true) (= wait_3_ false)))
(assert (or (or (and (= wait_2_ true) (= wait_2_ true)) (and (= hitGround_b2__2_ true) (= wait_2_ false))) (and (= hitGround_b1__2_ true) (= wait_2_ false))))
(assert (=> (= wait_2_ true) (= wait_2_ true)))
(assert (=> (= hitGround_b2__2_ true) (= wait_2_ false)))
(assert (=> (= hitGround_b1__2_ true) (= wait_2_ false)))
(assert (or (or (and (= wait_1_ true) (= wait_1_ true)) (and (= hitGround_b2__1_ true) (= wait_1_ false))) (and (= hitGround_b1__1_ true) (= wait_1_ false))))
(assert (=> (= wait_1_ true) (= wait_1_ true)))
(assert (=> (= hitGround_b2__1_ true) (= wait_1_ false)))
(assert (=> (= hitGround_b1__1_ true) (= wait_1_ false)))
(assert (or (or (and (= wait_0_ true) (= wait_0_ true)) (and (= hitGround_b2__0_ true) (= wait_0_ false))) (and (= hitGround_b1__0_ true) (= wait_0_ false))))
(assert (=> (= wait_0_ true) (= wait_0_ true)))
(assert (=> (= hitGround_b2__0_ true) (= wait_0_ false)))
(assert (=> (= hitGround_b1__0_ true) (= wait_0_ false)))
(assert (= qlabel_init_ true))
(assert (= qlabel_init_ true))
(assert (= mode_4_ mode_3_))
(assert (= mode_3_ mode_2_))
(assert (= mode_2_ mode_1_))
(assert (= mode_1_ mode_0_))
(assert (=> (= hitGround_b2__3_ true) (= duration_3 0)))
(assert (=> (= hitGround_b1__3_ true) (= duration_3 0)))
(assert (=> (= hitGround_b2__2_ true) (= duration_2 0)))
(assert (=> (= hitGround_b1__2_ true) (= duration_2 0)))
(assert (=> (= hitGround_b2__1_ true) (= duration_1 0)))
(assert (=> (= hitGround_b1__1_ true) (= duration_1 0)))
(assert true_a)
(assert (=> (= hitGround_b2__0_ true) (= duration_0 0)))
(assert (=> (= hitGround_b1__0_ true) (= duration_0 0)))
; Flow:
; caused [height(B)=H_1, velocity(B)=V_1] after velocity(B)=V0 & duration = T & height(B) = H0 & H_1=int(0,T,[H0, V0],[d/dt[height](s0), d/dt[velocity](s0)]) & s0 & wait.
(assert (=> (and (= mode_0_ 1) (= wait_0_ true)) (= [height_b1_0_t velocity_b1_0_t height_b2_0_t velocity_b2_0_t] (integral 0. duration_0 [height_b1_0 velocity_b1_0 height_b2_0 velocity_b2_0] flow_1))))
(assert (=> (and (= mode_1_ 1) (= wait_1_ true)) (= [height_b1_1_t velocity_b1_1_t height_b2_1_t velocity_b2_1_t] (integral 0. duration_1 [height_b1_0_t velocity_b1_0_t height_b2_0_t velocity_b2_0_t] flow_1))))
(assert (=> (and (= mode_2_ 1) (= wait_2_ true)) (= [height_b1_2_t velocity_b1_2_t height_b2_2_t velocity_b2_2_t] (integral 0. duration_2 [height_b1_1_t velocity_b1_1_t height_b2_1_t velocity_b2_1_t] flow_1))))
(assert (=> (and (= mode_3_ 1) (= wait_3_ true)) (= [height_b1_3_t velocity_b1_3_t height_b2_3_t velocity_b2_3_t] (integral 0. duration_3 [height_b1_2_t velocity_b1_2_t height_b2_2_t velocity_b2_2_t] flow_1))))
(assert (not (and (= hitGround_b2__0_ true) (not (= height_b2_0 0)))))
(assert (not (and (= hitGround_b1__0_ true) (not (= height_b1_0 0)))))
(assert (not (and (= hitGround_b2__1_ true) (not (= height_b2_0_t 0)))))
(assert (not (and (= hitGround_b2__2_ true) (not (= height_b2_1_t 0)))))
(assert (not (and (= hitGround_b2__3_ true) (not (= height_b2_2_t 0)))))
(assert (not (and (= hitGround_b1__1_ true) (not (= height_b1_0_t 0)))))
(assert (not (and (= hitGround_b1__2_ true) (not (= height_b1_1_t 0)))))
(assert (not (and (= hitGround_b1__3_ true) (not (= height_b1_2_t 0)))))
(assert (not (and (= hitGround_b2__0_ true) (not (= mode_0_ 1)))))
(assert (not (and (= hitGround_b2__1_ true) (not (= mode_1_ 1)))))
(assert (not (and (= hitGround_b2__2_ true) (not (= mode_2_ 1)))))
(assert (not (and (= hitGround_b2__3_ true) (not (= mode_3_ 1)))))
(assert (not (and (= hitGround_b1__0_ true) (not (= mode_0_ 1)))))
(assert (not (and (= hitGround_b1__1_ true) (not (= mode_1_ 1)))))
(assert (not (and (= hitGround_b1__2_ true) (not (= mode_2_ 1)))))
(assert (not (and (= hitGround_b1__3_ true) (not (= mode_3_ 1)))))
(assert (not (and (not (= height_b1_0_t 0)) (= hitGround_b1__0_ true))))
(assert (not (and (not (= height_b2_0_t 0)) (= hitGround_b2__0_ true))))
(assert (not (and (not (= velocity_b1_0_t (* -0.8 velocity_b1_0))) (= hitGround_b1__0_ true))))
(assert (not (and (not (= velocity_b2_0_t (* -0.9 velocity_b2_0))) (= hitGround_b2__0_ true))))
(assert (not (and (not (= height_b2_1_t 0)) (= hitGround_b2__1_ true))))
(assert (not (and (not (= height_b2_2_t 0)) (= hitGround_b2__2_ true))))
(assert (not (and (not (= height_b2_3_t 0)) (= hitGround_b2__3_ true))))
(assert (not (and (not (= height_b1_1_t 0)) (= hitGround_b1__1_ true))))
(assert (not (and (not (= height_b1_2_t 0)) (= hitGround_b1__2_ true))))
(assert (not (and (not (= height_b1_3_t 0)) (= hitGround_b1__3_ true))))
(assert (not (and (not (= velocity_b2_1_t (* -0.9 velocity_b2_0_t))) (= hitGround_b2__1_ true))))
(assert (not (and (not (= velocity_b2_2_t (* -0.9 velocity_b2_1_t))) (= hitGround_b2__2_ true))))
(assert (not (and (not (= velocity_b2_3_t (* -0.9 velocity_b2_2_t))) (= hitGround_b2__3_ true))))
(assert (not (and (not (= velocity_b1_1_t (* -0.8 velocity_b1_0_t))) (= hitGround_b1__1_ true))))
(assert (not (and (not (= velocity_b1_2_t (* -0.8 velocity_b1_1_t))) (= hitGround_b1__2_ true))))
(assert (not (and (not (= velocity_b1_3_t (* -0.8 velocity_b1_2_t))) (= hitGround_b1__3_ true))))
(assert (not (and (not (= height_b1_0_t height_b1_0)) (= hitGround_b2__0_ true))))
(assert (not (and (not (= height_b2_0_t height_b2_0)) (= hitGround_b1__0_ true))))
(assert (not (and (not (= velocity_b1_0_t velocity_b1_0)) (= hitGround_b2__0_ true))))
(assert (not (and (not (= velocity_b2_0_t velocity_b2_0)) (= hitGround_b1__0_ true))))
(assert (not (and (not (= height_b2_1_t height_b2_0_t)) (= hitGround_b1__1_ true))))
(assert (not (and (not (= height_b2_2_t height_b2_1_t)) (= hitGround_b1__2_ true))))
(assert (not (and (not (= height_b2_3_t height_b2_2_t)) (= hitGround_b1__3_ true))))
(assert (not (and (not (= height_b1_2_t height_b1_1_t)) (= hitGround_b2__2_ true))))
(assert (not (and (not (= height_b1_3_t height_b1_2_t)) (= hitGround_b2__3_ true))))
(assert (not (and (not (= velocity_b2_1_t velocity_b2_0_t)) (= hitGround_b1__1_ true))))
(assert (not (and (not (= velocity_b2_2_t velocity_b2_1_t)) (= hitGround_b1__2_ true))))
(assert (not (and (not (= velocity_b2_3_t velocity_b2_2_t)) (= hitGround_b1__3_ true))))
(assert (not (and (not (= height_b1_1_t height_b1_0_t)) (= hitGround_b2__1_ true))))
(assert (not (and (not (= velocity_b1_1_t velocity_b1_0_t)) (= hitGround_b2__1_ true))))
(assert (not (and (not (= velocity_b1_2_t velocity_b1_1_t)) (= hitGround_b2__2_ true))))
(assert (not (and (not (= velocity_b1_3_t velocity_b1_2_t)) (= hitGround_b2__3_ true))))
(assert (not (not (= mode_0_ 1))))
(assert (not (not (= height_b1_0 2))))
(assert (not (not (= height_b2_0 3))))
(assert (not (not (= velocity_b1_0 0))))
(assert (not (not (= velocity_b2_0 0))))
(assert (not (= hitGround_b1__1_ false)))
(assert (not (= hitGround_b2__3_ false)))

(check-sat)
(exit)

\end{lstlisting}

\subsubsection{Output} 
\begin{figure}
  \includegraphics[height=7cm]{2ball-output.png}
  \caption{Output}\label{fig:ball-out}
\end{figure}
\begin{lstlisting}
Solution:
duration_0 : [ ENTIRE ] = [0.6324539184570312, 0.6324558258056641]
duration_1 : [ ENTIRE ] = [0, 0]
duration_2 : [ ENTIRE ] = [0.1421389579772949, 0.1421456336975098]
duration_3 : [ ENTIRE ] = [0, 0]
height_b1_0 : [ ENTIRE ] = [2, 2]
height_b1_0_t : [ ENTIRE ] = [0, 0]
height_b1_1_t : [ ENTIRE ] = [0, 0]
height_b1_2_t : [ ENTIRE ] = [0.6180163634126075, 0.6182833269122059]
height_b1_3_t : [ ENTIRE ] = [0.6180163634126075, 0.6182833269122059]
height_b2_0 : [ ENTIRE ] = [3, 3]
height_b2_0_t : [ ENTIRE ] = [0.9999981420223776, 1.000010205141734]
height_b2_1_t : [ ENTIRE ] = [0.9999981420223776, 1.000010205141734]
height_b2_2_t : [ ENTIRE ] = [0, 0]
height_b2_3_t : [ ENTIRE ] = [0, 0]
mode_0_ : [ ENTIRE ] = [1, 1]
mode_1_ : [ ENTIRE ] = [1, 1]
mode_2_ : [ ENTIRE ] = [1, 1]
mode_3_ : [ ENTIRE ] = [1, 1]
mode_4_ : [ ENTIRE ] = [1, 1]
velocity_b1_0 : [ ENTIRE ] = [0, 0]
velocity_b1_0_t : [ ENTIRE ] = [-6.324558258056641, -6.324539184570312]
velocity_b1_1_t : [ ENTIRE ] = [5.05963134765625, 5.059646606445313]
velocity_b1_2_t : [ ENTIRE ] = [3.637947082519531, 3.63872528076172]
velocity_b1_3_t : [ ENTIRE ] = [3.637947082519531, 3.63872528076172]
velocity_b2_0 : [ ENTIRE ] = [0, 0]
velocity_b2_0_t : [ ENTIRE ] = [-6.324558258056641, -6.324539184570312]
velocity_b2_1_t : [ ENTIRE ] = [-6.324558258056641, -6.324539184570312]
velocity_b2_2_t : [ ENTIRE ] = [-7.746391296386719, -7.7459716796875]
velocity_b2_3_t : [ ENTIRE ] = [6.97137451171875, 6.971752166748048]
true_a : Bool = true
false_a : Bool = false
hitGround_b1__0_ : Bool = false
hitGround_b1__1_ : Bool = true
hitGround_b1__2_ : Bool = false
hitGround_b1__3_ : Bool = false
hitGround_b2__0_ : Bool = false
hitGround_b2__1_ : Bool = false
hitGround_b2__2_ : Bool = false
hitGround_b2__3_ : Bool = true
qlabel_init_ : Bool = true
wait_0_ : Bool = true
wait_1_ : Bool = false
wait_2_ : Bool = true
wait_3_ : Bool = false
delta-sat with delta = 0.00100000000000000
\end{lstlisting}

\EOCCC

\subsection{Turning Car --- Non-convex Invariants}
\begin{center}
  \includegraphics[height=7cm]{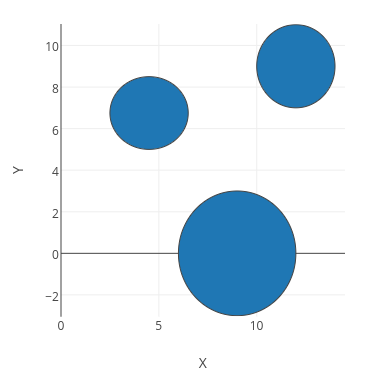}
\end{center}

Consider a car that is moving at a constant speed of 1 unit. The car is initially at origin where $x=0$ and $y=0$ and $\theta=0$. Additionally there are pillars defined by the equations $(x-6)^2+y^2\le 9$,$(x-5)^2+(y-7)^2\le 4$,$(x-12)^2+(y-9)^2\le 4$. The goal is to find a plan such that the car ends up at $x=13$ and $y=0$ without hitting the pillars. 

The dynamics of the car is as follows:
\bi
\item
Moving Straight
\[\frac{d[x]}{dt}= \cos(\theta),\ \frac{d[y]}{dt}= \sin(\theta),\ \frac{d[theta]}{dt}=0\]
\item
Turning Left
\[\frac{d[x]}{dt}= \cos(\theta),\ \frac{d[y]}{dt}= \sin(\theta),\ \frac{d[theta]}{dt}= \tan(\frac{\pi}{18})\]
\item
Turning Right
\[\frac{d[x]}{dt}= \cos(\theta),\ \frac{d[y]}{dt}= \sin(\theta),\ \frac{d[theta]}{dt}= \tan(-\frac{\pi}{18})\]
\ei

%The invariant in the other example were just simple logical formulas. 
We assume the car is a pint. For the car not to hit the pillars, the invariants are $(x-9)^2+y^2>9$,$(x-5)^2+(y-7)^2>4$,$(x-12)^2+(y-9)^2>4$.

\subsubsection{Hybrid Automata Components}
\begin{center}
  \includegraphics[height=10cm]{car-turn-HA.png}
\end{center}
\begin{itemize}
\item Variables:
\begin{itemize}
\item \(X,X',\dot{X}\)
\item \(Y,Y',\dot{Y}\)
\item \(Theta, Theta',\dot{Theta}\)
\end{itemize}
\item States:
\begin{itemize}
\item MoveStraight ($mode=1$)
\item MoveLeft ($mode=2$)
\item MoveRight ($mode=3$)
\end{itemize}
\item Directed Graph: The graph is given above.
\item H-events:
\begin{itemize}
\item Straighten
\item TurnLeft
\item TurnRight
\end{itemize}
\item Invariants: 
\begin{itemize} 
\item ${\sf Inv}(all modes): ((X-9)^2+Y^2>9)\land\ ((X-5)^2+(Y-7)^2>4)\land\ ((X-12)^2+(Y-9)^2>4))$
\end{itemize}
\item Flow: 
\begin{itemize} 
\item \({\sf Flow}(1)(X,Y,Theta): \dot{X}=sin(Theta)\ \land\ \dot{Y}=cos(Theta) \land\ \dot{Theta}=0\).
\item \({\sf Flow}(2)(X,Y,Theta): \dot{X}=sin(Theta)\ \land\ \dot{Y}=cos(Theta) \land\ \dot{Theta}=tan(\pi/18)\).
\item \({\sf Flow}(3)(X,Y,Theta): \dot{X}=sin(Theta)\ \land\ \dot{Y}=cos(Theta) \land\ \dot{Theta}=tan(-\pi/18)\).
\end{itemize}
\item Reset:
\begin{itemize} 
\item  ${\sf Reset}((MoveRight, MoveStraight)): X' = X\ \land\ Y'=Y\ \land\ Theta'=Theta$

\item  ${\sf Reset}((MoveLeft, MoveStraight)): X' = X\ \land\ Y'=Y\ \land\ Theta'=Theta$

\item  ${\sf Reset}( (MoveStraight, MoveLeft) ): X' = X\ \land\ Y'=Y\ \land\ Theta'=Theta$

\item  ${\sf Reset}( (MoveRight, MoveLeft) ): X' = X\ \land\ Y'=Y\ \land\ Theta'=Theta$

\item ${\sf Reset}( (MoveStraight, MoveRight) ): X' = X\ \land\ Y'=Y\ \land\ Theta'=Theta$

\item ${\sf Reset}( (MoveLeft, MoveRight) ): X' = X\ \land\ Y'=Y\ \land\ Theta'=Theta$

\end{itemize}
\end{itemize}

\subsubsection{In the Input Language of {\sc cplus2aspmt}}\label{ssec:car-cplus}
\begin{lstlisting}[breaklines]
% File: car.cp

:- constants
x         :: differentiableFluent(0..40);
y         :: differentiableFluent(-50..50);
theta     :: differentiableFluent(-50..50);
straighten, 
turnLeft, 
turnRight :: exogenousAction.

:- variables
X,X0,S,Y,X1,X2,D,D1,T,RP,R.

% Reset
constraint (x=D & y=X0 & theta=X1) after x=D & y=X0 & theta=X1 & turnLeft.
constraint (x=D & y=X0 & theta=X1) after x=D & y=X0 & theta=X1 & turnRight.
constraint (x=D & y=X0 & theta=X1) after x=D & y=X0 & theta=X1 & straighten.

% Mode
straighten causes mode=1.
turnLeft causes mode=2.
turnRight causes mode=3.
nonexecutable straighten if mode=1.
nonexecutable turnLeft if mode=2.
nonexecutable turnRight if mode=3.

% Duration
straighten causes duration=0.
turnRight causes duration=0.
turnLeft causes duration=0.

% Wait
default wait.
straighten causes ~wait.
turnLeft causes ~wait.
turnRight causes ~wait.  

% Rates
derivative of theta is 0 if mode=1.
derivative of y is sin(theta) if mode=1.
derivative of x is cos(theta) if mode=1.

derivative of theta is tan(0.226893) if mode=2.
derivative of y is  sin(theta) if mode=2.
derivative of x is cos(theta) if mode=2.

derivative of theta is tan(-0.226893) if mode=3.
derivative of y is sin(theta) if mode=3.
derivative of x is cos(theta) if mode=3.

% Invariant
constraint (x=X & y=Y ->> (X-9)*(X-9) + Y*Y > 9).
always_t (x=X & y=Y ->> (X-9)*(X-9) + Y*Y > 9) if mode=1.
always_t (x=X & y=Y ->> (X-9)*(X-9) + Y*Y > 9) if mode=2.
always_t (x=X & y=Y ->> (X-9)*(X-9) + Y*Y > 9) if mode=3.

constraint (x=X & y=Y ->> (X-5)*(X-5) + (Y-7)*(Y-7)>4).
always_t (x=X & y=Y ->> (X-5)*(X-5) + (Y-7)*(Y-7)>4) if mode=1.
always_t (x=X & y=Y ->> (X-5)*(X-5) + (Y-7)*(Y-7)>4) if mode=2.
always_t (x=X & y=Y ->> (X-5)*(X-5) + (Y-7)*(Y-7)>4) if mode=3.

constraint (x=X & y=Y ->> (X-12)*(X-12) + (Y-9)*(Y-9)>4).
always_t (x=X & y=Y ->> (X-12)*(X-12) + (Y-9)*(Y-9)>4) if mode=1.
always_t (x=X & y=Y ->> (X-12)*(X-12) + (Y-9)*(Y-9)>4) if mode=2.
always_t (x=X & y=Y ->> (X-12)*(X-12) + (Y-9)*(Y-9)>4) if mode=3.

:- query
label :: test;
0:x=0;
0:y=0;
0:theta=0.69183;
0:mode=1;
3:x=13;
3:y=0.
\end{lstlisting}

\BOCCC
\subsubsection{{\sc aspmt} Input Generated by {\sc cplus2aspmt}}

\begin{lstlisting}[breaklines]
% Sort Declarations ---------------------------------------------------
:- sorts
step;astep;qSort;pstep;zero.

% Variable Declarations -----------------------------------------------
:- variables
ST::pstep;
AS::astep.

% Object Declarations -----------------------------------------------
:- objects
0 :: zero;
0..maxstep :: step;
1..maxstep :: pstep;
0..maxstep-1 :: astep;
query :: qSort.

% Constant Declarations -----------------------------------------------
:- constants
mode(step)::real[0..50];
duration(astep)::real[0..50];
wait(astep)::boolean;
x(zero)::continuousFluent[0..40];
x_t(astep)::continuousFluent[0..40];
y(zero)::continuousFluent[-50..50];
y_t(astep)::continuousFluent[-50..50];
theta(zero)::continuousFluent[-50..50];
theta_t(astep)::continuousFluent[-50..50];
straighten(astep)::boolean;
turnLeft(astep)::boolean;
turnRight(astep)::boolean;
qlabel(qSort) :: boolean.


{mode(ST)=GVAR_real} <- mode(ST-1)=GVAR_real.
{mode(0)=GVAR_real}.
{duration(ST-1)=GVAR_real}.
{x(0)=GVAR_continuous}.
{x_t(ST-1)=GVAR_continuous}.
{y(0)=GVAR_continuous}.
{y_t(ST-1)=GVAR_continuous}.
{theta(0)=GVAR_continuous}.
{theta_t(ST-1)=GVAR_continuous}.
{straighten(ST-1)=GVAR_boolean}.
{turnLeft(ST-1)=GVAR_boolean}.
{turnRight(ST-1)=GVAR_boolean}.

% Guard

% Reset

false <- not (x_t(0)=D & y_t(0)=X0 & theta_t(0)=X1) & x(0)=D & y(0)=X0 & theta(0)=X1 & turnLeft(0)=true.
false <- not (x_t(ST)=D & y_t(ST)=X0 & theta_t(ST)=X1) & x_t(ST-1)=D & y_t(ST-1)=X0 & theta_t(ST-1)=X1 & turnLeft(ST)=true.

false <- not (x_t(0)=D & y_t(0)=X0 & theta_t(0)=X1) & x(0)=D & y(0)=X0 & theta(0)=X1 & turnRight(0)=true.
false <- not (x_t(ST)=D & y_t(ST)=X0 & theta_t(ST)=X1) & x_t(ST-1)=D & y_t(ST-1)=X0 & theta_t(ST-1)=X1 & turnRight(ST)=true.

false <- not (x_t(0)=D & y_t(0)=X0 & theta_t(0)=X1) & x(0)=D & y(0)=X0 & theta(0)=X1 & straighten(0)=true.
false <- not (x_t(ST)=D & y_t(ST)=X0 & theta_t(ST)=X1) & x_t(ST-1)=D & y_t(ST-1)=X0 & theta_t(ST-1)=X1 & straighten(ST)=true.

% Mode

mode(ST)=1 <- straighten(ST-1)=true.

mode(ST)=2 <- turnLeft(ST-1)=true.

mode(ST)=3 <- turnRight(ST-1)=true.

false <- straighten(ST-1)=true & mode(ST-1)=1.

false <- turnLeft(ST-1)=true & mode(ST-1)=2.

false <- turnRight(ST-1)=true & mode(ST-1)=3.

% Duration

duration(ST-1)=0 <- straighten(ST-1)=true.

duration(ST-1)=0 <- turnRight(ST-1)=true.

duration(ST-1)=0 <- turnLeft(ST-1)=true.

% Wait

wait(ST-1)=true <- wait(ST-1)=true.

wait(ST-1)=false <- straighten(ST-1)=true.

wait(ST-1)=false <- turnLeft(ST-1)=true.

wait(ST-1)=false <- turnRight(ST-1)=true.

% Rates

d/dt[theta](1)=0.

d/dt[y](1)=sin(theta).

d/dt[x](1)=cos(theta).

d/dt[theta](2)=tan(0.226893).

d/dt[y](2)=sin(theta).

d/dt[x](2)=cos(theta).

d/dt[theta](3)=tan(-0.226893).

d/dt[y](3)=sin(theta).

d/dt[x](3)=cos(theta).

%Invariant

false <- not (x(0)=X & y(0)=Y -> (X-9)*(X-9)+Y*Y>9).
false <- not (x_t(ST-1)=X & y_t(ST-1)=Y -> (X-9)*(X-9)+Y*Y>9).

always (x_t(ST-1)=X & y_t(ST-1)=Y -> (X-9)*(X-9)+Y*Y>9) <- mode=1 & duration(ST-1)=D.

always (x_t(ST-1)=X & y_t(ST-1)=Y -> (X-9)*(X-9)+Y*Y>9) <- mode=2 & duration(ST-1)=D.

always (x_t(ST-1)=X & y_t(ST-1)=Y -> (X-9)*(X-9)+Y*Y>9) <- mode=3 & duration(ST-1)=D.

false <- not (x(0)=X & y(0)=Y -> (X-5)*(X-5)+(Y-7)*(Y-7)>4).
false <- not (x_t(ST-1)=X & y_t(ST-1)=Y -> (X-5)*(X-5)+(Y-7)*(Y-7)>4).

always (x_t(ST-1)=X & y_t(ST-1)=Y -> (X-5)*(X-5)+(Y-7)*(Y-7)>4) <- mode=1 & duration(ST-1)=D.

always (x_t(ST-1)=X & y_t(ST-1)=Y -> (X-5)*(X-5)+(Y-7)*(Y-7)>4) <- mode=2 & duration(ST-1)=D.

always (x_t(ST-1)=X & y_t(ST-1)=Y -> (X-5)*(X-5)+(Y-7)*(Y-7)>4) <- mode=3 & duration(ST-1)=D.

false <- not (x(0)=X & y(0)=Y -> (X-12)*(X-12)+(Y-9)*(Y-9)>4).
false <- not (x_t(ST-1)=X & y_t(ST-1)=Y -> (X-12)*(X-12)+(Y-9)*(Y-9)>4).

always (x_t(ST-1)=X & y_t(ST-1)=Y -> (X-12)*(X-12)+(Y-9)*(Y-9)>4) <- mode=1 & duration(ST-1)=D.

always (x_t(ST-1)=X & y_t(ST-1)=Y -> (X-12)*(X-12)+(Y-9)*(Y-9)>4) <- mode=2 & duration(ST-1)=D.

always (x_t(ST-1)=X & y_t(ST-1)=Y -> (X-12)*(X-12)+(Y-9)*(Y-9)>4) <- mode=3 & duration(ST-1)=D.


%Integral Law
[CT10, CT11, CT12]=int(0,D,[CT00, CT01, CT02],1)<- mode(0)=1 & duration(0)=D & x(0)=CT00 & y(0)=CT01 & theta(0)=CT02 & x_t(0)=CT10 & y_t(0)=CT11 & theta_t(0)=CT12 & wait(0).
[CT10, CT11, CT12]=int(0,D,[CT00, CT01, CT02],1)<- mode(ST)=1 & duration(ST)=D & x_t(ST-1)=CT00 & y_t(ST-1)=CT01 & theta_t(ST-1)=CT02 & x_t(ST)=CT10 & y_t(ST)=CT11 & theta_t(ST)=CT12 & wait(ST).
[CT10, CT11, CT12]=int(0,D,[CT00, CT01, CT02],2)<- mode(0)=2 & duration(0)=D & x(0)=CT00 & y(0)=CT01 & theta(0)=CT02 & x_t(0)=CT10 & y_t(0)=CT11 & theta_t(0)=CT12 & wait(0).
[CT10, CT11, CT12]=int(0,D,[CT00, CT01, CT02],2)<- mode(ST)=2 & duration(ST)=D & x_t(ST-1)=CT00 & y_t(ST-1)=CT01 & theta_t(ST-1)=CT02 & x_t(ST)=CT10 & y_t(ST)=CT11 & theta_t(ST)=CT12 & wait(ST).
[CT10, CT11, CT12]=int(0,D,[CT00, CT01, CT02],3)<- mode(0)=3 & duration(0)=D & x(0)=CT00 & y(0)=CT01 & theta(0)=CT02 & x_t(0)=CT10 & y_t(0)=CT11 & theta_t(0)=CT12 & wait(0).
[CT10, CT11, CT12]=int(0,D,[CT00, CT01, CT02],3)<- mode(ST)=3 & duration(ST)=D & x_t(ST-1)=CT00 & y_t(ST-1)=CT01 & theta_t(ST-1)=CT02 & x_t(ST)=CT10 & y_t(ST)=CT11 & theta_t(ST)=CT12 & wait(ST).

% Query

<- not 0:x=0.
<- not 0:y=0.
<- not 0:theta=0.
<- not 0:mode=1.
<- not 5:x=13.
<- not 5:y=0.

\end{lstlisting}

\subsubsection{${\tt dReal}$ Input Generated by {\sc cplus2aspmt}}

\begin{lstlisting}[breaklines]
(set-logic QF_NRA_ODE)
(declare-const true_a Bool)
(declare-const false_a Bool)
(declare-const theta Real)
(declare-const theta_t Real)
(declare-const x Real)
(declare-const x_t Real)
(declare-const y Real)
(declare-const y_t Real)
(declare-const duration_0_ Real)
(declare-const duration_1_ Real)
(declare-const duration_2_ Real)
(declare-const duration_3_ Real)
(declare-const duration_4_ Real)
(declare-const mode_0_ Real)
(declare-const mode_1_ Real)
(declare-const mode_2_ Real)
(declare-const mode_3_ Real)
(declare-const mode_4_ Real)
(declare-const mode_5_ Real)
(declare-const qlabel_init_ Bool)
(declare-const straighten_0_ Bool)
(declare-const straighten_1_ Bool)
(declare-const straighten_2_ Bool)
(declare-const straighten_3_ Bool)
(declare-const straighten_4_ Bool)
(declare-const theta_0_ Real)
(declare-const theta_0_t Real)
(declare-const theta_1_t Real)
(declare-const theta_2_t Real)
(declare-const theta_3_t Real)
(declare-const theta_4_t Real)
(declare-const turnLeft_0_ Bool)
(declare-const turnLeft_1_ Bool)
(declare-const turnLeft_2_ Bool)
(declare-const turnLeft_3_ Bool)
(declare-const turnLeft_4_ Bool)
(declare-const turnRight_0_ Bool)
(declare-const turnRight_1_ Bool)
(declare-const turnRight_2_ Bool)
(declare-const turnRight_3_ Bool)
(declare-const turnRight_4_ Bool)
(declare-const wait_0_ Bool)
(declare-const wait_1_ Bool)
(declare-const wait_2_ Bool)
(declare-const wait_3_ Bool)
(declare-const wait_4_ Bool)
(declare-const x_0_ Real)
(declare-const x_0_t Real)
(declare-const x_1_t Real)
(declare-const x_2_t Real)
(declare-const x_3_t Real)
(declare-const x_4_t Real)
(declare-const y_0_ Real)
(declare-const y_0_t Real)
(declare-const y_1_t Real)
(declare-const y_2_t Real)
(declare-const y_3_t Real)
(declare-const y_4_t Real)
(define-ode flow_1 ((= d/dt[x] (cos theta))(= d/dt[y] (sin theta))(= d/dt[theta] 0)))
(define-ode flow_2 ((= d/dt[x] (cos theta))(= d/dt[y] (sin theta))(= d/dt[theta] (tan (/ 2268 10000)))))
(define-ode flow_3 ((= d/dt[x] (cos theta))(= d/dt[y] (sin theta))(= d/dt[theta] (tan (/ -2268 10000)))))
(assert true_a)
(assert (not false_a))
(assert (>= duration_0_ 0))
(assert (<= duration_0_ 50))
(assert (>= duration_1_ 0))
(assert (<= duration_1_ 50))
(assert (>= duration_2_ 0))
(assert (<= duration_2_ 50))
(assert (>= duration_3_ 0))
(assert (<= duration_3_ 50))
(assert (>= duration_4_ 0))
(assert (<= duration_4_ 50))
(assert (>= mode_0_ 0))
(assert (<= mode_0_ 50))
(assert (>= mode_1_ 0))
(assert (<= mode_1_ 50))
(assert (>= mode_2_ 0))
(assert (<= mode_2_ 50))
(assert (>= mode_3_ 0))
(assert (<= mode_3_ 50))
(assert (>= mode_4_ 0))
(assert (<= mode_4_ 50))
(assert (>= mode_5_ 0))
(assert (<= mode_5_ 50))
(assert (>= theta_0_ -50))
(assert (<= theta_0_ 50))
(assert (>= theta_0_t -50))
(assert (<= theta_0_t 50))
(assert (>= theta_1_t -50))
(assert (<= theta_1_t 50))
(assert (>= theta_2_t -50))
(assert (<= theta_2_t 50))
(assert (>= theta_3_t -50))
(assert (<= theta_3_t 50))
(assert (>= theta_4_t -50))
(assert (<= theta_4_t 50))
(assert (>= x_0_ 0))
(assert (<= x_0_ 40))
(assert (>= x_0_t 0))
(assert (<= x_0_t 40))
(assert (>= x_1_t 0))
(assert (<= x_1_t 40))
(assert (>= x_2_t 0))
(assert (<= x_2_t 40))
(assert (>= x_3_t 0))
(assert (<= x_3_t 40))
(assert (>= x_4_t 0))
(assert (<= x_4_t 40))
(assert (>= y_0_ -50))
(assert (<= y_0_ 50))
(assert (>= y_0_t -50))
(assert (<= y_0_t 50))
(assert (>= y_1_t -50))
(assert (<= y_1_t 50))
(assert (>= y_2_t -50))
(assert (<= y_2_t 50))
(assert (>= y_3_t -50))
(assert (<= y_3_t 50))
(assert (>= y_4_t -50))
(assert (<= y_4_t 50))
(assert (or (or (or (and (= wait_4_ true) (= wait_4_ true)) (and (= turnRight_4_ true) (= wait_4_ false))) (and (= turnLeft_4_ true) (= wait_4_ false))) (and (= straighten_4_ true) (= wait_4_ false))))
(assert (=> (= wait_4_ true) (= wait_4_ true)))
(assert (=> (= turnRight_4_ true) (= wait_4_ false)))
(assert (=> (= turnLeft_4_ true) (= wait_4_ false)))
(assert (=> (= straighten_4_ true) (= wait_4_ false)))
(assert (or (or (or (and (= wait_3_ true) (= wait_3_ true)) (and (= turnRight_3_ true) (= wait_3_ false))) (and (= turnLeft_3_ true) (= wait_3_ false))) (and (= straighten_3_ true) (= wait_3_ false))))
(assert (=> (= wait_3_ true) (= wait_3_ true)))
(assert (=> (= turnRight_3_ true) (= wait_3_ false)))
(assert (=> (= turnLeft_3_ true) (= wait_3_ false)))
(assert (=> (= straighten_3_ true) (= wait_3_ false)))
(assert (or (or (or (and (= wait_2_ true) (= wait_2_ true)) (and (= turnRight_2_ true) (= wait_2_ false))) (and (= turnLeft_2_ true) (= wait_2_ false))) (and (= straighten_2_ true) (= wait_2_ false))))
(assert (=> (= wait_2_ true) (= wait_2_ true)))
(assert (=> (= turnRight_2_ true) (= wait_2_ false)))
(assert (=> (= turnLeft_2_ true) (= wait_2_ false)))
(assert (=> (= straighten_2_ true) (= wait_2_ false)))
(assert (or (or (or (and (= wait_1_ true) (= wait_1_ true)) (and (= turnRight_1_ true) (= wait_1_ false))) (and (= turnLeft_1_ true) (= wait_1_ false))) (and (= straighten_1_ true) (= wait_1_ false))))
(assert (=> (= wait_1_ true) (= wait_1_ true)))
(assert (=> (= turnRight_1_ true) (= wait_1_ false)))
(assert (=> (= turnLeft_1_ true) (= wait_1_ false)))
(assert (=> (= straighten_1_ true) (= wait_1_ false)))
(assert (or (or (or (and (= wait_0_ true) (= wait_0_ true)) (and (= turnRight_0_ true) (= wait_0_ false))) (and (= turnLeft_0_ true) (= wait_0_ false))) (and (= straighten_0_ true) (= wait_0_ false))))
(assert (=> (= wait_0_ true) (= wait_0_ true)))
(assert (=> (= turnRight_0_ true) (= wait_0_ false)))
(assert (=> (= turnLeft_0_ true) (= wait_0_ false)))
(assert (=> (= straighten_0_ true) (= wait_0_ false)))
(assert (= qlabel_init_ true))
(assert (= qlabel_init_ true))
(assert (or (or (or (= mode_5_ mode_4_) (and (= turnRight_4_ true) (= mode_5_ 3))) (and (= turnLeft_4_ true) (= mode_5_ 2))) (and (= straighten_4_ true) (= mode_5_ 1))))
(assert (=> (= turnRight_4_ true) (= mode_5_ 3)))
(assert (=> (= turnLeft_4_ true) (= mode_5_ 2)))
(assert (=> (= straighten_4_ true) (= mode_5_ 1)))
(assert (or (or (or (= mode_4_ mode_3_) (and (= turnRight_3_ true) (= mode_4_ 3))) (and (= turnLeft_3_ true) (= mode_4_ 2))) (and (= straighten_3_ true) (= mode_4_ 1))))
(assert (=> (= turnRight_3_ true) (= mode_4_ 3)))
(assert (=> (= turnLeft_3_ true) (= mode_4_ 2)))
(assert (=> (= straighten_3_ true) (= mode_4_ 1)))
(assert (or (or (or (= mode_3_ mode_2_) (and (= turnRight_2_ true) (= mode_3_ 3))) (and (= turnLeft_2_ true) (= mode_3_ 2))) (and (= straighten_2_ true) (= mode_3_ 1))))
(assert (=> (= turnRight_2_ true) (= mode_3_ 3)))
(assert (=> (= turnLeft_2_ true) (= mode_3_ 2)))
(assert (=> (= straighten_2_ true) (= mode_3_ 1)))
(assert (or (or (or (= mode_2_ mode_1_) (and (= turnRight_1_ true) (= mode_2_ 3))) (and (= turnLeft_1_ true) (= mode_2_ 2))) (and (= straighten_1_ true) (= mode_2_ 1))))
(assert (=> (= turnRight_1_ true) (= mode_2_ 3)))
(assert (=> (= turnLeft_1_ true) (= mode_2_ 2)))
(assert (=> (= straighten_1_ true) (= mode_2_ 1)))
(assert (or (or (or (= mode_1_ mode_0_) (and (= turnRight_0_ true) (= mode_1_ 3))) (and (= turnLeft_0_ true) (= mode_1_ 2))) (and (= straighten_0_ true) (= mode_1_ 1))))
(assert (=> (= turnRight_0_ true) (= mode_1_ 3)))
(assert (=> (= turnLeft_0_ true) (= mode_1_ 2)))
(assert (=> (= straighten_0_ true) (= mode_1_ 1)))
(assert (=> (= turnLeft_4_ true) (= duration_4_ 0)))
(assert (=> (= turnRight_4_ true) (= duration_4_ 0)))
(assert (=> (= straighten_4_ true) (= duration_4_ 0)))
(assert (=> (= turnLeft_3_ true) (= duration_3_ 0)))
(assert (=> (= turnRight_3_ true) (= duration_3_ 0)))
(assert (=> (= straighten_3_ true) (= duration_3_ 0)))
(assert (=> (= turnLeft_2_ true) (= duration_2_ 0)))
(assert (=> (= turnRight_2_ true) (= duration_2_ 0)))
(assert (=> (= straighten_2_ true) (= duration_2_ 0)))
(assert (=> (= turnLeft_1_ true) (= duration_1_ 0)))
(assert (=> (= turnRight_1_ true) (= duration_1_ 0)))
(assert (=> (= straighten_1_ true) (= duration_1_ 0)))
(assert true_a)
(assert (=> (= turnLeft_0_ true) (= duration_0_ 0)))
(assert (=> (= turnRight_0_ true) (= duration_0_ 0)))
(assert (=> (= straighten_0_ true) (= duration_0_ 0)))
(assert (=> (and (= mode_0_ 1) (= wait_0_ true)) (= [x_0_t y_0_t theta_0_t] (integral 0. duration_0_ [x_0_ y_0_ theta_0_] flow_1))))
(assert (=> (and (= mode_4_ 1) (= wait_4_ true)) (= [x_4_t y_4_t theta_4_t] (integral 0. duration_4_ [x_3_t y_3_t theta_3_t] flow_1))))
(assert (=> (and (= mode_3_ 1) (= wait_3_ true)) (= [x_3_t y_3_t theta_3_t] (integral 0. duration_3_ [x_2_t y_2_t theta_2_t] flow_1))))
(assert (=> (and (= mode_2_ 1) (= wait_2_ true)) (= [x_2_t y_2_t theta_2_t] (integral 0. duration_2_ [x_1_t y_1_t theta_1_t] flow_1))))
(assert (=> (and (= mode_1_ 1) (= wait_1_ true)) (= [x_1_t y_1_t theta_1_t] (integral 0. duration_1_ [x_0_t y_0_t theta_0_t] flow_1))))
(assert (=> (and (= mode_0_ 2) (= wait_0_ true)) (= [x_0_t y_0_t theta_0_t] (integral 0. duration_0_ [x_0_ y_0_ theta_0_] flow_2))))
(assert (=> (and (= mode_4_ 2) (= wait_4_ true)) (= [x_4_t y_4_t theta_4_t] (integral 0. duration_4_ [x_3_t y_3_t theta_3_t] flow_2))))
(assert (=> (and (= mode_3_ 2) (= wait_3_ true)) (= [x_3_t y_3_t theta_3_t] (integral 0. duration_3_ [x_2_t y_2_t theta_2_t] flow_2))))
(assert (=> (and (= mode_2_ 2) (= wait_2_ true)) (= [x_2_t y_2_t theta_2_t] (integral 0. duration_2_ [x_1_t y_1_t theta_1_t] flow_2))))
(assert (=> (and (= mode_1_ 2) (= wait_1_ true)) (= [x_1_t y_1_t theta_1_t] (integral 0. duration_1_ [x_0_t y_0_t theta_0_t] flow_2))))
(assert (=> (and (= mode_0_ 3) (= wait_0_ true)) (= [x_0_t y_0_t theta_0_t] (integral 0. duration_0_ [x_0_ y_0_ theta_0_] flow_3))))
(assert (=> (and (= mode_4_ 3) (= wait_4_ true)) (= [x_4_t y_4_t theta_4_t] (integral 0. duration_4_ [x_3_t y_3_t theta_3_t] flow_3))))
(assert (=> (and (= mode_3_ 3) (= wait_3_ true)) (= [x_3_t y_3_t theta_3_t] (integral 0. duration_3_ [x_2_t y_2_t theta_2_t] flow_3))))
(assert (=> (and (= mode_2_ 3) (= wait_2_ true)) (= [x_2_t y_2_t theta_2_t] (integral 0. duration_2_ [x_1_t y_1_t theta_1_t] flow_3))))
(assert (=> (and (= mode_1_ 3) (= wait_1_ true)) (= [x_1_t y_1_t theta_1_t] (integral 0. duration_1_ [x_0_t y_0_t theta_0_t] flow_3))))
(assert (not (and (not (= x_0_t x_0_)) (= turnLeft_0_ true))))
(assert (not (and (not (= y_0_t y_0_)) (= turnLeft_0_ true))))
(assert (not (and (not (= theta_0_t theta_0_)) (= turnLeft_0_ true))))
(assert (not (and (not (= x_4_t x_3_t)) (= turnLeft_4_ true))))
(assert (not (and (not (= x_3_t x_2_t)) (= turnLeft_3_ true))))
(assert (not (and (not (= x_2_t x_1_t)) (= turnLeft_2_ true))))
(assert (not (and (not (= x_1_t x_0_t)) (= turnLeft_1_ true))))
(assert (not (and (not (= y_4_t y_3_t)) (= turnLeft_4_ true))))
(assert (not (and (not (= y_3_t y_2_t)) (= turnLeft_3_ true))))
(assert (not (and (not (= y_2_t y_1_t)) (= turnLeft_2_ true))))
(assert (not (and (not (= y_1_t y_0_t)) (= turnLeft_1_ true))))
(assert (not (and (not (= theta_4_t theta_3_t)) (= turnLeft_4_ true))))
(assert (not (and (not (= theta_3_t theta_2_t)) (= turnLeft_3_ true))))
(assert (not (and (not (= theta_2_t theta_1_t)) (= turnLeft_2_ true))))
(assert (not (and (not (= theta_1_t theta_0_t)) (= turnLeft_1_ true))))
(assert (not (and (not (= x_0_t x_0_)) (= turnRight_0_ true))))
(assert (not (and (not (= y_0_t y_0_)) (= turnRight_0_ true))))
(assert (not (and (not (= theta_0_t theta_0_)) (= turnRight_0_ true))))
(assert (not (and (not (= x_4_t x_3_t)) (= turnRight_4_ true))))
(assert (not (and (not (= x_3_t x_2_t)) (= turnRight_3_ true))))
(assert (not (and (not (= x_2_t x_1_t)) (= turnRight_2_ true))))
(assert (not (and (not (= x_1_t x_0_t)) (= turnRight_1_ true))))
(assert (not (and (not (= y_4_t y_3_t)) (= turnRight_4_ true))))
(assert (not (and (not (= y_3_t y_2_t)) (= turnRight_3_ true))))
(assert (not (and (not (= y_2_t y_1_t)) (= turnRight_2_ true))))
(assert (not (and (not (= y_1_t y_0_t)) (= turnRight_1_ true))))
(assert (not (and (not (= theta_4_t theta_3_t)) (= turnRight_4_ true))))
(assert (not (and (not (= theta_3_t theta_2_t)) (= turnRight_3_ true))))
(assert (not (and (not (= theta_2_t theta_1_t)) (= turnRight_2_ true))))
(assert (not (and (not (= theta_1_t theta_0_t)) (= turnRight_1_ true))))
(assert (not (and (not (= x_0_t x_0_)) (= straighten_0_ true))))
(assert (not (and (not (= y_0_t y_0_)) (= straighten_0_ true))))
(assert (not (and (not (= theta_0_t theta_0_)) (= straighten_0_ true))))
(assert (not (and (not (= x_4_t x_3_t)) (= straighten_4_ true))))
(assert (not (and (not (= x_3_t x_2_t)) (= straighten_3_ true))))
(assert (not (and (not (= x_2_t x_1_t)) (= straighten_2_ true))))
(assert (not (and (not (= x_1_t x_0_t)) (= straighten_1_ true))))
(assert (not (and (not (= y_4_t y_3_t)) (= straighten_4_ true))))
(assert (not (and (not (= y_3_t y_2_t)) (= straighten_3_ true))))
(assert (not (and (not (= y_2_t y_1_t)) (= straighten_2_ true))))
(assert (not (and (not (= y_1_t y_0_t)) (= straighten_1_ true))))
(assert (not (and (not (= theta_4_t theta_3_t)) (= straighten_4_ true))))
(assert (not (and (not (= theta_3_t theta_2_t)) (= straighten_3_ true))))
(assert (not (and (not (= theta_2_t theta_1_t)) (= straighten_2_ true))))
(assert (not (and (not (= theta_1_t theta_0_t)) (= straighten_1_ true))))
(assert (not (and (= straighten_4_ true) (= mode_4_ 1))))
(assert (not (and (= straighten_3_ true) (= mode_3_ 1))))
(assert (not (and (= straighten_2_ true) (= mode_2_ 1))))
(assert (not (and (= straighten_1_ true) (= mode_1_ 1))))
(assert (not (and (= straighten_0_ true) (= mode_0_ 1))))
(assert (not (and (= turnLeft_4_ true) (= mode_4_ 2))))
(assert (not (and (= turnLeft_3_ true) (= mode_3_ 2))))
(assert (not (and (= turnLeft_2_ true) (= mode_2_ 2))))
(assert (not (and (= turnLeft_1_ true) (= mode_1_ 2))))
(assert (not (and (= turnLeft_0_ true) (= mode_0_ 2))))
(assert (not (and (= turnRight_4_ true) (= mode_4_ 3))))
(assert (not (and (= turnRight_3_ true) (= mode_3_ 3))))
(assert (not (and (= turnRight_2_ true) (= mode_2_ 3))))
(assert (not (and (= turnRight_1_ true) (= mode_1_ 3))))
(assert (not (and (= turnRight_0_ true) (= mode_0_ 3))))
(assert (not (not (> (+ (* (- x_0_ 9) (- x_0_ 9)) (* y_0_ y_0_)) 9))))
(assert (not (not (> (+ (* (- x_4_t 9) (- x_4_t 9)) (* y_4_t y_4_t)) 9))))
(assert (not (not (> (+ (* (- x_3_t 9) (- x_3_t 9)) (* y_3_t y_3_t)) 9))))
(assert (not (not (> (+ (* (- x_2_t 9) (- x_2_t 9)) (* y_2_t y_2_t)) 9))))
(assert (not (not (> (+ (* (- x_1_t 9) (- x_1_t 9)) (* y_1_t y_1_t)) 9))))
(assert (not (not (> (+ (* (- x_0_t 9) (- x_0_t 9)) (* y_0_t y_0_t)) 9))))
(assert (forall_t 1 [0 duration_4_] (not (not (> (+ (* (- x_4_t 9) (- x_4_t 9)) (* y_4_t y_4_t)) 9)))))
(assert (forall_t 1 [0 duration_3_] (not (not (> (+ (* (- x_3_t 9) (- x_3_t 9)) (* y_3_t y_3_t)) 9)))))
(assert (forall_t 1 [0 duration_2_] (not (not (> (+ (* (- x_2_t 9) (- x_2_t 9)) (* y_2_t y_2_t)) 9)))))
(assert (forall_t 1 [0 duration_1_] (not (not (> (+ (* (- x_1_t 9) (- x_1_t 9)) (* y_1_t y_1_t)) 9)))))
(assert (forall_t 1 [0 duration_0_] (not (not (> (+ (* (- x_0_t 9) (- x_0_t 9)) (* y_0_t y_0_t)) 9)))))
(assert (forall_t 2 [0 duration_4_] (not (not (> (+ (* (- x_4_t 9) (- x_4_t 9)) (* y_4_t y_4_t)) 9)))))
(assert (forall_t 2 [0 duration_3_] (not (not (> (+ (* (- x_3_t 9) (- x_3_t 9)) (* y_3_t y_3_t)) 9)))))
(assert (forall_t 2 [0 duration_2_] (not (not (> (+ (* (- x_2_t 9) (- x_2_t 9)) (* y_2_t y_2_t)) 9)))))
(assert (forall_t 2 [0 duration_1_] (not (not (> (+ (* (- x_1_t 9) (- x_1_t 9)) (* y_1_t y_1_t)) 9)))))
(assert (forall_t 2 [0 duration_0_] (not (not (> (+ (* (- x_0_t 9) (- x_0_t 9)) (* y_0_t y_0_t)) 9)))))
(assert (forall_t 3 [0 duration_4_] (not (not (> (+ (* (- x_4_t 9) (- x_4_t 9)) (* y_4_t y_4_t)) 9)))))
(assert (forall_t 3 [0 duration_3_] (not (not (> (+ (* (- x_3_t 9) (- x_3_t 9)) (* y_3_t y_3_t)) 9)))))
(assert (forall_t 3 [0 duration_2_] (not (not (> (+ (* (- x_2_t 9) (- x_2_t 9)) (* y_2_t y_2_t)) 9)))))
(assert (forall_t 3 [0 duration_1_] (not (not (> (+ (* (- x_1_t 9) (- x_1_t 9)) (* y_1_t y_1_t)) 9)))))
(assert (forall_t 3 [0 duration_0_] (not (not (> (+ (* (- x_0_t 9) (- x_0_t 9)) (* y_0_t y_0_t)) 9)))))
(assert (not (not (> (+ (* (- x_0_ 5) (- x_0_ 5)) (* (- y_0_ 7) (- y_0_ 7))) 4))))
(assert (not (not (> (+ (* (- x_4_t 5) (- x_4_t 5)) (* (- y_4_t 7) (- y_4_t 7))) 4))))
(assert (not (not (> (+ (* (- x_3_t 5) (- x_3_t 5)) (* (- y_3_t 7) (- y_3_t 7))) 4))))
(assert (not (not (> (+ (* (- x_2_t 5) (- x_2_t 5)) (* (- y_2_t 7) (- y_2_t 7))) 4))))
(assert (not (not (> (+ (* (- x_1_t 5) (- x_1_t 5)) (* (- y_1_t 7) (- y_1_t 7))) 4))))
(assert (not (not (> (+ (* (- x_0_t 5) (- x_0_t 5)) (* (- y_0_t 7) (- y_0_t 7))) 4))))
(assert (forall_t 1 [0 duration_4_] (not (not (> (+ (* (- x_4_t 5) (- x_4_t 5)) (* (- y_4_t 7) (- y_4_t 7))) 4)))))
(assert (forall_t 1 [0 duration_3_] (not (not (> (+ (* (- x_3_t 5) (- x_3_t 5)) (* (- y_3_t 7) (- y_3_t 7))) 4)))))
(assert (forall_t 1 [0 duration_2_] (not (not (> (+ (* (- x_2_t 5) (- x_2_t 5)) (* (- y_2_t 7) (- y_2_t 7))) 4)))))
(assert (forall_t 1 [0 duration_1_] (not (not (> (+ (* (- x_1_t 5) (- x_1_t 5)) (* (- y_1_t 7) (- y_1_t 7))) 4)))))
(assert (forall_t 1 [0 duration_0_] (not (not (> (+ (* (- x_0_t 5) (- x_0_t 5)) (* (- y_0_t 7) (- y_0_t 7))) 4)))))
(assert (forall_t 2 [0 duration_4_] (not (not (> (+ (* (- x_4_t 5) (- x_4_t 5)) (* (- y_4_t 7) (- y_4_t 7))) 4)))))
(assert (forall_t 2 [0 duration_3_] (not (not (> (+ (* (- x_3_t 5) (- x_3_t 5)) (* (- y_3_t 7) (- y_3_t 7))) 4)))))
(assert (forall_t 2 [0 duration_2_] (not (not (> (+ (* (- x_2_t 5) (- x_2_t 5)) (* (- y_2_t 7) (- y_2_t 7))) 4)))))
(assert (forall_t 2 [0 duration_1_] (not (not (> (+ (* (- x_1_t 5) (- x_1_t 5)) (* (- y_1_t 7) (- y_1_t 7))) 4)))))
(assert (forall_t 2 [0 duration_0_] (not (not (> (+ (* (- x_0_t 5) (- x_0_t 5)) (* (- y_0_t 7) (- y_0_t 7))) 4)))))
(assert (forall_t 3 [0 duration_4_] (not (not (> (+ (* (- x_4_t 5) (- x_4_t 5)) (* (- y_4_t 7) (- y_4_t 7))) 4)))))
(assert (forall_t 3 [0 duration_3_] (not (not (> (+ (* (- x_3_t 5) (- x_3_t 5)) (* (- y_3_t 7) (- y_3_t 7))) 4)))))
(assert (forall_t 3 [0 duration_2_] (not (not (> (+ (* (- x_2_t 5) (- x_2_t 5)) (* (- y_2_t 7) (- y_2_t 7))) 4)))))
(assert (forall_t 3 [0 duration_1_] (not (not (> (+ (* (- x_1_t 5) (- x_1_t 5)) (* (- y_1_t 7) (- y_1_t 7))) 4)))))
(assert (forall_t 3 [0 duration_0_] (not (not (> (+ (* (- x_0_t 5) (- x_0_t 5)) (* (- y_0_t 7) (- y_0_t 7))) 4)))))
(assert (not (not (> (+ (* (- x_0_ 12) (- x_0_ 12)) (* (- y_0_ 9) (- y_0_ 9))) 4))))
(assert (not (not (> (+ (* (- x_4_t 12) (- x_4_t 12)) (* (- y_4_t 9) (- y_4_t 9))) 4))))
(assert (not (not (> (+ (* (- x_3_t 12) (- x_3_t 12)) (* (- y_3_t 9) (- y_3_t 9))) 4))))
(assert (not (not (> (+ (* (- x_2_t 12) (- x_2_t 12)) (* (- y_2_t 9) (- y_2_t 9))) 4))))
(assert (not (not (> (+ (* (- x_1_t 12) (- x_1_t 12)) (* (- y_1_t 9) (- y_1_t 9))) 4))))
(assert (not (not (> (+ (* (- x_0_t 12) (- x_0_t 12)) (* (- y_0_t 9) (- y_0_t 9))) 4))))
(assert (forall_t 1 [0 duration_4_] (not (not (> (+ (* (- x_4_t 12) (- x_4_t 12)) (* (- y_4_t 9) (- y_4_t 9))) 4)))))
(assert (forall_t 1 [0 duration_3_] (not (not (> (+ (* (- x_3_t 12) (- x_3_t 12)) (* (- y_3_t 9) (- y_3_t 9))) 4)))))
(assert (forall_t 1 [0 duration_2_] (not (not (> (+ (* (- x_2_t 12) (- x_2_t 12)) (* (- y_2_t 9) (- y_2_t 9))) 4)))))
(assert (forall_t 1 [0 duration_1_] (not (not (> (+ (* (- x_1_t 12) (- x_1_t 12)) (* (- y_1_t 9) (- y_1_t 9))) 4)))))
(assert (forall_t 1 [0 duration_0_] (not (not (> (+ (* (- x_0_t 12) (- x_0_t 12)) (* (- y_0_t 9) (- y_0_t 9))) 4)))))
(assert (forall_t 2 [0 duration_4_] (not (not (> (+ (* (- x_4_t 12) (- x_4_t 12)) (* (- y_4_t 9) (- y_4_t 9))) 4)))))
(assert (forall_t 2 [0 duration_3_] (not (not (> (+ (* (- x_3_t 12) (- x_3_t 12)) (* (- y_3_t 9) (- y_3_t 9))) 4)))))
(assert (forall_t 2 [0 duration_2_] (not (not (> (+ (* (- x_2_t 12) (- x_2_t 12)) (* (- y_2_t 9) (- y_2_t 9))) 4)))))
(assert (forall_t 2 [0 duration_1_] (not (not (> (+ (* (- x_1_t 12) (- x_1_t 12)) (* (- y_1_t 9) (- y_1_t 9))) 4)))))
(assert (forall_t 2 [0 duration_0_] (not (not (> (+ (* (- x_0_t 12) (- x_0_t 12)) (* (- y_0_t 9) (- y_0_t 9))) 4)))))
(assert (forall_t 3 [0 duration_4_] (not (not (> (+ (* (- x_4_t 12) (- x_4_t 12)) (* (- y_4_t 9) (- y_4_t 9))) 4)))))
(assert (forall_t 3 [0 duration_3_] (not (not (> (+ (* (- x_3_t 12) (- x_3_t 12)) (* (- y_3_t 9) (- y_3_t 9))) 4)))))
(assert (forall_t 3 [0 duration_2_] (not (not (> (+ (* (- x_2_t 12) (- x_2_t 12)) (* (- y_2_t 9) (- y_2_t 9))) 4)))))
(assert (forall_t 3 [0 duration_1_] (not (not (> (+ (* (- x_1_t 12) (- x_1_t 12)) (* (- y_1_t 9) (- y_1_t 9))) 4)))))
(assert (forall_t 3 [0 duration_0_] (not (not (> (+ (* (- x_0_t 12) (- x_0_t 12)) (* (- y_0_t 9) (- y_0_t 9))) 4)))))
(assert (not (not (= x_0_ 0))))
(assert (not (not (= y_0_ 0))))
(assert (not (not (= theta_0_ 0))))
(assert (not (not (= mode_0_ 1))))
(assert (not (not (= x_4_t 13))))
(assert (not (not (= y_4_t 0))))



(check-sat)
(exit)

\end{lstlisting}
\EOCCC

\subsubsection{Output}
\begin{center}
  \includegraphics[height=7cm]{car-turn-fig.png}
\end{center}
\begin{lstlisting}
Command: cplus2aspmt car.cp -c maxstep=3 -c query=test

Output:
Solution:
duration_0_ : [ ENTIRE ] = [8.250457763671875, 8.25128173828125]
duration_1_ : [ ENTIRE ] = [0, 0]
duration_2_ : [ ENTIRE ] = [11.80044126510621, 11.80111503601076]
mode_0_ : [ ENTIRE ] = [1, 1]
mode_1_ : [ ENTIRE ] = [1, 1]
mode_2_ : [ ENTIRE ] = [3, 3]
mode_3_ : [ ENTIRE ] = [3, 3]
theta_0_ : [ ENTIRE ] = [0.6918, 0.6918000000000001]
theta_0_t : [ ENTIRE ] = [0.6918, 0.6918000000000001]
theta_1_t : [ ENTIRE ] = [0.6918, 0.6918000000000001]
theta_2_t : [ ENTIRE ] = [-2.031548557285888, -2.03139307087505]
x_0_ : [ ENTIRE ] = [0, 0]
x_0_t : [ ENTIRE ] = [6.353669267319927, 6.354303809342038]
x_1_t : [ ENTIRE ] = [6.353669267319927, 6.354303809342038]
x_2_t : [ ENTIRE ] = [13, 13]
y_0_ : [ ENTIRE ] = [0, 0]
y_0_t : [ ENTIRE ] = [5.263168261764744, 5.263693895267374]
y_1_t : [ ENTIRE ] = [5.263168261764744, 5.263693895267374]
y_2_t : [ ENTIRE ] = [0, 0]
true_a : Bool = true
false_a : Bool = false
qlabel_test_ : Bool = true
straighten_0_ : Bool = false
straighten_1_ : Bool = false
straighten_2_ : Bool = false
turnLeft_0_ : Bool = false
turnLeft_1_ : Bool = false
turnLeft_2_ : Bool = false
turnRight_0_ : Bool = false
turnRight_1_ : Bool = true
turnRight_2_ : Bool = false
wait_0_ : Bool = true
wait_1_ : Bool = false
wait_2_ : Bool = true
delta-sat with delta = 0.00100000000000000
Total time in milliseconds: 296721
\end{lstlisting}

\section{Experiments} \label{sec:experiments}

%We run experiments on the examples described above to compare the performance of {\sc cplus2aspmt} with other related systems. We also run experiments to demonstrate the capability of the system in minimizing the effort in encoding by comparing encoding sizes and number of ground atoms with Hybrid Automata and SMT solvers.

%\subsection{Runtime Performance}
\BOCC
\begin{table}[h!]
\centering
\begin{tabular}{ p{6cm}||p{3cm}|p{3cm} }
 \hline
 Example & {\sc dReach} & {\sc cplus2ASPMT}\\
\hline
\hline
Water Tank Example(2 steps) & 0.080s/15 & 0.11/16\\
Water Tank Example(4 steps) & 0.083s/20 & 0.33/32\\
Water Tank Example(6 steps) & 0.085s/25 & 4.6/45\\
\hline
\hline
 Two Ball Example(2 steps)  & 0.110s/23 & 0.107s/23\\
 Two Ball Example(4 steps)  & 0.169s/33 & 1.393s/43\\
 Two Ball Example(6 steps)  & 0.285s/43 & 5.69s/61\\
\hline
\hline
 Turning Car example(3 steps) & 5.34s/34 & 8.01s/44 \\
\hline
\end{tabular}
\vspace{0.5cm}
\caption{Runtime comparison of Hybrid Automata solvers - runtime(s)/final number of ground atoms}
\label{tab:time-results}
\end{table}
\EOCC

\BOCC
\begin{table}[h!]
\centering
\begin{tabular}{ c||c|c|c }
\hline
Steps & {\sc dReach} & {\sc cplus2ASPMT} (path given) & {\sc cplus2ASPMT} (goal given but no path)\\
  \hline
  \hline
 \multicolumn{4}{c}{Thermostat Example}\\
 \hline
 6 & .454 & .213 & .243\\
 \hline
 10 & 0.658 & 0.645 & 3.22\\
 \hline
12 & 1.050 & 1.060 & $>10$m\\
\hline
20 & 5.25 &  9.6076 & $>$10m\\
 \hline
 \multicolumn{4}{c}{Car Example}\\ 
 \hline
3 & 0.876 & 1.123 & 8.735\\
\hline
4 & 2.312 & 17.23 & 3m28\\
\hline
5 & 5.765 & 33.533 & $>$10m\\
 \hline 
6 & 7.322 & 2m43 & $>$10m\\
 \hline
\end{tabular}
\vspace{0.5cm}
\caption{Runtime comparison (seconds)}
\label{tab:time-results}
\end{table}

We compare the run time of system ${\tt dReach}$ \cite{kong15dreach} and {\sc cplus2aspmt} on the thermostat \cite{lygeros04lecture} and the car example in Table~\ref{tab:time-results}. 
Both systems use ${\tt dReal}$ internally but ${\tt dReach}$ is optimized for hybrid automata computation.
It generates a compact logical encoding and makes iterative calls to ${\tt dReal}$ to decide reachability properties.  On the other hand {\sc cplus2aspmt} generates a one-time large encoding and calls ${\tt dReal}$ once. 
 From the table we see that both systems have a similar runtime for smaller number of steps but our system takes a much longer time as steps and domain size increase.
 
It may be possible to improve the run time of {\sc cplus2aspmt} by leveraging incremental answer set computation, which we leave for future work.
\EOCC

%\BOC
%\section{Experiments}

%We run experiments on the examples described above to compare the performance of {\sc cplus2aspmt} with other related systems. We also run experiments to demonstrate the capability of the system in minimizing the effort in encoding by comparing encoding sizes and number of ground atoms with Hybrid Automata and SMT solvers.

%\subsection{Runtime Performance}

\begin{table}[h!]
\centering
\begin{tabular}{ c||c|c|c|c|c }
\hline
Steps & 1 & 3 & 6 & 8 & 10 \\
 \hline
 \hline
 ${\tt dReach}$ {(encoding from \cite{bryce15smt})} & 0.098 & 0.225 & 0.690 & 2.123 & 3.143 \\
 \hline
 {\sc cplus2aspmt} & 0.198 & 7.55 & 18.23 & 88.93 & $>600$\\
 \hline
\end{tabular}
\vspace{0.5cm}
\caption{Runtime Comparison (seconds)}
\label{tab:time-results}
\end{table}

We compare the run time of the system in \cite{bryce15smt} and {\sc cplus2aspmt} for the car domain~\cite{fox06modelling} and the result is shown in Table~\ref{tab:time-results}. 
In~\cite{bryce15smt}, the encoding was in the language of ${\tt dReach}$, which calls ${\tt dReal}$ internally. The computation is optimized for pruning invalid paths of a transition system. It filters out invalid paths using  heuristics described in their paper, generates a compact logical encoding, and makes a call to ${\tt dReal}$ to decide reachability properties.  On the other hand {\sc cplus2aspmt} generates a one-time large encoding without filtering paths and calls ${\tt dReal}$ once. 
 From the table we see that the system presented in \cite{bryce15smt} does perform better than {\sc cplus2aspmt}. As steps increases the difference in run time also increases.
 
It may be possible to improve the run time of {\sc cplus2aspmt} by leveraging incremental answer set computation and path heuristics, which we leave for future work.

\BOCC
\subsection{Encoding Size}

\begin{table}[h!]
\centering
\begin{tabular}{ c||c|c|c}
 \hline
 Example & {\sc cplus2ASPMT} & {\sc dReach} & {\sc dReal}\\
 \hline
Water Tank Example(2 steps) & 47 & 33 & 115\\
 \hline
 Two Ball Example(4 steps)  & 45 & 26 & 275\\
\hline
 Turning Car example(3 steps) & 66 & 37 & 304 \\
 \hline
\end{tabular}
\vspace{0.5cm}
\caption{Encoding size (number of lines) comparison}
\label{tab:space-results}
\end{table}

%{\cred 
For each of the examples described in the earlier section we also compare their encoding size with SMT solver ${\tt dReal}$ and Hybrid Automata solver ${\tt dReach}$ in Table~\ref{tab:space-results}. We clearly observe that our system is much more efficient than SMT solvers like ${\tt dReal}$ in terms of encoding effort. Additionally, it is important to note that in our system, the encoding size remains the size no matter the number of steps required to obtain the solution. This is not the case for its equivalent SMT encoding as SMT input must always be grounded. We notice that ${\tt dReach}$ has a slight advantage in this metric as well. 
%But once again, we highlight the fact that {\sc dReach} is not tolerable to hybrid transition systems not defined by Hybrid Automata while {\sc cplus2ASPMT} is.
%}
\EOCC

%
%\begin{thebibliography}{}
%
%\bibitem[\protect\citeauthoryear{Boyd and Vandenberghe}{Boyd and
%  Vandenberghe}{2004}]{boyd04convex}
%{\sc Boyd, S.} {\sc and} {\sc Vandenberghe, L.} 2004.
%\newblock {\em Convex optimization}.
%\newblock Cambridge university press.
%
%\end{thebibliography}

\end{appendix}

\end{document}